\documentclass[a4paper,fleqn]{cas-sc}
\usepackage[T1]{fontenc}
\usepackage[utf8]{inputenc}
\usepackage{multirow}

\usepackage[linesnumbered,ruled,vlined,longend]{algorithm2e}
\usepackage{amsmath,amssymb}
\usepackage{caption}
\usepackage{graphicx}
\usepackage[table]{xcolor}
\usepackage{float}
\usepackage{mathtools}
\usepackage{subcaption}
\usepackage{bm}
\usepackage{todonotes}
\usepackage{amsthm}
\usepackage{makecell}
\usepackage{placeins}
\usepackage{tikz}
\usetikzlibrary{positioning, arrows.meta}
\usetikzlibrary{matrix}
\usepackage{listings}
\usepackage{cleveref}

\usepackage[numbers]{natbib}

\def\tsc#1{\csdef{#1}{\textsc{\lowercase{#1}}\xspace}}
\tsc{WGM}
\tsc{QE}

\definecolor{kOneTint}{RGB}{220,240,255}
\definecolor{kTenTint}{RGB}{225,245,225}
\definecolor{kHundredTint}{RGB}{255,235,225}
\definecolor{kOne}{RGB}{220,235,250}
\definecolor{kTen}{RGB}{220,245,220}
\definecolor{kHundred}{RGB}{235,220,245}

\theoremstyle{plain}
\newtheorem{theorem}{Theorem}[section]

\renewcommand{\L}{\mathcal{L}}
\newcommand{\btheta}{{\bm{\theta}}}

\begin{document}
\let\WriteBookmarks\relax
\def\floatpagepagefraction{1}
\def\textpagefraction{.001}

\shorttitle{Curvature-Aware Optimization for High-Accuracy PINNs}
\shortauthors{Jnini et al.}

\title[mode=title]{Curvature-Aware Optimization for High-Accuracy Physics-Informed Neural Networks}

\author[1]{Anas Jnini}
\fnmark[1]

\author[2]{Elham Kiyani}
\fnmark[1]

\author[2]{Khemraj Shukla}
\fnmark[1]

\author[3]{Jorge F. Urb\'an}
\fnmark[1]

\author[4]{Nazanin {Ahmadi Daryakenari}}

\author[5]{Johannes M\"uller}

\author[6]{Marius Zeinhofer}
\cormark[1]
\ead{marius.zeinhofer@math.ethz.ch}

\author[2]{George {Em Karniadakis}}
\cormark[1]
\ead{george\_karniadakis@brown.edu}

\fntext[1]{These authors contributed equally to this work and are listed in alphabetical order by last name.}
\cortext[1]{Corresponding author}

\affiliation[1]{
    organization={Department of Information Engineering and Computer Science, University of Trento},
    city={Trento},
    country={Italy}
}

\affiliation[2]{
    organization={Division of Applied Mathematics, Brown University},
    city={Providence},
    state={RI},
    country={USA}
}

\affiliation[3]{
    organization={Departament de F\'isica, Universitat d'Alacant},
    country={Spain}
}

\affiliation[4]{
    organization={Center for Biomedical Engineering, Brown University},
    city={Providence},
    state={RI},
    country={USA}
}

\affiliation[5]{
    organization={Institute of Mathematics, TU Berlin},
    country={Germany}
}

\affiliation[6]{
    organization={ETH Zurich},
    country={Switzerland}
}

\begin{abstract}
Efficient and robust optimization is essential for neural networks, enabling scientific machine learning models to converge rapidly to very high accuracy---faithfully capturing complex physical behavior governed by differential equations. In this work, we present advanced optimization strategies to accelerate the convergence of physics-informed neural networks (PINNs) for challenging partial (PDEs) and ordinary differential equations (ODEs). Specifically, we provide efficient implementations of the Natural Gradient (NG) optimizer, Self-Scaling BFGS and Broyden optimizers, and demonstrate their performance on problems including the Helmholtz equation, Stokes flow, inviscid Burgers equation, Euler equations for high-speed flows, and stiff ODEs arising in pharmacokinetics and pharmacodynamics. Beyond optimizer development, we also propose new PINN-based methods for solving the inviscid Burgers and Euler equations, and compare the resulting solutions against high-order numerical methods to provide a rigorous and fair assessment. Finally, we address the challenge of scaling these quasi-Newton optimizers for batched training, enabling efficient and scalable solutions for large data-driven problems.
\end{abstract}


\maketitle

\section{Introduction}
Optimizing neural networks is a central challenge in modern machine learning and an active area of theoretical research~\cite{lecun2015deep, arora2018optimization, kawaguchi2016deep,kopanivcakova2026introduction}. Despite their highly non-convex structure, neural network optimization problems are often more tractable in practice than worst-case non-convex problems, suggesting the presence of exploitable structure in their landscapes. Most training algorithms are variants of gradient-based methods, where gradients are efficiently computed via backpropagation, and parameters are updated iteratively using carefully tuned step sizes. The associated optimization difficulties can be broadly decomposed into three aspects: ensuring convergence to stationary points, achieving fast convergence, and attaining solutions of high global quality. These challenges manifest through local issues such as vanishing or exploding gradients and slow convergence, as well as global issues including poor local minima, flat plateaus, and complex loss landscapes. Understanding these phenomena and developing principled algorithmic and theoretical tools to address them is key to making neural networks reliable and powerful optimization models.

Because training ultimately comes down to choosing parameters that minimize a loss, optimization acts as the mechanism that turns data and model structure into learned behavior. A useful analogy is navigating rugged terrain in dense fog: the loss surface may have valleys, cliffs, and flat basins, and the algorithm must decide how to move using only limited local information. Gradients provide a direction of steepest descent, but without any sense of curvature, progress can be slow in flat regions or unstable in sharply curved ones.

This limitation motivates a natural hierarchy of methods. Formally, first-order methods such as gradient descent, update parameters using only gradient evaluations. They are simple and scalable, but they can zigzag in narrow valleys and require careful step-size tuning to avoid divergence. Second-order methods, in contrast, incorporate curvature through matrices like the Hessian, Gauss–Newton, or Fisher information matrices~\cite{amari1997information, amari1998natural}. By capturing local geometry, these methods aim to produce better-scaled search directions and step sizes, often enabling faster and more reliable convergence.

\begin{figure}[h!]
  \centering
  \includegraphics[width=0.99\linewidth, height=0.89\textheight, keepaspectratio]{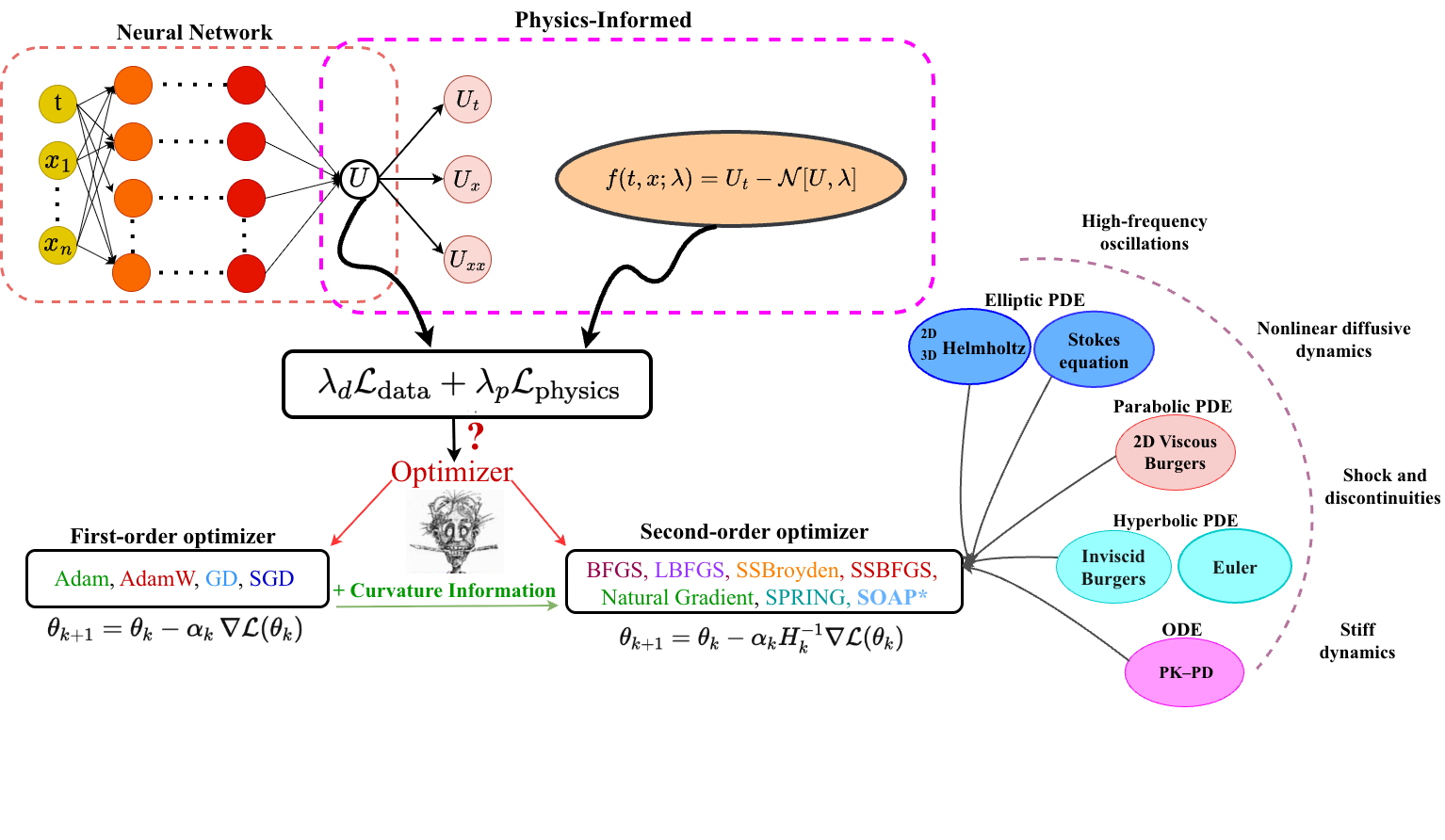}
\caption{Schematic illustration of PINN training and optimization methods considered in this work. Our benchmarks include  PDEs of elliptic (Helmholtz and Stokes), parabolic (2D viscous Burgers), hyperbolic (inviscid Burgers and 1D Euler) type, and a stiff PK--PD ODE system, characterized by oscillatory solutions, nonlinear diffusion, shock-induced discontinuities, and stiffness. 
The methods studied here are BFGS, SSBFGS, SSBroyden, NG, SPRING, and SOAP. SOAP$^{*}$ is included for comparison as a curvature-aware preconditioned method, although it is not a classical Newton or quasi-Newton method.}
\label{Fig:Schematic}
\end{figure}


The physical intuition mirrors the terrain analogy. At first, you may only know which direction leads uphill or downhill---this corresponds to gradient descent, which follows the local slope (the gradient) but ignores how the landscape bends. As a result, the optimizer may creep across plateaus or oscillate across steep ravines. Now imagine discovering a detailed map that reveals not only the slope but also how the ground curves beneath your feet. That curvature information is encoded in second-order objects such as the Hessian, Gauss--Newton, and Fisher information. Leveraging this structure enables updates to adapt to the local shape of the loss surface, which is the central idea behind second-order optimization.
Together, these lines of work suggest that the key bottleneck is not whether curvature helps, but how to exploit it under tight computational and memory constraints. This motivates methods that approximate second-order information in ways that are both scalable and robust, aiming to bridge the gap between the efficiency of first-order training and the geometric awareness of curvature-based updates.

To better understand the role of curvature information in PINN training, the following functional interpretation of second-order methods is useful.
%
%
Specifically, the training of PINNs can be interpreted as the minimization of a nonlinear least-squares objective of the form
\begin{equation*}
\mathcal{L}(\btheta)=\frac{1}{2}\|r(\btheta)\|^{2},
\end{equation*}
where \(r(\btheta)\) denotes the collection of residuals associated with the governing equations, boundary conditions, and data constraints. Linearizing the residual around the current iterate \(\btheta_k\) yields
\begin{equation*}
r(\btheta)\approx r(\btheta_k)+J_k(\btheta-\btheta_k),
\end{equation*}
where \(J_k\) is the Jacobian of the residual with respect to the parameters. In this setting, classical gradient descent updates take the form
\begin{equation*}
\btheta_{k+1}
=
\btheta_k-\eta_k J_k^\top r(\btheta_k),
\end{equation*}
which can be interpreted as steepest descent in parameter space.
In contrast, curvature-aware methods such as Gauss--Newton introduce the approximate Hessian \(H_k \approx J_k^\top J_k\), leading to updates of the form
\begin{equation*}
\btheta_{k+1}
=
\btheta_k-\eta_k (J_k^\top J_k)^{-1}J_k^\top r(\btheta_k).
\end{equation*}

This interpretation provides a unifying explanation for the improved performance of second-order methods in PINNs. In particular, the poor training behavior commonly observed in PINNs can be attributed to ill-conditioning of the NTK, which leads to slow convergence and biased learning toward low-frequency components (spectral bias). By incorporating curvature information, Gauss--Newton-type methods effectively regularize and rescale the NTK spectrum, improving conditioning and enabling more uniform error reduction across modes. The proposed Gram--Gauss--Newton approximation can thus be viewed as a tractable surrogate for this ideal preconditioning, balancing computational cost and conditioning improvement in overparameterized regimes.

In the present work, we conduct a broad empirical study of optimizer performance, including SSBFGS, SSBroyden, NG, SOAP~\cite{vyas2024soap}, and AdamW~\cite{zhou2024towards}. 
Our primary focus is on challenging training regimes arising in PINNs, a framework originally introduced to solve forward and inverse problems for differential equations by embedding the governing physics directly into the training objective~\cite{RAISSI2019686}.
Since their introduction, PINNs have attracted substantial attention in scientific machine learning because they provide a mesh-free and flexible approach for incorporating physical constraints into neural-network training. At the same time, a growing body of work has shown that PINN training is often significantly more delicate than standard supervised learning, owing to gradient-flow pathologies, imbalanced loss terms, spectral bias, and optimization failure modes that become especially severe for stiff, multiscale, highly oscillatory, or shock-dominated problems~\cite{wang2021understanding,NEURIPS2021_df438e52,KIYANI2023116258,toscano2025pinns,review}.
Recent studies have also emphasized that, beyond architecture and sampling choices, the optimizer itself can be a decisive factor in determining whether PINNs converge to low-error solutions and whether they can reach very high accuracy in practice~\cite{kiyani2025optimizing,urban2025unveiling,jnini2025gauss,jnini2025dual,guzman-cordero2025improving,schwencke2025anagramnaturalgradientrelative,schwencke2025amstramgramadaptivemulticutoffstrategy}.
Motivated by these observations, we benchmark curvature-aware optimizers on a suite of demanding problems spanning a stiff ODE system (PK--PD) and three classes of PDEs: elliptic problems (2D and 3D Helmholtz), parabolic problems (2D viscous Burgers), and hyperbolic problems (inviscid Burgers and the 1D Euler equations). These test cases represent distinct sources of training difficulty, including stiff dynamics, high-frequency oscillatory solutions, nonlinear diffusive behavior, and shock-induced discontinuities. Since our earlier comprehensive study~\cite{kiyani2025optimizing} has already established that double precision yields superior performance compared to single precision on such problems, we restrict attention here to double precision computations unless noted otherwise. A detailed comparison of single- and double-precision performance can be found in~\cite{kiyani2025optimizing}.
%
%
Beyond benchmarking, this work also introduces novel PINN-based methodologies for solving hyperbolic PDEs with discontinuous solutions. In particular, convergence is demonstrated for both the inviscid Burgers equation and the compressible Euler equations. Furthermore, new batch-training variants of the SSBFGS and SSBroyden algorithms are proposed to improve efficiency and scalability in large-scale training settings.
The aforementioned functional perspective motivates the design and evaluation of curvature-aware optimization strategies for PINNs. The main contributions of this work are summarized as follows:
\begin{itemize}

\item We introduce novel methods in PINNs for solving hyperbolic conservation laws governing high-speed flows, with applications to Burgers’ and Euler equations.

\item We provide a systematic analysis of optimizer performance across different classes of PDEs, including elliptic, parabolic, and hyperbolic regimes.

\item We develop a new framework for analyzing PDE learning dynamics through loss landscapes in both function and parameter spaces.

\item We present a detailed roofline analysis to characterize the computational scaling and efficiency of quasi-Newton optimizers in PINNs.

\item We propose a new batch training strategy tailored for quasi-Newton optimization in PINNs.

\item We perform a comprehensive comparison of multiple optimization methods, including SSBFGS, SSBroyden, NG, SPRING, and SOAP, across a range of PDEs solved using PINNs. The benchmark problems include the Helmholtz equation, Stokes flow, inviscid Burgers' equation, Euler equations, as well as stiff ODEs arising in pharmacokinetics and pharmacodynamics modeling.


\item We  propose a new methodology to mitigate the spectral bias in Helmholtz equation for higher higher wavenumbers.

\item We demonstrate the applicability of PINNs to stiff ODE systems arising in pharmacokinetics and pharmacodynamics.

\item We propose a novel incorporation of Jacobi scaling within NG methods to improve optimization performance in terms of stability.

\item While the present study focuses on forward problems, we outline how the proposed methods naturally extend to inverse problems, including parameter estimation and data assimilation.

\end{itemize}

\paragraph{Outline:}
The paper begins with a comprehensive review of the second-order and quasi-second-order optimization methods in Section~\ref{sec:Review_of_Second-order_Optimizers}, including the formulation of each optimizer and the line-search strategies used in this study. 
Evaluation metrics and landscape-visualization tools are introduced in Sections~\ref{sec:Metrics} and~\ref{sec:Landscape_Visualization}.
A broad set of challenging PINN benchmarks is then presented in Section~\ref{sec:PINNs_Benchmarks}, including the 2D and 3D Helmholtz problems (Section~\ref{sec:Helmholtz_Problem}), the inviscid Burgers equation (Section~\ref{sec:Inviscid_Burgers_Equation}), the Euler equations (Section~\ref{sec:euler_equation}), Stokes flow (Section~\ref{sec:Stokes_Flow}), and a stiff PK--PD ODE system (Section~\ref{sec:stiff_odes}). 
%

%
%
\section{Review of Considered Optimizers}\label{sec:Review_of_Second-order_Optimizers}

In this section, we briefly describe the underlying principles as well as implementations of each optimizer considered in our study. 
We focus on the core update rules, curvature approximations, and computational requirements, highlighting the practical differences most relevant to large-scale training and PINN applications. 
Figure~\ref{Fig:Schematic} provides an overview of the PINN training workflow and where optimization enters the pipeline. 
%
%
Let us consider the PDE problem
\begin{equation}
\begin{aligned}
    \mathcal{A}u &= f \quad \text{in } \Omega, \\
    u &= g \quad \text{on } \Gamma.
\end{aligned}
\end{equation}
where $\mathcal A$ is a PDE operator in strong form and $\Gamma\subseteq\partial\Omega$ denotes the boundary. For simplicity, we assumed Dirichlet boundary conditions in the formulation above. We understand this formulation to include both stationary and transient problems for appropriate choices of $\Omega$, $\Gamma$, and $\mathcal A$. For collocation points $x_i \in \Omega$ in the interior of the domain and $y_j \in \Gamma$ on the boundary, PINNs employ the training loss
\begin{equation}\label{eq:discrete_pinn_formulation}
\mathcal{L}(\btheta)
=
\frac{1}{2N_\Omega}\sum_{i=1}^{N_\Omega} \bigl(\mathcal{A}u_{\btheta}(x_i)-f(x_i)\bigr)^2
+
\frac{1}{2N_\Gamma}\sum_{j=1}^{N_\Gamma} \bigl(u_{\btheta}(y_j)-g(y_j)\bigr)^2.
\end{equation}
This is a nonlinear least-squares objective.
Here, \(\btheta \in \mathbb{R}^p\) denotes the vector of trainable parameters of the neural network \(u_{\btheta}\).
We consider iterative optimization methods, starting with the most widely known example, namely (stochastic) gradient descent:
\begin{equation}
    \btheta_{k+1} = \btheta_k - \eta_k \mathbf{g}_k,
    \qquad
    \mathbf{g}_k \coloneqq \nabla_{\btheta}\mathcal{L}(\btheta_k).
\end{equation}
where $\eta_k>0$ denotes the stepsize, which is also referred to as the learning rate. 
It is well documented that such simple gradient-based optimizers perform inadequately for PINNs.  
This is the reason why we consider optimizers that incorporate curvature information, which we review in this section. 
All of these methods incorporate this curvature information via preconditioning and yield update rules of the form 
\begin{equation}
    \btheta_{k+1} = \btheta_k - \eta_k P_k^{-1}\mathbf{g}_k,
    \qquad
    \mathbf{g}_k \coloneqq \nabla_{\btheta}\L(\btheta_k),
\end{equation}
where $P_k$ is the preconditioner.
We consider optimizers that can be interpreted as fast and scalable approximations of one of the three preconditioners shown in
\Cref{tab:optimizers}.
\renewcommand{\arraystretch}{1.5}
\begin{table}[h!]
    \centering
    \begin{tabular}{|c|c|c|c|}
         \hline
         \textbf{Family} & \textbf{Preconditioner} & \textbf{Properties} & \textbf{Variants} \\ \hline
         AdaGrad & $\left(\sum_{i=1}^k \mathbf{g}_i \mathbf{g}_i^\top\right)^{1/2}$ & \makecell{Scale invariant\\Rotation invariant} & \makecell{Adam\\Shampoo\\SOAP} \\ \hline
         Newton's method & $\nabla^2 \mathcal{L}(\btheta_k)$ & Affinely invariant & \makecell{SSBFGS\\SSBroyden} \\ \hline
         \makecell{Gauss-Newton\\Natural 
         Gradient} & $J_k^\top J_k$ & \makecell{Reparameterization invariant\\Function-space interpretation} & \makecell{NG / GN} \\ \hline
    \end{tabular}
    \caption{Overview of the optimizers considered in this work. The listed properties refer to the idealized exact methods, rather than to the practical variants, which typically rely on additional approximations.}
    \label{tab:optimizers}
\end{table}

\subsection{AdaGrad methods}
Adaptive gradient methods (AdaGrad) were introduced in~\cite{duchi2011adaptive}. 
In their full-matrix form, they use the iteration
\begin{equation}
    \btheta_{k+1}
    =
    \btheta_k - \eta_k S_k^{-1/2}\mathbf{g}_k,
    \qquad
    S_k \coloneqq \sum_{i=1}^k \mathbf{g}_i \mathbf{g}_i^\top,
    \qquad
    \mathbf{g}_i \coloneqq \nabla_{\btheta}\L(\btheta_i).
\end{equation}
Here, $S_k$ is the accumulated gradient outer-product matrix.
To stabilize and normalize the preconditioner, it is common to replace this sum by an exponentially weighted moving average, which gives
\begin{equation}
    \btheta_{k+1}
    =
    \btheta_k - \eta_k S_k^{-1/2}\mathbf{g}_k,
    \qquad
    S_k
    \coloneqq
    \frac{1-\beta}{1-\beta^k}
    \sum_{i=1}^k \beta^{\,k-i}\mathbf{g}_i\mathbf{g}_i^\top,
    \qquad
    \mathbf{g}_i \coloneqq \nabla_{\btheta}\L(\btheta_i),
\end{equation}
for some parameter $\beta \in [0,1)$. This corresponds to a full-matrix RMSProp-style preconditioner~\cite{goodfellow2016deep,tieleman2017divide}.
The square root in the preconditioner has recently been controversially discussed in the literature~\cite{lin2024can}, as it weakens the connections to curvature matrices of the method. 

As applying the full preconditioner is computationally prohibitive, a number of fast and scalable approximations have been proposed in the literature. Here, we review three of the most prominent ones. The first is Adam, the workhorse of deep learning, which effectively uses a diagonal approximation of the AdaGrad preconditioner. More recent methods, such as Shampoo and SOAP, exploit the Kronecker structure induced by linear layers in neural-network parameterizations. These methods can therefore be viewed as extensions of Adam that move closer to preconditioning with the full square root of the gradient outer-product matrix. 

\paragraph{Adam:} 
Adam~\cite{kingma2014adam} and its more recent variant AdamW~\cite{loshchilov2018decoupled}
are among the most important optimizers in deep learning.
They can be viewed as AdaGrad-type methods that use exponential moving averages,
a diagonal approximation of the preconditioner, and momentum at the gradient level. This leads to the moment estimates
%
%
%
\begin{equation*}
\begin{split}
    \mathbf{m}_k
    &=
    \frac{1-\beta_1}{1-\beta_1^k}
    \sum_{i=1}^k \beta_1^{\,k-i}\,\mathbf{g}_i,
    \qquad
    \mathbf{v}_k
    =
    \frac{1-\beta_2}{1-\beta_2^k}
    \sum_{i=1}^k \beta_2^{\,k-i}\,
    (\mathbf{g}_i \odot \mathbf{g}_i),
\end{split}
\end{equation*}
where \(\mathbf{g}_k = \nabla_{\btheta}\mathcal{L}(\btheta_k)\), and \(\odot\) denotes the Hadamard
(entrywise) product. 

The AdamW update with learning rate \(\alpha_k\) and weight decay coefficient \(\lambda\) is
\begin{equation}
    \btheta_{k+1}
    =
    \btheta_k
    -
    \alpha_k
    \frac{\mathbf{m}_k}{\sqrt{\mathbf{v}_k}+\varepsilon}
    -
    \alpha_k \lambda \btheta_k ,
\end{equation}
where the square root and division are understood componentwise.
The original Adam update is recovered when \(\lambda=0\).
This update can be written in the generic preconditioned form
\begin{equation}
    \btheta_{k+1}
    =
    \btheta_k
    -
    \alpha_k P_k^{-1} \mathbf{m}_k
    -
    \alpha_k \lambda \btheta_k,
    \qquad
    \text{with }
    P_k = \operatorname{diag}\!\left((\mathbf{v}_k + \varepsilon)^{1/2}\right).
\end{equation}
%

This formulation highlights Adam-type methods as diagonally preconditioned gradient descent schemes, where each parameter is rescaled independently using running gradient statistics.
Since \(\mathbf{g}_k \odot \mathbf{g}_k = \operatorname{diag}(\mathbf{g}_k \mathbf{g}_k^\top)\), it follows that
\begin{equation*}
    \mathbf{v}_k
    =
    \frac{1-\beta_2}{1-\beta_2^k}
    \operatorname{diag}\!\left(
        \sum_{i=1}^k \beta_2^{\,k-i}\, \mathbf{g}_i \mathbf{g}_i^\top
    \right).
\end{equation*}
Hence, we can interpret Adam’s diagonal preconditioner as the diagonal of the exponential moving average of the gradient outer-product matrix $\mathbf{g}_k \mathbf{g}_k^\top$, where $\mathbf{g}_k=\nabla_{\btheta}\L(\btheta_k)$. While the diagonal adaptivity of Adam provides robustness and scalability, it neglects cross-parameter correlations and offers limited correction for intra-layer ill-conditioning.

\paragraph{Shampoo:}
Shampoo~\cite{gupta2018shampoo} generalizes diagonal preconditioning by exploiting the tensor structure of neural network parameters. It relies on two levels of approximation. First, the full preconditioner is approximated by a block-diagonal matrix, where each block corresponds to the parameters of a single layer of the network. Second, each block is approximated using a Kronecker-product structure. In this way, Shampoo captures richer within-layer correlations than diagonal methods such as Adam, while remaining substantially more scalable than a dense preconditioner.

For the \(l\)-th layer, let \(S_k^{\,l}\) denote the diagonal block of \(S_k\) associated with the parameters of that layer. We denote the layer parameters by \(W_l\), and let \(\mathbf{g}_k^{\,l}=\nabla_{W_l}\mathcal{L}(\btheta_k)\) be the corresponding gradient. Let \(G_k^{\,l}\) denote the reshaped form of \(\mathbf{g}_k^{\,l}\) as a matrix, and define
\[
L_k^{\,l}=G_k^{\,l}(G_k^{\,l})^\top,
\qquad
R_k^{\,l}=(G_k^{\,l})^\top G_k^{\,l}.
\]
Shampoo then approximates the layerwise preconditioner using a Kronecker-factored structure built from these two matrices. In particular, the square root of the block \(S_k^{\,l}\) is approximated by
\[
(S_k^{\,l})^{1/2}\approx (R_k^{\,l})^{1/4}\otimes (L_k^{\,l})^{1/4}
\]
and efficiently applied via
\[
    (S^l_k)^{-1/2}\mathbf{g}_k^l
    \approx
    (R_k^{\,l})^{-1/4}\otimes (L_k^{\,l})^{-1/4}\mathbf{g}_k^l
    =
    (L^l_k)^{-1/4} G_k^l (R^l_k)^{-1/4}.
\]
Additionally, an exponential moving average, as in the Adam algorithm is also used.
This yields a scalable approximation of second-order information while capturing within-layer correlations much more effectively than diagonal methods. The quality of this approximation in Shampoo remains an active topic of discussion~\cite{lin2025understanding}.




\paragraph{SOAP:}
SOAP, which stands for \emph{Shampoo with Adam in the Preconditioner Eigenbasis}, was introduced in~\cite{vyas2024soap} and addresses some of the numerical and practical limitations of Shampoo. Like Shampoo, it is based on the Kronecker-factored second-moment matrices \(L_k^{\,l}\) and \(R_k^{\,l}\) for each layer \(l\), but it modifies how these matrices are used in the update.

Rather than directly applying the inverse square roots \((L_k^{\,l})^{-1/2}\) and \((R_k^{\,l})^{-1/2}\) to the gradient, as in Shampoo, SOAP computes eigendecompositions of \(L_k^{\,l}\) and \(R_k^{\,l}\) and performs Adam-style adaptive updates in the corresponding rotated coordinate system. In this basis, the curvature approximation becomes approximately diagonal, which enables elementwise adaptive scaling similar to Adam. The resulting update is then mapped back to the original parameter space; for details, we refer to~\cite[Algorithm~3]{vyas2024soap}.

Thus, whereas Shampoo applies explicit Kronecker inverse-square-root preconditioning, SOAP can be interpreted as running an Adam-like update in a slowly varying eigenbasis determined by the same curvature model. Consequently, SOAP combines structured layerwise curvature information with the robustness of diagonal adaptivity, providing a practical compromise between first-order stability and structured geometric awareness.


\subsection{Self-Scaled quasi-Newton methods}
Quasi-Newton methods are scalable approximations of Newton’s method that use gradient information to construct approximations of second-order curvature. At iteration \(k\), the parameters are updated according to
\begin{equation}\label{eq:qn_1}
    \btheta_{k+1} = \btheta_k - \alpha_k H_k \mathbf{g}_k,
\end{equation}
where
\[
\mathbf{g}_k \coloneqq \nabla_{\btheta}\mathcal{L}(\btheta_k)
\]
denotes the gradient of the loss, and \(H_k \approx \nabla_{\btheta}^2 \mathcal{L}(\btheta_k)^{-1}\) is an approximation of the inverse Hessian.
%
To achieve this, the updates of \(H_k\) are designed to preserve symmetry and positive definiteness while satisfying the secant condition
\begin{equation}
    H_{k+1}\mathbf{y}_k = \mathbf{s}_k,
\end{equation}
where
\[
\mathbf{s}_k \coloneqq \btheta_{k+1} - \btheta_k,
\qquad
\mathbf{y}_k \coloneqq \mathbf{g}_{k+1} - \mathbf{g}_k.
\]
Given the matrix \(H_k\) at iteration \(k\), there are many possible update formulas that produce \(H_{k+1}\) while enforcing the secant condition. A broad class of such updates is given by the \emph{self-scaled Broyden} family~\cite{urban2025unveiling,kiyani2025optimizing,pkpd2025}.

\begin{align}\label{eq:ssbroyden_Hk}
    H_{k+1}
    &=
    \frac{1}{\tau_k}
    \left[
        H_k
        - \frac{H_k \mathbf{y}_k \mathbf{y}_k^\top H_k}{\mathbf{y}_k^\top H_k \mathbf{y}_k}
        + \phi_k \mathbf{v}_k \mathbf{v}_k^\top
    \right]
    +
    \frac{\mathbf{s}_k \mathbf{s}_k^\top}{\mathbf{y}_k^\top \mathbf{s}_k},
    \\
    \mathbf{v}_k
    &=
    \left(\mathbf{y}_k^\top H_k \mathbf{y}_k\right)^{1/2}
    \left[
        \frac{\mathbf{s}_k}{\mathbf{y}_k^\top \mathbf{s}_k}
        -
        \frac{H_k \mathbf{y}_k}{\mathbf{y}_k^\top H_k \mathbf{y}_k}
    \right].
\end{align}
where \((\tau_k,\phi_k)\) are scalar parameters that may vary with the iteration index \(k\). By inverting \eqref{eq:ssbroyden_Hk}, one obtains the corresponding update formula for the matrices \(B_k \equiv H_k^{-1}\), for example by repeated application of the Sherman--Morrison formula:
\begin{align}
    \mathbf{B}_{k+1}
    &=
    \tau_k \left[
        \mathbf{B}_k
        - \frac{\mathbf{B}_k \mathbf{s}_k \mathbf{s}_k^\top \mathbf{B}_k}{\mathbf{s}_k^\top \mathbf{B}_k \mathbf{s}_k}
        + \eta_k \mathbf{w}_k \mathbf{w}_k^\top
    \right]
    + \frac{\mathbf{y}_k \mathbf{y}_k^\top}{\mathbf{y}_k^\top \mathbf{s}_k},
    \\
    \mathbf{w}_k
    &=
    \left(\mathbf{s}_k^\top \mathbf{B}_k \mathbf{s}_k \right)^{1/2}
    \left[
        \frac{\mathbf{y}_k}{\mathbf{y}_k^\top \mathbf{s}_k}
        - \frac{\mathbf{B}_k \mathbf{s}_k}{\mathbf{s}_k^\top \mathbf{B}_k \mathbf{s}_k}
    \right],
    \\
    \eta_k
    &= \frac{1-\phi_k}{1+a_k\phi_k},
    \\
    a_k
    &=
    \frac{\left(\mathbf{y}_k^\top \mathbf{H}_k \mathbf{y}_k \right)
          \left(\mathbf{s}_k^\top \mathbf{B}_k \mathbf{s}_k\right)}
         {\left(\mathbf{y}_k^\top \mathbf{s}_k \right)^2}
    - 1.
\end{align}
If $\tau_k = 1$, then the choice of $\phi_k$ recovers some of the most well-known quasi-Newton methods. In particular, if $(\tau_k,\phi_k)=(1,1)$, then the update reduces to the BFGS algorithm~\cite{broyden1970convergence,fletcher1970BFGS,goldfarb1970BFGS,shanno1970BFGS}, whereas if $(\tau_k,\phi_k)=(1,0)$, it reduces to the DFP algorithm~\cite{davidon1959DFP,fletcher1963DFP}. By contrast, the term \emph{self-scaled} refers to the case $\tau_k \neq 1$. In this setting, the method can be interpreted as first scaling the matrix $\mathbf{H}_k$ by the factor $1/\tau_k$ and then applying the corresponding standard quasi-Newton update, such as BFGS or DFP. For further motivation on the use of self-scaling in optimization, particularly in the context of PINNs, we refer the reader to~\cite{kiyani2025optimizing,albaali2014broydens,urban2025unveiling} and the references therein. A more detailed discussion is provided in Appendix~\ref{sec:appendix_Self_Scaled}.

\subsubsection*{Extensions to batched versions}
One limitation of quasi-Newton optimizers  especially (SSBFGS and SSBroyden) is that they are typically designed for full-batch training and do not naturally support stochastic gradient descent. This makes them less attractive for very large-scale, data-driven training. To address this, we propose a lazy stochastic quasi-Newton approach for problems that do not fit on a single GPU. Our stochastic variants of SSBFGS and SSBroyden are presented in Algorithm~\ref{alg:sQuasiNewton}. Earlier studies have investigated batch training using the L-BFGS optimizer \cite{bollapragada2018progressive}. 
However, these methods are not stable in the PINN setting, because incorporating the PDE residual into the loss increases the nonlinearity of the objective. 
This can destabilize the Hessian approximation due to large gradient variance. To implement any SSBFGS or SSBroyden algorithm in a batch training setting, 
the secant condition is often violated due to noise in the gradient. 
Therefore, we update the Hessian approximation only when the following condition is satisfied:
\begin{align}\label{eq:batch_condition}
y_{k-1}^\top s_{k-1} \ge \tau \|s_{k-1}\|^2,
\end{align}
where $\tau > 0$ is a small threshold. 

The curvature condition in Equation~\ref{eq:batch_condition} ensures that:
\begin{enumerate}
    \item the curvature along the step direction $s_{k-1}$ is positive, i.e.,
    \[
    y_{k-1}^{\top} s_{k-1} > 0;
    \]
    \item the quasi-Newton Hessian (or inverse Hessian) approximation remains positive definite;
    \item the resulting update is numerically stable.
\end{enumerate}

The condition in Equation~\ref{eq:batch_condition} is a modification of the classical curvature condition 
introduced by~\cite{nocedal2006numerical}, who showed that updating the Hessian requires the curvature condition
\[
y_k^\top s_k > 0.
\]
In practice, if this condition is violated, the update is either skipped or damped to ensure numerical stability. 
In the following, we state this condition formally as a theorem, referring to the condition in Equation~\ref{eq:batch_condition}.

\begin{theorem}\label{them:batch_th1eorem}
Let $f : \mathbb{R}^n \to \mathbb{R}$ be twice continuously differentiable and assume that its Hessian is uniformly positive definite, i.e., there exists a constant $\tau > 0$ such that
\begin{equation}
\nabla^2 f(x) \succeq \tau I
\quad \text{for all } x \in \mathbb{R}^n .
\label{eq:uniform_pd}
\end{equation}

For two successive iterates $x_{k-1}, x_k \in \mathbb{R}^n$, define
\begin{equation}
s_{k-1} := x_k - x_{k-1},
\qquad
y_{k-1} := \nabla f(x_k) - \nabla f(x_{k-1}).
\label{eq:sy_def}
\end{equation}

Then there exists a point
\[
\xi = x_{k-1} + \btheta s_{k-1}, \quad \text{for some } \btheta \in (0,1),
\]
such that
\begin{equation}
y_{k-1} = \nabla^2 f(\xi)\, s_{k-1}.
\label{eq:hessian_rep}
\end{equation}

Moreover, the curvature condition

\begin{equation}
y_{k-1}^{\top} s_{k-1} \ge \tau \, \|s_{k-1}\|^2
\label{eq:curvature_condition}
\end{equation}
holds.
\end{theorem}

The proof of the theorem is provided in Appendix~\ref{app:ssbfgs_thm}. 
Detailed pseudocode for the stochastic SSBFGS (sSSBFGS) and SSBroyden (sSSBroyden) implementations is given in Appendix~\ref{app:algo_sSSBFGS_sSSBroyden}. 
A computational experiment demonstrating the correctness of Algorithm~\ref{app:algo_sSSBFGS_sSSBroyden} is presented in Appendix~\ref{app:correctness_algorithm}.


\subsection{Natural Gradient Methods}\label{subsec:NGs}
NG methods were introduced in the context of online learning in~\cite{amari1998natural}. 
In the setting of PINNs, they coincide with a Gauss-Newton method applied to the nonlinear least-squares objective, and were shown to improve PINN accuracy by several orders of magnitude in~\cite{muller2023achieving}. 
In their naive form, they require solving a linear system whose size matches the number of trainable parameters. However, efficient implementations have been developed using Kronecker-factored approximations~\cite{dangel2024kronecker}, randomized numerical linear algebra, or direct exploitation of the low-rank structure of the preconditioner~\cite{guzman-cordero2025improving}. We cast the PINN loss function into the standard non-linear least squares form
\begin{equation}
\begin{split}
    \mathcal{L}(\btheta)
    &=
    \frac{1}{2N_\Omega}\sum_{k=1}^{N_\Omega} \bigl(\mathcal{A}u_{\btheta}(x_k)-f(x_k)\bigr)^2
    +
    \frac{1}{2N_\Gamma}\sum_{l=1}^{N_\Gamma} \bigl(u_{\btheta}(y_l)-g(y_l)\bigr)^2 \\
    &=
    \frac12 \|r(\btheta)\|^2,
\end{split}
\end{equation}
where the residual \(r:\mathbb{R}^p \to \mathbb{R}^{N_\Omega+N_\Gamma}\) is defined by
\begin{equation*}
\begin{aligned}
    r(\btheta)
    &=
    \bigl(r_\Omega(\btheta)^\top,\, r_\Gamma(\btheta)^\top\bigr)^\top
    \in \mathbb{R}^{N_\Omega+N_\Gamma}, \\
    \bigl(r_\Omega(\btheta)\bigr)_k
    &=
    \frac{1}{\sqrt{N_\Omega}}\bigl(\mathcal{A}u_{\btheta}(x_k)-f(x_k)\bigr),
    \qquad k=1,\dots,N_\Omega, \\
    \bigl(r_\Gamma(\btheta)\bigr)_l
    &=
    \frac{1}{\sqrt{N_\Gamma}}\bigl(u_{\btheta}(y_l)-g(y_l)\bigr),
    \qquad l=1,\dots,N_\Gamma .
\end{aligned}
\end{equation*}
We denote the Jacobian of the residual at $\btheta$ by $J=J(\btheta)$ or at iteration $k$ by $J_k = J(\btheta_k)$. It is composed of the Jacobians $Jr_\Omega(\btheta) = J_\Omega$ and $Jr_\Gamma(\btheta) = J_\Gamma$. The corresponding Gauss-Newton matrix is denoted by $G=G(\btheta)=J^\top J$ and is given by

\begin{align}\label{eq:gramian}
\begin{split}
    G(\btheta)_{ij} 
    &=
    J^\top_\Omega J_\Omega + J^\top_\Gamma J_\Gamma
    \\
    &=
    \frac{1}{N_\Omega}\sum_{k=1}^{N_\Omega} 
    D\mathcal A(u_\btheta)[ \partial_{\btheta_j}u_\btheta(x_k)]D\mathcal A(u_\btheta)[\partial_{\btheta_i} u_\btheta(x_k)]
    \\
    &+
    \frac{1}{N_\Gamma}\sum_{l=1}^{N_\Gamma}\partial_{\btheta_i} u_\btheta(y_l) \partial_{\btheta_j} u_\btheta(y_l),
\end{split}
\end{align}
where $D\mathcal A$ denotes the linearization of the differential operator $\mathcal A$. 
In practice, we use a Levenberg-Marquardt regularization, meaning we add a scaled identity onto $G$. The regularization strength $\lambda$ is either hand-tuned or adapted using a trust-region strategy. The resulting iteration is
\begin{equation}\label{eq:NG_paramspace}
    \btheta_{k+1} 
    =  
    \btheta_k - \alpha_k [G(\btheta_k) + \lambda_k I]^{-1} \nabla \L(\btheta_k)  
\end{equation}

\paragraph{Jacobi Scaling:}
Numerical stability in Gauss-Newton updates is frequently compromised by the ill-conditioning of the Gramian $K  = J^{\top}J + \lambda I$, which arises from the spectral heterogeneity of the Jacobian $J$ and the disparate scales of PINN residuals \cite{vandersluis1969condition}. 
To equilibrate the system, we define $D = \text{diag}(K)$ and compute the normalized operator $\tilde{K} = D^{-1/2} K D^{-1/2}$, which is known as \emph{Jacobi scaling}. 
This ensures $\text{diag}(\tilde{K}) = \mathbf{1}$, a transformation shown by van der Sluis to be nearly optimal for minimizing the condition number of symmetric positive definite matrices \cite{vandersluis1969condition}. 

\paragraph{Kernel trick:}
The dense matrix $G$ in \eqref{eq:gramian} is of shape $(P,P)$, where $P$ is the number of trainable weights in the neural network ansatz. The cubic cost $\mathcal O(P^3)$ solving a system involving $G$ at every iteration is prohibitive, even for moderate values of $P$. However, the structure $G=J^\top J$ with $J$ being of shape $(N, P)$ reveals that $G$ is of low rank, whenever the batch size $N=N_\Omega + N_\Gamma$ is small. This can be explicitly exploited using the push-through identity 
\begin{equation*}
    (J^\top J + \lambda I)^{-1}J^\top = J^\top(JJ^T + \lambda I)^{-1},
\end{equation*}
where the matrix to be inverted on the right-hand side above is of shape $(N, N)$. This means that the overall complexity is dominated by the matrix-matrix product and in the case $N\leq P$ is given by $\mathcal O(N^2P)$, hence linear in $P$. Note that the regime $N<P$ is typical for PINNs. 
The resulting iteration is
\begin{equation}\label{eq:NG_samplespace}
    \btheta_{k+1} 
    =  
    \btheta_k - \alpha_k J_k^\top[J_kJ_k^\top  + \lambda_k I]^{-1}  r_k. 
\end{equation}
Additionally, randomized numerical linear algebra and matrix-free implementations can be used~\cite{guzman-cordero2025improving}. 
However, we find that the kernel trick, which is exact rather than an approximation, is powerful enough to enable efficient computations for our computational experiments. Alternatively, the kernel trick admits an equivalent primal-dual interpretation that reformulates the linearized least-squares problem within the residual feature space~\cite{jnini2025dual}.

\paragraph{Adding momentum via SPRING:}
The recently proposed SPRING~\cite{goldshlager2024kaczmarz} algorithm accelerates the linear solve at every iteration of the Gauss-Newton method in the sense of Nesterov \cite{goldshlager2025sketch}.
We can rewrite \eqref{eq:NG_samplespace} as a minimization problem
\begin{equation*}
    J_k^\top[J_kJ_k^\top  + \lambda_k I]^{-1}  r_k
    =
    \operatorname{argmin}_\phi\left[ \|J_k\phi - r_k \|^2 + \lambda\|\phi\|^2 \right]
\end{equation*}
and introduce momentum via modifying the above to 
\begin{equation*}
    \phi_k = \operatorname{argmin}_\phi\left[ \|J_k\phi - r_k \|^2 + \lambda\|\phi - \mu\phi_{k-1}\|^2 \right]
\end{equation*}
for a momentum parameter $\mu\in[0,1]$. 
Writing out the optimality conditions yields the iteration
\begin{align}\label{eq:spring_samplespace}
    \begin{split}
    \phi_k &= \mu\phi_{k-1} + J_k^\top[J_kJ_k^\top  + \lambda_k I]^{-1}(r_k - \mu J_k\phi_{k-1})
    \\
    \btheta_{k+1} &= \btheta_k - \alpha_k\phi_k.
    \end{split}
\end{align}
Just like in the case without momentum, we use the push-through identity in the regime $N<P$. 

\paragraph{Functional optimization:} 
The use of the Gauss-Newton method was popularized under the name \emph{energy natural gradient} in \cite{muller2023achieving} for PINNs and variational problems.
The reason for this is that it the Gauss-Newton method is indeed a special case of a much larger family of optimizers, which are obtained as finite-dimensional discretizations of infinite-dimensional optimization algorithms~\cite{muller2024position}.   
From this general infite-dimensional viewpoint efficient algorithms have been developed for computational fluid dynamics~\cite{jnini2025gauss}, quantum many-body problems~\cite{armegioiu2025functional}. 

%

Although in the context of PINNs, NGs agree with the Gauss-Newton method, deriving from a function space perspective, has the benefit that it allows us to understand the function space dynamics of the optimizer. 
Indeed, it can be shown that NG offers a locally optimal approximation of the infinite-dimensional algorithm and hence mimic the Gauss-Newton method in function space~\cite{muller2023achieving,jnini2025dual}. 
Further, NG methods are up to discretization errors invariant under general reparametrizations~\cite{van2023invariance}, where Newton's method is affinely invariant and AdaGrad type methods are merely scale and rotation invariant. 

\section{Landscape Visualization} 
\label{sec:Landscape_Visualization}

We visualize the qualititative differences of the different families of optimizers in two ways. 
Firstly, we visualize the loss landscape in parameter space and how the optimizers traverse this landscape. 
Secondly, we visualize the updates of the optimizers in function space and contrast them to the error of the current predictions.

\subsection{Loss landscape in parameter space}
\begin{figure}[h!]
\begin{tikzpicture}
\matrix (m) [
    matrix of nodes,
    nodes={anchor=center},
    column sep=0.5cm,
    row sep=0.5cm
]{
    & Random projection &  First pair of singular vectors & Second pair of singular vectors \\
    \rotatebox{90}{1D Euler} 
        & \includegraphics[trim={12cm 0cm 12cm 0cm}, clip, width=0.3\linewidth]{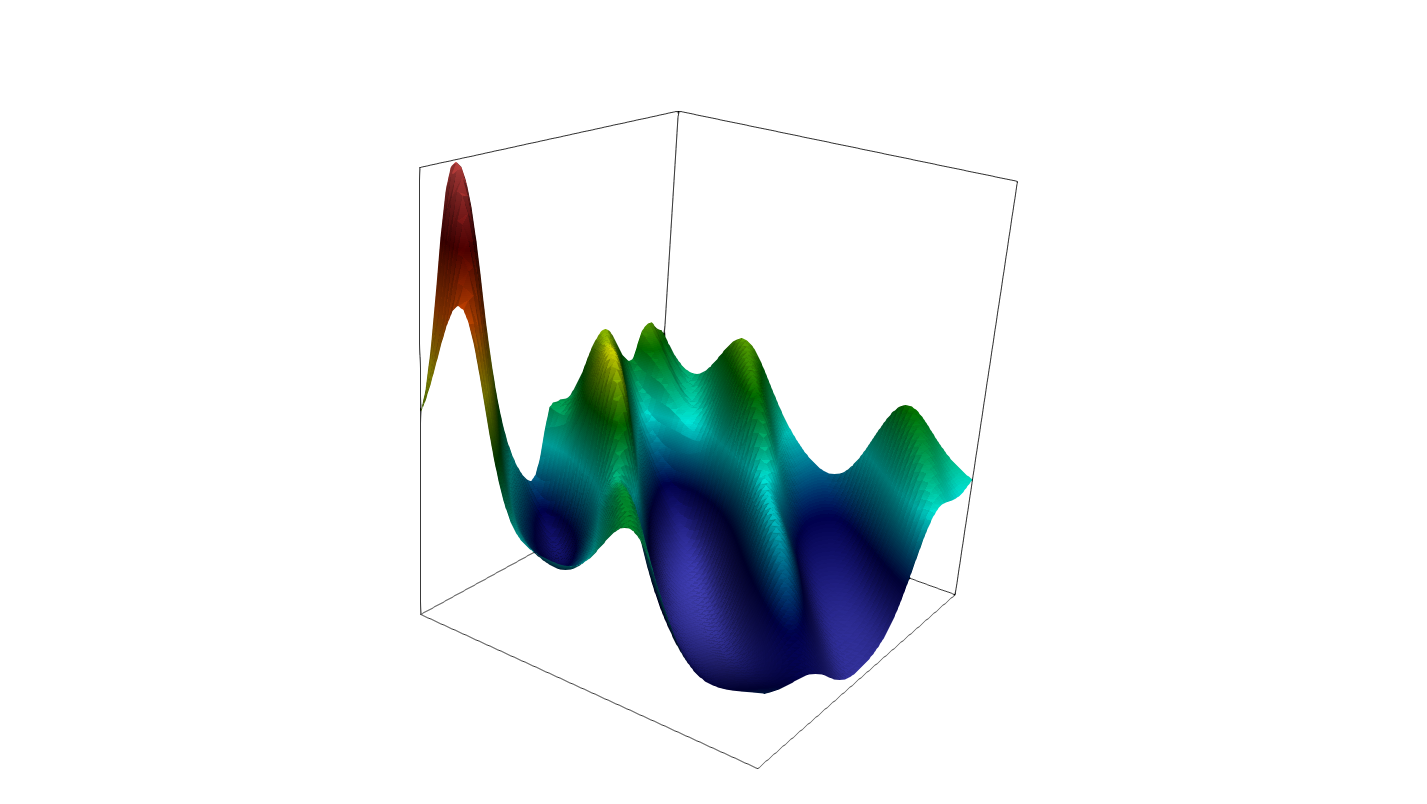}
        & \includegraphics[trim={8cm 0cm 8cm 0cm}, clip, width=0.3\linewidth]{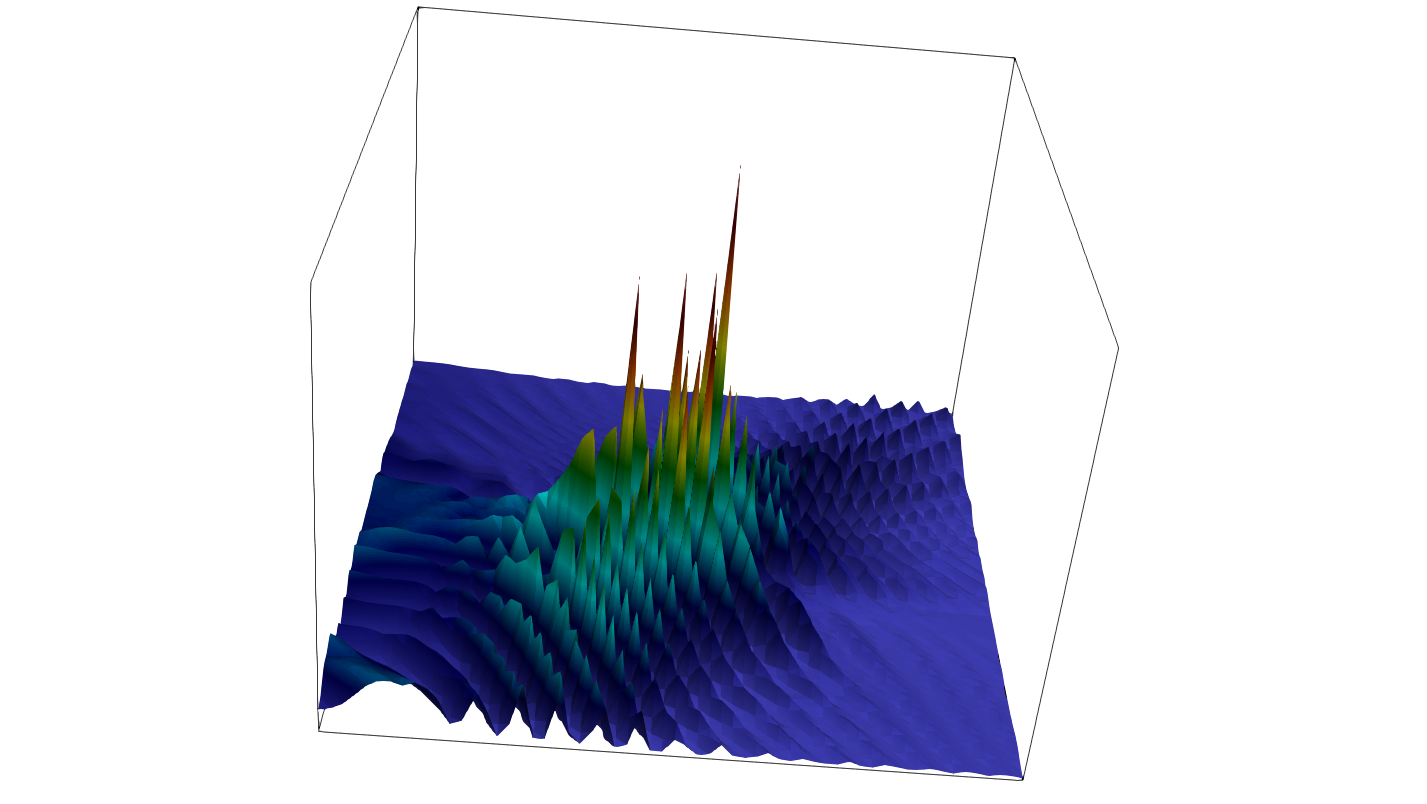}
        & \includegraphics[trim={6cm 0cm 6cm 0cm}, clip, width=0.3\linewidth]{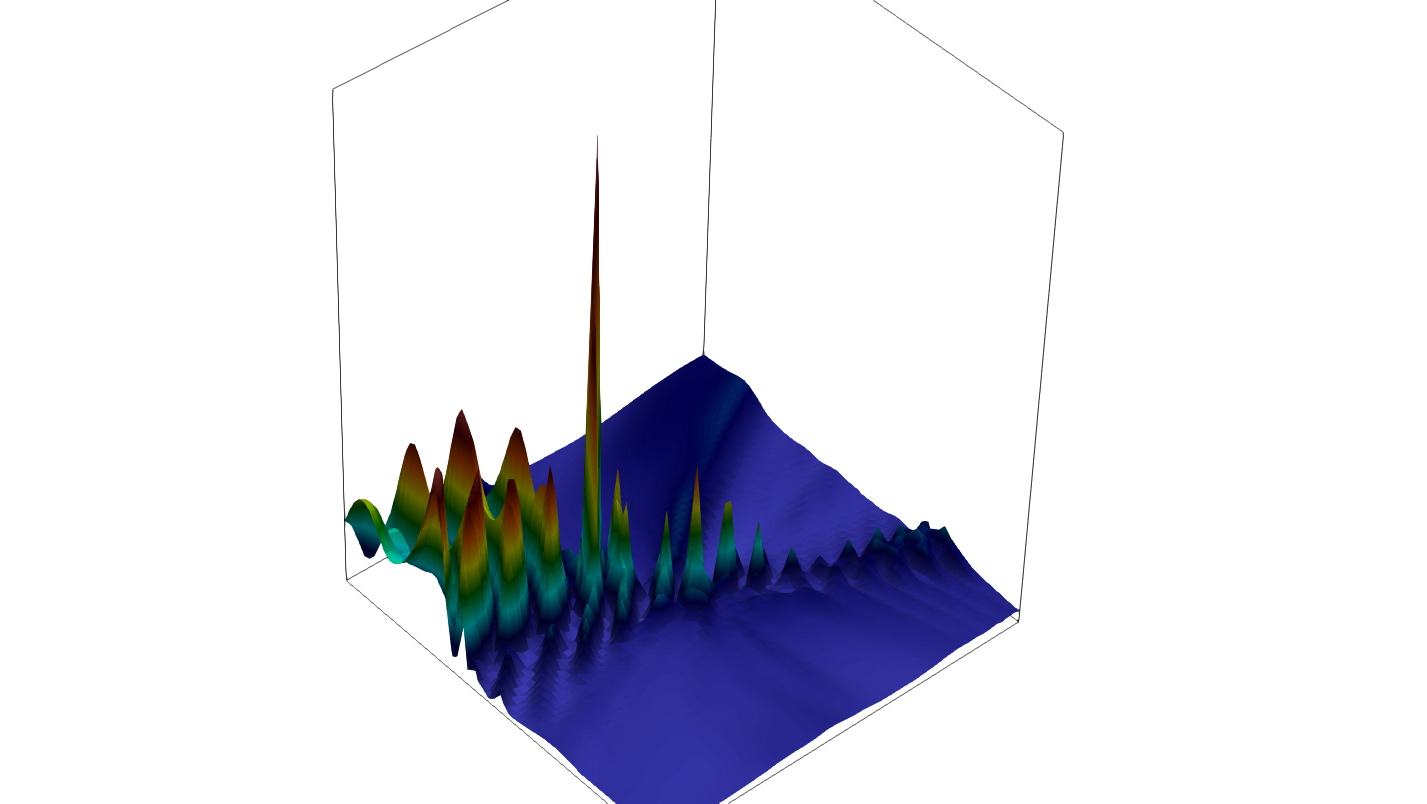} \\
    \rotatebox{90}{2D Stokes}
        & \includegraphics[trim={14cm 0cm 13cm 0cm}, clip, width=0.3\linewidth]{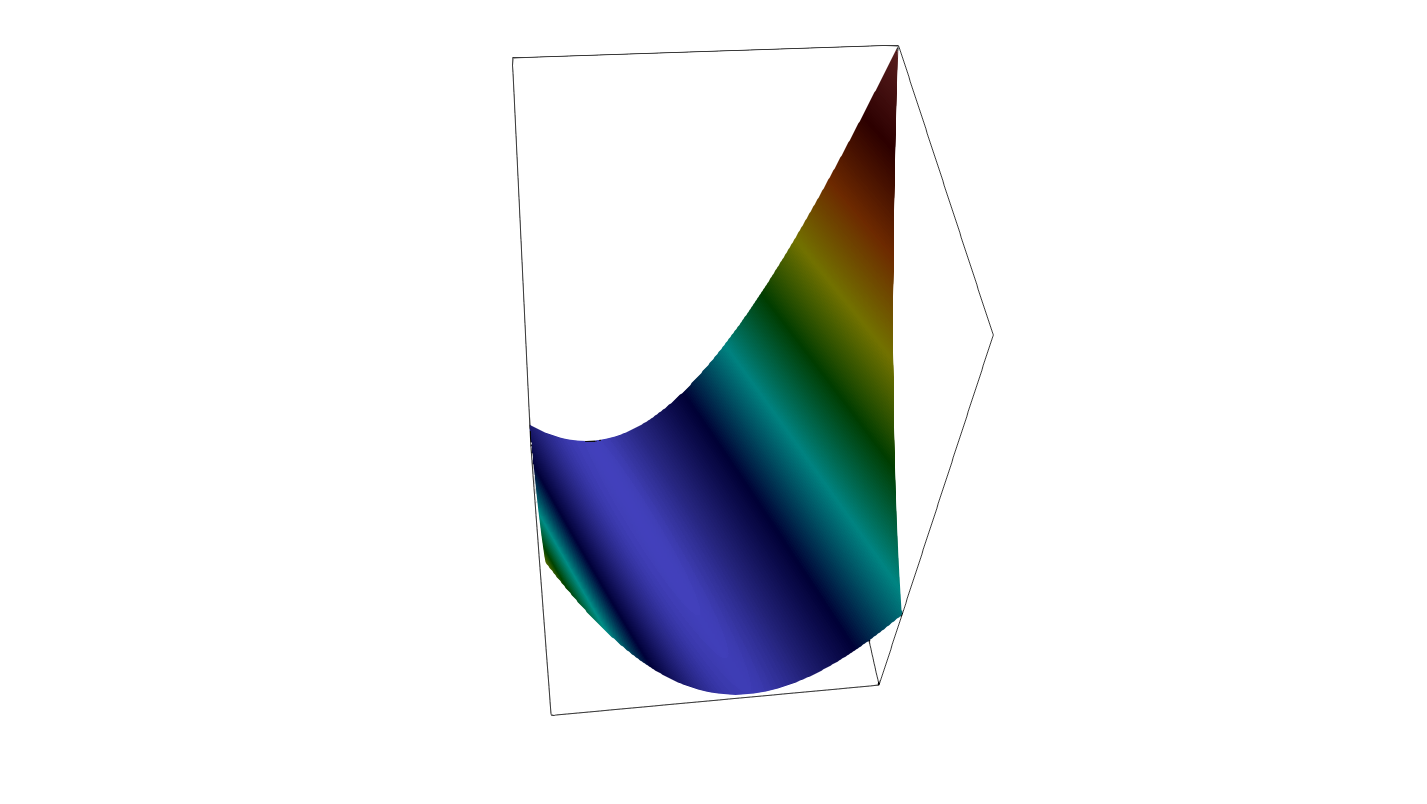}
        & \includegraphics[trim={12cm 0cm 12cm 2cm}, clip, width=0.3\linewidth]{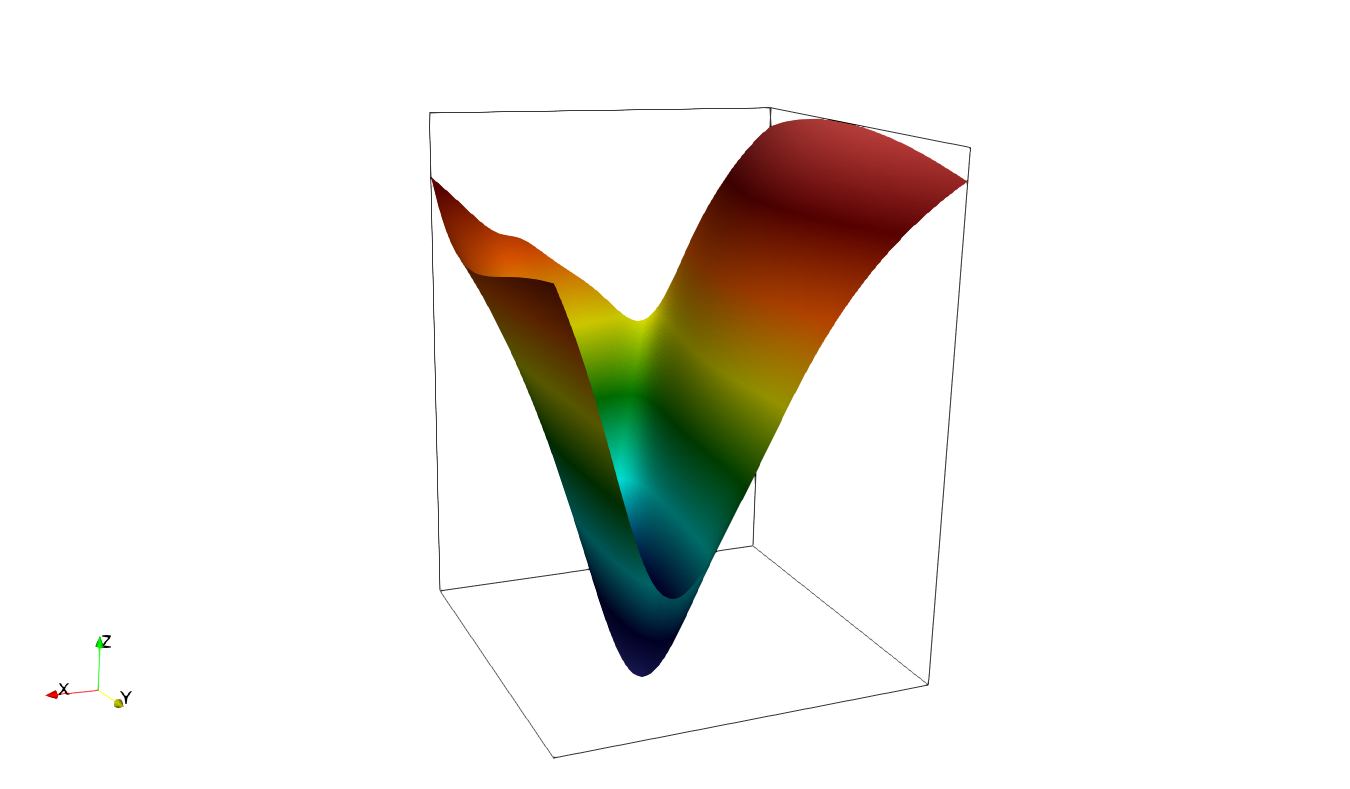}
        & \includegraphics[trim={14cm 4cm 14cm 4cm}, clip, width=0.3\linewidth]{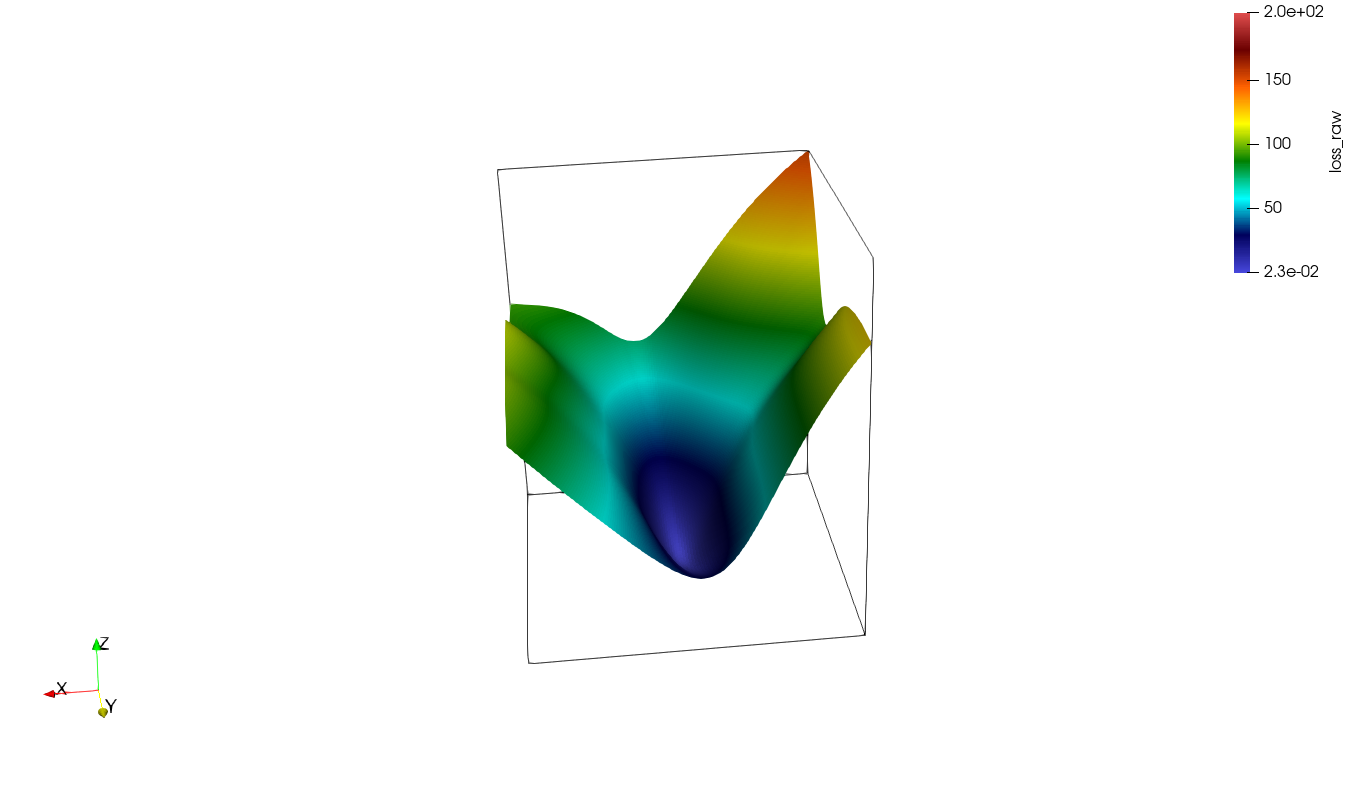} \\
};
\end{tikzpicture}
\caption{Loss landscapes for two representative PDEs are shown along random directions and projections onto pairs of singular vectors. (a) Top row: 1D Euler equation~\ref{eq:a_euler} (Sod problem), projected along orthogonal random directions, leading singular vectors, and next dominant pair (left to right). Bottom row: Stokes equation~\ref{eq:Stokes}, following the same projection scheme. Axes correspond to $(\alpha, \beta)$.}
\label{fig:loss_landscape}
\end{figure}

To analyze the geomtery of the optimization problem, we visualize a two-dimensional slice of the loss landscape around a reference parameter vector $\btheta_0$ for which we follow the methods proposed in~\cite{li2018visualizing}. 
Let $\mathcal{L}(\btheta)$ denote the total loss of the model. 
Since the parameter is space is high dimensional, we project the loss function onto a two-dimensional subspace spanned by two direction $v_1$ and $v_2$. 
Specifically, we evaluate the loss at the point of form $\mathcal{L}(\btheta_0 + \alpha v_1 + \beta v_2)$ where $(\alpha, \beta)$ are scalar cordinates defining the grid over the plane. 
To obtain the informative direction, we compue the top right-singular vector of the matrix per-sample gradients. let $l_i(\btheta)$ denote the loss contribution from the $i-$th training sample. 
We define the gradient matrix as follows
\[
\bm{G} =
\begin{bmatrix}
\nabla_\btheta \ell_1(\btheta_0)^\top \\
\nabla_\btheta \ell_2(\btheta_0)^\top \\
\vdots \\
\nabla_\btheta \ell_N(\btheta_0)^\top
\end{bmatrix}
\in \mathbb{R}^{N \times P} ,
\]
where $P$ is the number of parameters and $N$ is the number of samples. 
The singular value decomposition of $G$ is computed as 
\[
    \bm{G}=\bm{U}\bm{\Sigma}\bm{V}^\top
\]
The first two columns of $V$, denoted $v_1$ and $v_2$, correspond to the principal directions capturing the largest variation in the gradient space \cite{li2018visualizing}. 
These directions define the plane used to visualize the loss landscape. 
We evaluate the loss on a grid of $(\alpha,\beta)$ values and generate contour plots to study the curvature and anisotropy of the optimization landscape around $\btheta_0$.
In Figure~\ref{fig:loss_landscape}, loss landscapes for two representative PDE problems are visualized along random directions and along projections onto selected pairs of singular vectors. The top row illustrates the loss landscape for the 1D Euler equation~\ref{eq:a_euler} corresponding to the Sod shock tube problem. The panels show projections along: (i) two orthogonal random directions sampled from a multivariate normal distribution, (ii) the leading pair of dominant singular vectors, and (iii) the next pair of dominant singular vectors, progressing from left to right. The bottom row presents analogous plots for the elliptic Stokes equation~\ref{eq:Stokes}, following the same projection scheme as used for the Euler equation in the top row. In all figures, the horizontal and vertical axes correspond to the coordinates $(\alpha, \beta)$, representing the projection coefficients along the chosen directions.The Euler problem exhibits a relatively rough loss landscape due to the presence of shocks. In contrast, the Stokes problem, which is elliptic in nature, results in a much smoother loss landscape.

\subsection{Loss landscape in function space}

\begin{figure}[h!]
\begin{tikzpicture}
\node (img1) {
    \includegraphics[trim={0 0 0 5cm},clip,width=\textwidth]{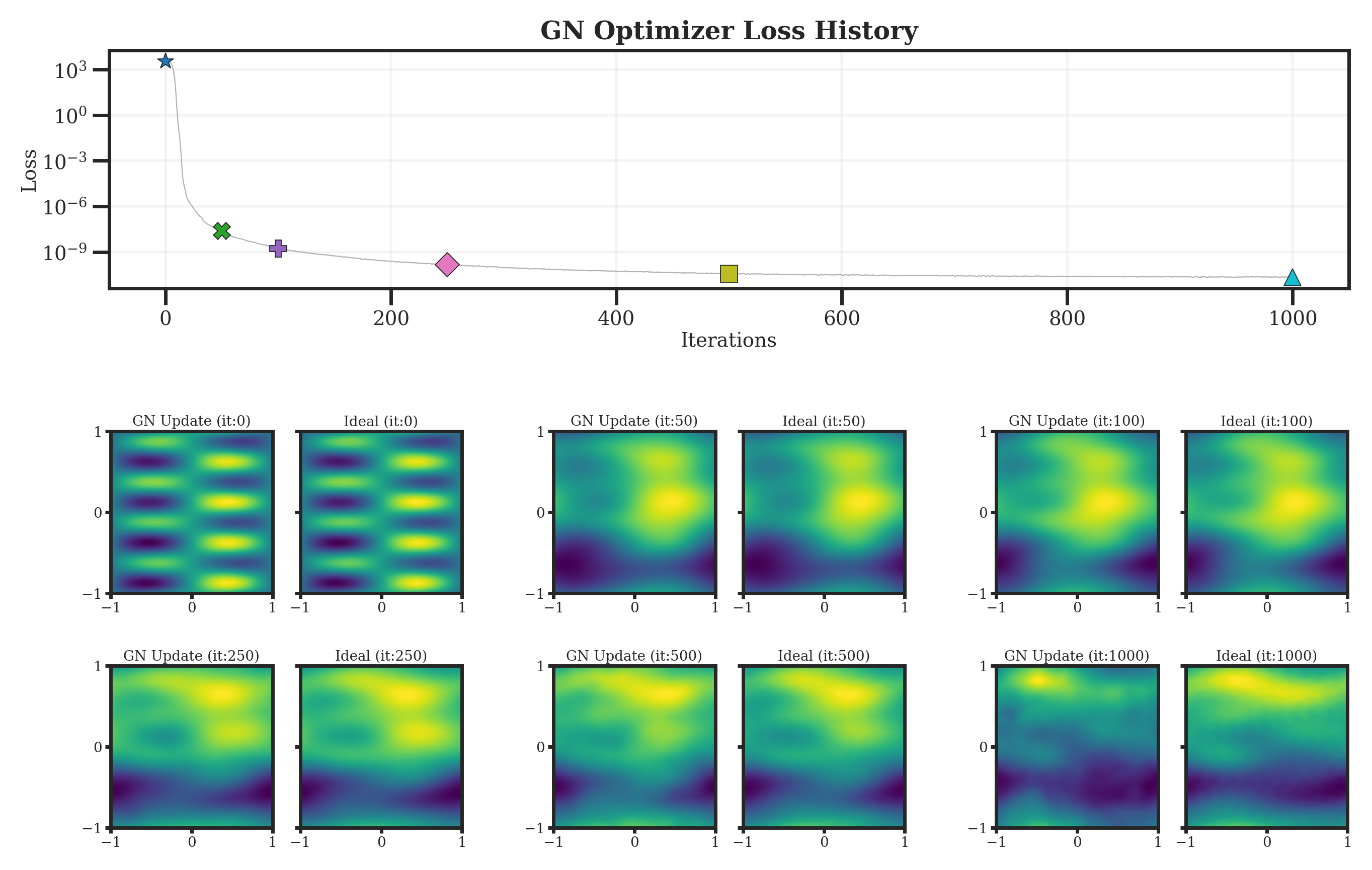}
};
\node[anchor=west] at (img1.west) {\rotatebox{90}{NG}};
\node (img2) [below=0.2cm of img1] {
    \includegraphics[trim={0 0 0 5cm},clip,width=\textwidth]{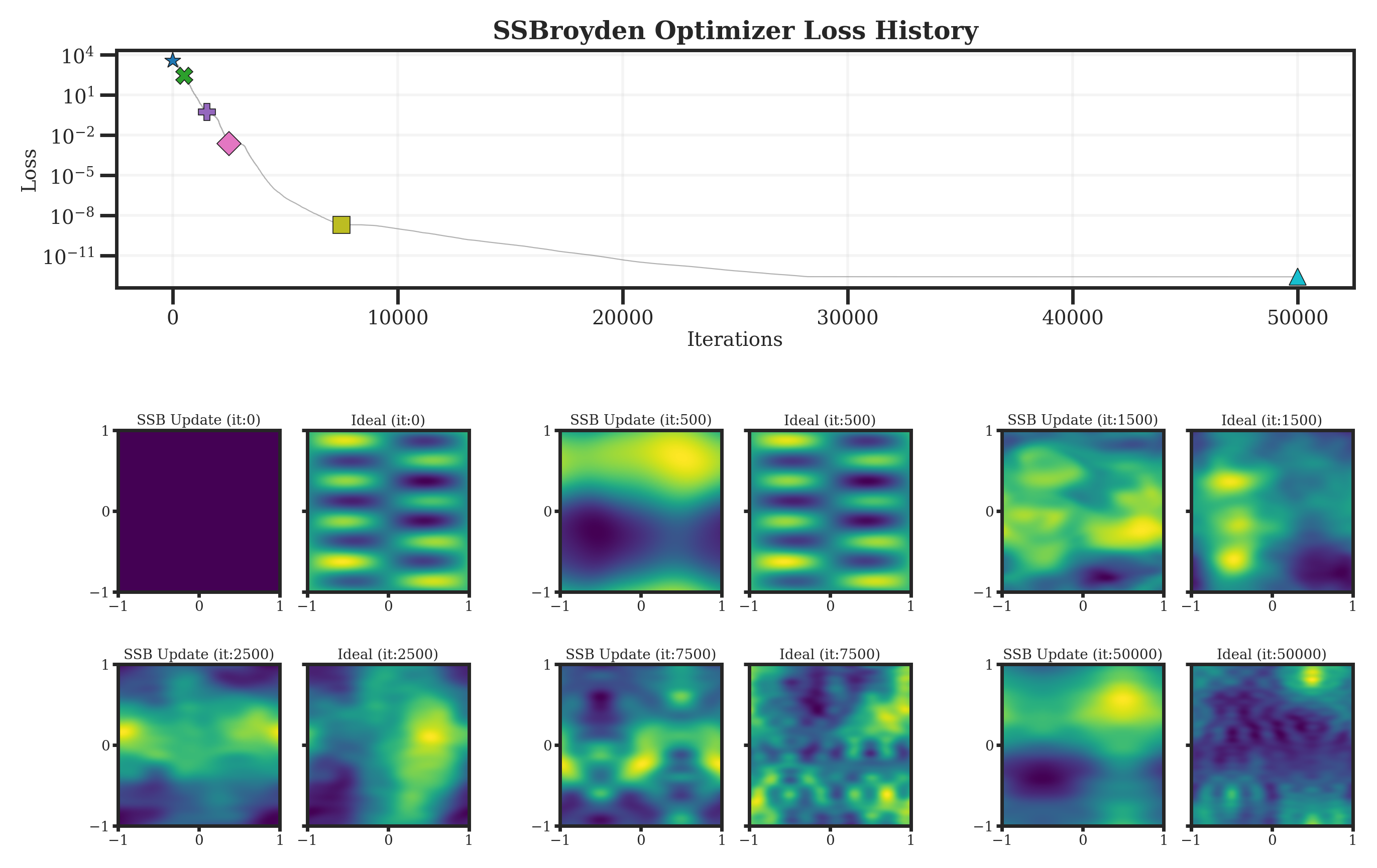}
};
\node[anchor=west] at (img2.west) {\rotatebox{90}{SSBroyden}};
\node (img3) [below=0.2cm of img2] {
    \includegraphics[trim={0 0 0 5cm},clip,width=\textwidth]{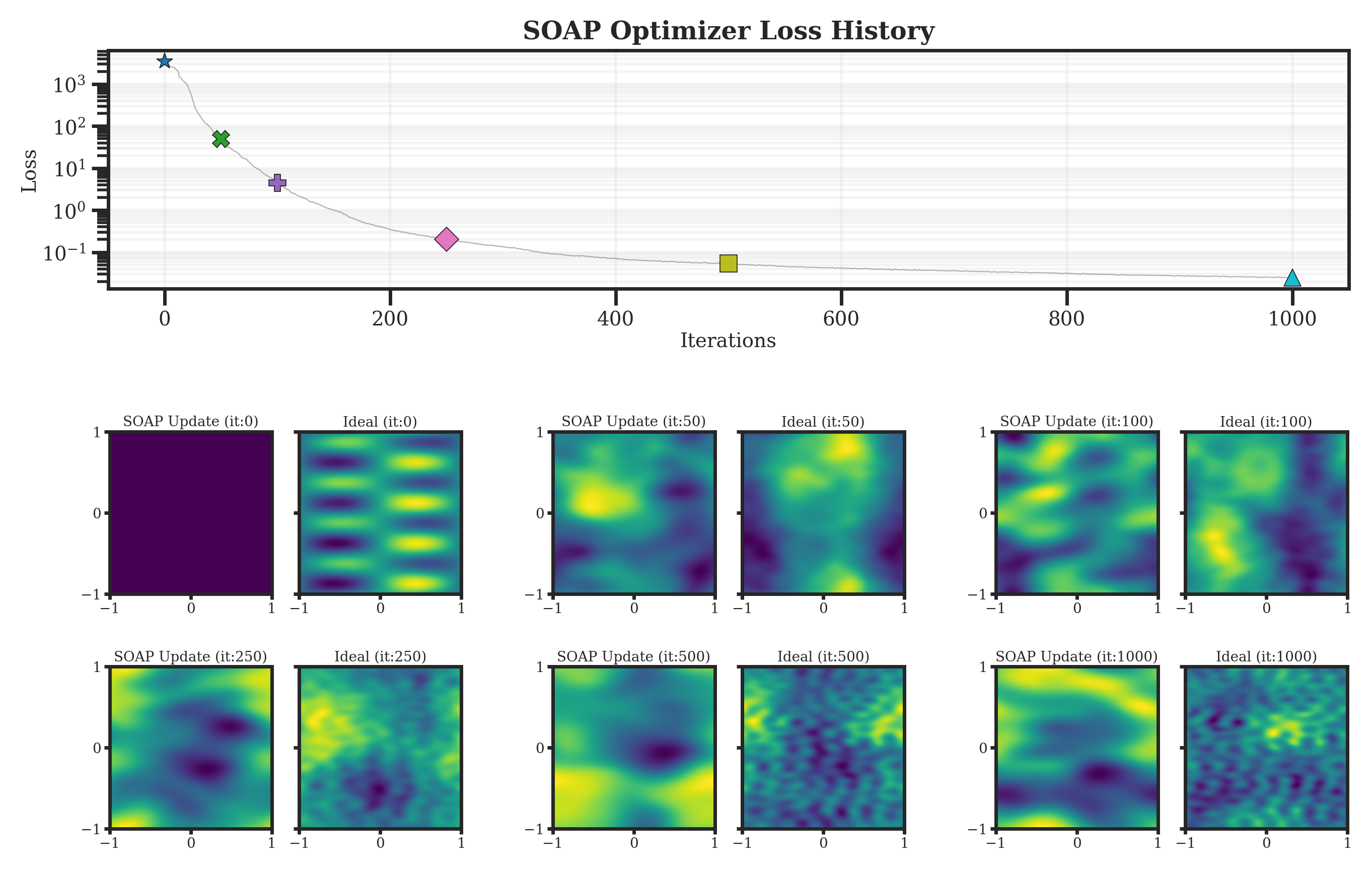}
};
\node[anchor=west] at (img3.west) {\rotatebox{90}{SOAP}};
\end{tikzpicture}
\caption{Visualization of the NG (top), SSBroyden (middle), and SOAP (bottom) training dynamics for the 2D Helmholtz problem with $a_1=1$, $a_2=4$, and $k=1$;
shown are side-by-side snapshot comparisons of the updates of the optimizers (left) and the current error of the prediction (right).}
\label{fig:pushforwards}
\end{figure}

All optimizers that we consider iteratively produce steps in parameter space $\Theta = \mathbb R^p$ given some initial parameters $\btheta_0$ and a rule to produce the direction $\delta_k\in \mathbb R^p$

\begin{equation*}
    \btheta_{k+1} = \btheta_k + \delta_k.
\end{equation*}
For instance, for the NG/Gauss-Newton optimizer this step is $\delta_k = (J_k^\top J_k + \lambda I)^{-1}J_k^\top r_k = J_k^\top (J_k J_k^\top + \lambda I)^{-1}r_k$, with appropriate adaptions for more advanced schemes like SPRING. Given a neural network architecture, which means we fixate a parametrization map 
\begin{equation}\label{eq:parametrization}
    P\colon\Theta \to \mathcal H, \quad \btheta \mapsto u_\btheta
\end{equation}
that takes parameters $\btheta$ and maps them to a neural network function $u_\btheta$ in an appropriate function space $\mathcal H$, we can visualize the \emph{pushforward} of $P$. The pushforward of a direction $\delta$ visualizes the update direction in \emph{function space}, an object that is often much more easy to interpret than $\delta$. It is given by 
\begin{equation}\label{eq:pushforward}
    DP(\btheta)[\delta](x) = \sum_{i=1}^p\delta_i\partial_{\btheta_i}u_\btheta(x).
\end{equation}
We compare the update directions for the various considered optimizers in Figure~\ref{fig:pushforwards}. It is clearly visible that the updates of the NG correct the error much better compared to the other methods; this aligns with its derivation from an infinite-dimensional algorithm, see~\ref{subsec:NGs} and \cite{muller2023achieving, muller2024position}.

\FloatBarrier

\section{Computational Experiments}
\label{sec:PINNs_Benchmarks}

This section presents a set of challenging benchmark problems for PINNs, chosen to assess optimizer robustness under stiff, oscillatory, multiscale, and shock-dominated loss landscapes. The benchmark suite includes the Helmholtz equation in two and three dimensions with varying wave numbers, the inviscid Burgers equation, the two-dimensional viscous Burgers equation, the compressible Euler equations, the two-dimensional Stokes equations, and a pharmacokinetic-pharmacodynamic (PK--PD) system, as summarized in Table~\ref{tab:placeholder}. These problems span regimes in which PINN training is known to be difficult due to strong nonlinearities, high-frequency oscillations, steep gradients, shocks, multiscale structure, and sensitivity to sampling and hyperparameter choices. 

As shown in our previous work~\cite{kiyani2025optimizing}, double precision consistently outperforms single precision for these problems. Therefore, unless otherwise stated, all results reported in this section are obtained using double precision. Readers interested in a detailed comparison between single- and double-precision performance are referred to~\cite{kiyani2025optimizing}.

It is worth noting that, in comparing the performance of the different optimizers across these benchmarks, we use several error metrics, including the relative $L^1$, relative $L^2$, and relative $L^\infty$ errors, depending on the problem and the available reference solution. For completeness, we define these metrics below. 

\paragraph{Metrics for Assessing Accuracy:}
\label{sec:Metrics}

Let $\{u_i^{\text{pred}}\}_{i=1}^N$ denote the predicted solution and $\{u_i^{\text{ref}}\}_{i=1}^N$ denote a reference solution evaluated at the same set of $N$ discrete points, where the reference may be an exact solution when available, or a high-resolution numerical solution otherwise.

The $L^1$ error is defined as
\begin{equation}
  \|u^{\text{pred}} - u^{\text{ref}}\|_1
  =
  \sum_{i=1}^N \bigl|u_i^{\text{pred}} - u_i^{\text{ref}}\bigr|,
\end{equation}
and the relative $L^1$ error is given by
\begin{equation}
  \mathrm{rel}_{L^1}
  =
  \frac{\|u^{\text{pred}} - u^{\text{ref}}\|_1}
       {\|u^{\text{ref}}\|_1}.
\end{equation}

The $L^2$ error is defined as
\begin{equation}
  \|u^{\text{pred}} - u^{\text{ref}}\|_2
  =
  \left( \sum_{i=1}^N \bigl(u_i^{\text{pred}} - u_i^{\text{ref}}\bigr)^2 \right)^{1/2},
\end{equation}
and the relative $L^2$ error is given by
\begin{equation}
  \mathrm{rel}_{L^2}
  =
  \frac{\|u^{\text{pred}} - u^{\text{ref}}\|_2}
       {\|u^{\text{ref}}\|_2}.
\end{equation}

The $L^\infty$ error is defined as
\begin{equation}
  \|u^{\text{pred}} - u^{\text{ref}}\|_\infty
  =
  \max_{1 \le i \le N} \bigl|u_i^{\text{pred}} - u_i^{\text{ref}}\bigr|,
\end{equation}
and the relative $L^\infty$ error is given by
\begin{equation}
  \mathrm{rel}_{L^\infty}
  =
  \frac{\|u^{\text{pred}} - u^{\text{ref}}\|_\infty}
       {\|u^{\text{ref}}\|_\infty}.
\end{equation}

The $L^\infty$ error is defined as
\begin{equation}\label{eq:infty}
 \|u^{\text{pred}} - u^{\text{ref}}\|_\infty
  = \max_{1 \le i \le N} \bigl|u_i^{\text{pred}} - u_i^{\text{ref}}\bigr|,
\end{equation}
and the relative $L^\infty$ error is given by
\begin{equation}\label{eq:rel_infty}
  \mathrm{rel}_{\infty}
  = \frac{\|u^{\text{pred}} - u^{\text{ref}}\|_\infty}
         {\|u^{\text{ref}}\|_\infty}
  = \frac{\max_{1 \le i \le N} \bigl|u_i^{\text{pred}} - u_i^{\text{ref}}\bigr|}
         {\max_{1 \le i \le N} \bigl|u_i^{\text{ref}}\bigr|}.
\end{equation}

\begin{table}[h!]
    \centering
    \begin{tabular}{|c|c|}
         \hline
         \textbf{Problem} & \textbf{Challenges}  \\ \hline 
         \ref{sec:Helmholtz_Problem} 2D Helmholtz  & Oscillatory solution \\ \hline 
        \ref{sec:Inviscid_Burgers_Equation} (1+1)D Inviscid Burgers & Shock waves \\ \hline
        \ref{sec:2D viscous Burgers Equation} (2+1)D viscous Burgers & Steep gradients / parabolic \\ \hline
        \ref{sec:euler_equation} (1+1)D Euler with Roe Flux &  shock and contact waves but dissipative  \\ \hline
        \ref{sec:Euler_equations_HLLC} (1+1)D Euler with HLLC Flux & Entropy solution, stability, and non-dissipative    \\ \hline
        \ref{sec:Stokes_Flow} 2D Stokes & Multiscale solution  \\ \hline
        \ref{sec:stiff_odes} Pharmacokinetics and -dynamics &  Stiffness and discontinuity \\ \hline
    \end{tabular}
    \caption{Overview over the computational experiments used in this study and the main sources of training difficulty for PINNs in each case.} 
    \label{tab:placeholder}
\end{table}

\subsection{2D and 3D Helmholtz equation}\label{sec:Helmholtz_Problem}
Consider the 2D Helmholtz equation on the square domain $\Omega = [x_0 = -1,x_f = 1]^2 \subset \mathbb{R}^2$
\begin{equation}
\label{eq:helmholtz}
\frac{\partial^2 u}{\partial x^2} \;+\; \frac{\partial^2 u}{\partial y^2} \;+\; k^2 u \;-\; q(x,y) \;=\; 0,
\end{equation}
where $q(x, y)$ is a forcing term
\begin{equation}
\begin{aligned}
q(x,y)
&= \Bigl(k^2 - (a_1\pi)^2 - (a_2\pi)^2\Bigr)\sin(a_1\pi x)\sin(a_2\pi y).
\end{aligned}
\end{equation}
By construction, the exact solution is
\begin{equation}
\label{eq:helmholtz_exact}
u(x,y) \;=\; \sin(a_1\pi x)\,\sin(a_2\pi y),
\end{equation}
which we use to validate model accuracy.
Periodic boundary conditions are imposed in both spatial directions. 
Simulations are performed for parameter pairs $(a_1, a_2) = (1,4), (6,6), (10,10)$ 
and wavenumbers $k = 1, 10,$ and $100$.
To enhance the representation of oscillatory solutions, Fourier feature extensions are employed for all cases in Helmholtz problem. The input to the neural network, denoted by $\mathbf{X}{\text{input}}$, is defined as
\begin{equation}\label{Fourier_modes}
\mathbf{X}{\text{input}} =
\left(
\cos\left(\frac{2\pi m x}{L_x}\right),
\sin\left(\frac{2\pi m x}{L_x}\right)
\right),
\end{equation}
where
\begin{equation}
L_x = x_f - x_0
\end{equation}
represents the spatial domain length and $m$ is the number of Fourier modes. 

\begin{figure}[t]
  \centering
  \includegraphics[width=\linewidth, height=0.99\textheight, keepaspectratio]{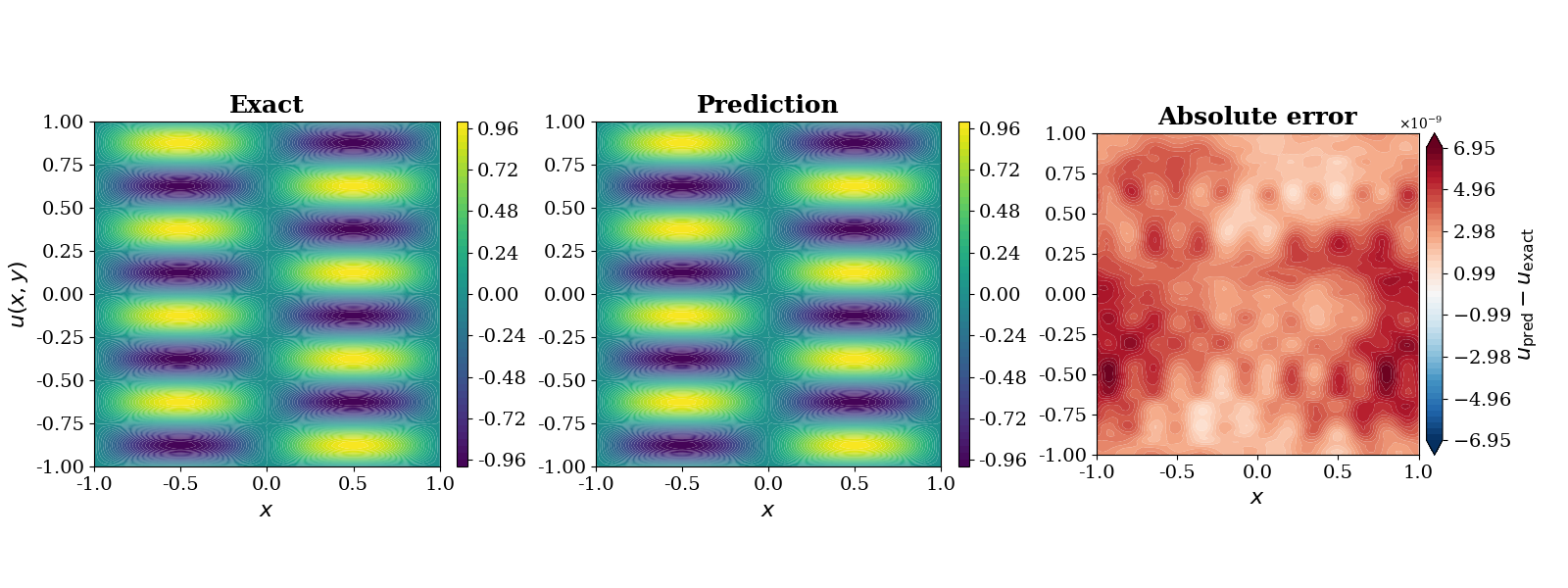}
\caption{\textbf{2D Helmholtz problem with $a_1=1$, $a_2=4$, and $k=1$.}
SSBroyden prediction, reference solution given by Equation~\eqref{eq:helmholtz_exact}, and the pointwise absolute error
$|u_{\mathrm{pred}} - u_{\mathrm{ref}}|$.}
  \label{fig:helmholtz1,4_solution}
\end{figure}

Figure~\ref{fig:helmholtz1,4_solution} compares the SSBroyden prediction with the reference solution~\eqref{eq:helmholtz_exact} for the 2D Helmholtz problem with $a_1=1$, $a_2=4$, and $k=1$. It also reports the corresponding pointwise absolute error, $|u_{\mathrm{pred}} - u_{\mathrm{ref}}|$. Figure~\ref{fig:Helmholtz_loss_relL2_1x4} further compares the training loss histories of BFGS, SOAP, SSBFGS, and SSBroyden, together with NG and SPRING, in the left panels, while the corresponding relative $L^2$ errors are shown in the right panels. Table~\ref{tab:Helmholtz_witha_1=1a_2=4} summarizes the $L^\infty$ and relative $L^2$ errors, as well as the training times.
In all reported cases, the Fourier feature mapping in~\ref{Fourier_modes} uses mode $m=1$, resulting in four input features, namely two sine/cosine components for each spatial coordinate. The table summarizes the performance of BFGS, SSBFGS, and SSBroyden in both the PyTorch and JAX implementations, as well as NG and SPRING. It shows that the SciPy-based implementations of SSBroyden and SSBFGS are as accurate as their Optax-based counterparts. The longer training times of the SciPy implementations are mainly due to the fact that they are executed on the CPU, whereas Optax, NG, and SPRING are run on the GPU.
The SSBroyden, BFGS, and SSBFGS (PyTorch) optimizers use $15{,}000$ interior collocation points, which are resampled every 500 epochs using the RAD adaptive sampling strategy. Cases~1--4 correspond to BFGS, SSBroyden, SSBFGS, and the PyTorch implementation of SSBFGS, respectively. All cases employ the same PINN architecture, consisting of four hidden layers with 30 neurons per layer.

Overall, the results demonstrate accuracy comparable to the JAX implementations. However, the PyTorch runs rely on SciPy routines executed on the CPU, which leads to substantially longer training times than the other methods. This comparison is included for readers interested in SciPy-based implementations. The SOAP optimizer is run with Adam-like momentum parameters $(\beta_1=\beta_2=0.95)$, a learning rate of $\mathrm{lr}=3\times10^{-3}$, a weight decay of $10^{-4}$, and a preconditioner update every 10 iterations.
\begin{figure}
    \centering
    \begin{minipage}[b]{\linewidth}
        \centering
        \includegraphics[width=1\linewidth]{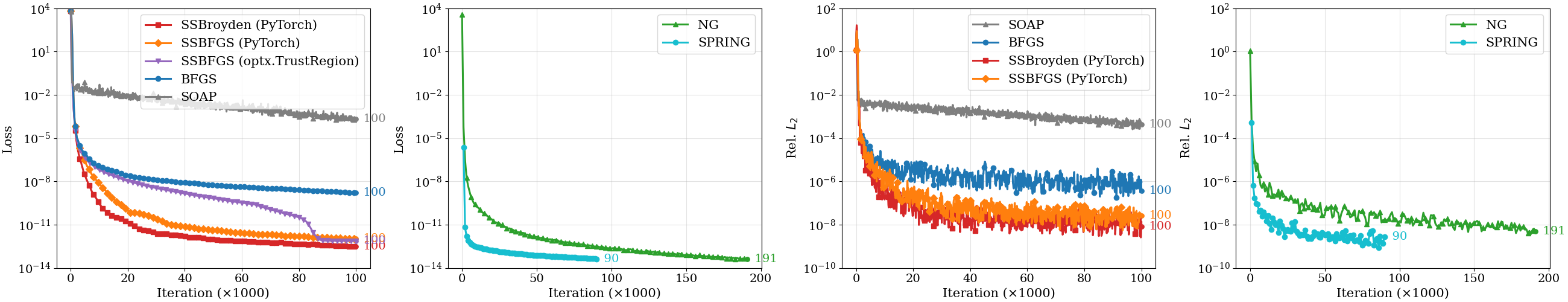}
        \caption{\textbf{2D Helmholtz problem with $a_1=1$, $a_2=4$, and $k=1$.}
        Training loss and relative $L^2$ error, comparing BFGS, SSBFGS, SSBroyden, NG, and SPRING. The SSBroyden (PyTorch) and SSBFGS (PyTorch) results are obtained using SciPy-based implementations executed on the CPU, while SSBFGS (JAX) is implemented using Optax in JAX.}
        \label{fig:Helmholtz_loss_relL2_1x4}
    \end{minipage}
    \vspace{1em}
    \begin{minipage}[b]{\linewidth}
      \centering
      \rowcolors{2}{cyan!10}{white}
      \scalebox{0.82}{
      \begin{tabular}{|l|r|l|r|r|r|r|}
        \hline
        \rowcolor{cyan!40}
        \textbf{Case} & \textbf{$k$} & \textbf{Optimizer} &
        \textbf{Relative $L_{\infty}$} & \textbf{Relative $L_2$} &
        \textbf{Training time (s)} & \textbf{\# parameters} \\
        \hline
        1  & 1 & BFGS (PyTorch)*              & \(4.2 \times 10^{-7}\) & \(3.7 \times 10^{-7}\) & 7066 & 2971 \\ \hline
        2  & 1 & SSBroyden (PyTorch)*         & \(6.9 \times 10^{-9}\) & \(7.7 \times 10^{-9}\) & 8215 & 2971 \\ \hline
        3  & 1 & SSBFGS (PyTorch)*            & \(2.0 \times 10^{-8}\) & \(2.6 \times 10^{-8}\) & 8322 & 2971 \\ \hline
        4  & 1 & SSBFGS (optx.TrustRegion)    & \(7.4 \times 10^{-9}\) & \(7.2 \times 10^{-9}\) & 564  & 2971 \\ \hline
        5  & 1 & SSBFGS (optx.Wolfe)          & \(5.1 \times 10^{-8}\) & \(5.8 \times 10^{-8}\) & 449  & 2971 \\ \hline
        6  & 1 & SSBroyden (optx.TrustRegion) & \(4.8 \times 10^{-9}\) & \(3.6 \times 10^{-9}\) & 319  & 2971 \\ \hline
        7  & 1 & SSBroyden (optx.Wolfe)       & \(5.9 \times 10^{-9}\) & \(5.3 \times 10^{-9}\) & 426  & 2971 \\ \hline
        9  & 1 & NG             & \(5.6 \times 10^{-9}\) & \(4.2 \times 10^{-9}\)  & 251  & 2971 \\ \hline
        10  & 1 & SPRING                       & \(2.4 \times 10^{-9}\) & \(2.0 \times 10^{-9}\)  & 540  & 2971 \\ \hline
        11 & 1 & SOAP                         & \(5.5 \times 10^{-4}\) & \(4.2 \times 10^{-4}\) & 1038 & 2971 \\ \hline
      \end{tabular}}
      \captionof{table}{\textbf{2D Helmholtz problem with $a_1=1$, $a_2=4$, and $k=1$.}
      All optimizers use the same fully connected network architecture with four hidden layers and 30 neurons per layer. The SSBroyden (PyTorch), BFGS, and SSBFGS (PyTorch) optimizers use $15{,}000$ interior collocation points, resampled every 500 epochs using the RAD adaptive sampling strategy, whereas optx.TrustRegion uses $25{,}000$ interior collocation points. The Fourier feature mapping in~\ref{Fourier_modes} uses mode $m=1$, resulting in four input features.}
      \label{tab:Helmholtz_witha_1=1a_2=4}
    \end{minipage}
\end{figure}

\begin{table}
    \centering
    \small
    \setlength{\tabcolsep}{3pt}

    \rowcolors{2}{cyan!10}{white}

    \resizebox{0.92\linewidth}{!}{%
    \begin{tabular}{|c|l|c|c|c|c|}
        \hline
        \rowcolor{cyan!45}
        \textbf{Case} & \textbf{Optimizer} & \textbf{\# collocation points} &
        \textbf{Relative $L_{\infty}$} & \textbf{Training time (s)} & \textbf{\# parameters} \\
        \hline
        1 & SSBFGS with TrustRegion & 10000 & $1.5 \times 10^{-8}$ & 287 & 2971 \\ \hline
        2 & SSBFGS with TrustRegion & 20000 & $7.1 \times 10^{-8}$ & 489 & 2971 \\ \hline
        3 & SSBFGS with TrustRegion & 25000 & $7.4 \times 10^{-9}$ & 564 & 2971 \\ \hline
    \end{tabular}}
    \caption{\textbf{2D Helmholtz problem with $a_1=1$, $a_2=4$, and $k=1$.}
    Relative $l_{\infty}$ error and training time for SSBFGS with different numbers of collocation points. All cases use the same PINN architecture with four hidden layers and 30 neurons per layer.}
    \label{tab:ssbfgs_jax_collocation_points}
\end{table}

Training time and solution accuracy strongly depend on the number of collocation points. Table~\ref{tab:ssbfgs_jax_collocation_points} illustrates this effect by reporting the relative $L_{\infty}$ metrics and training time for different choices of interior collocation points. In particular, three collocation-point settings are considered for the SSBFGS with TrustRegion method to highlight how the number of collocation points impacts optimizer performance.

\begin{figure}[!htb]
  \centering

  \begin{subfigure}[t]{0.90\linewidth}
    \centering
    \includegraphics[width=\linewidth]{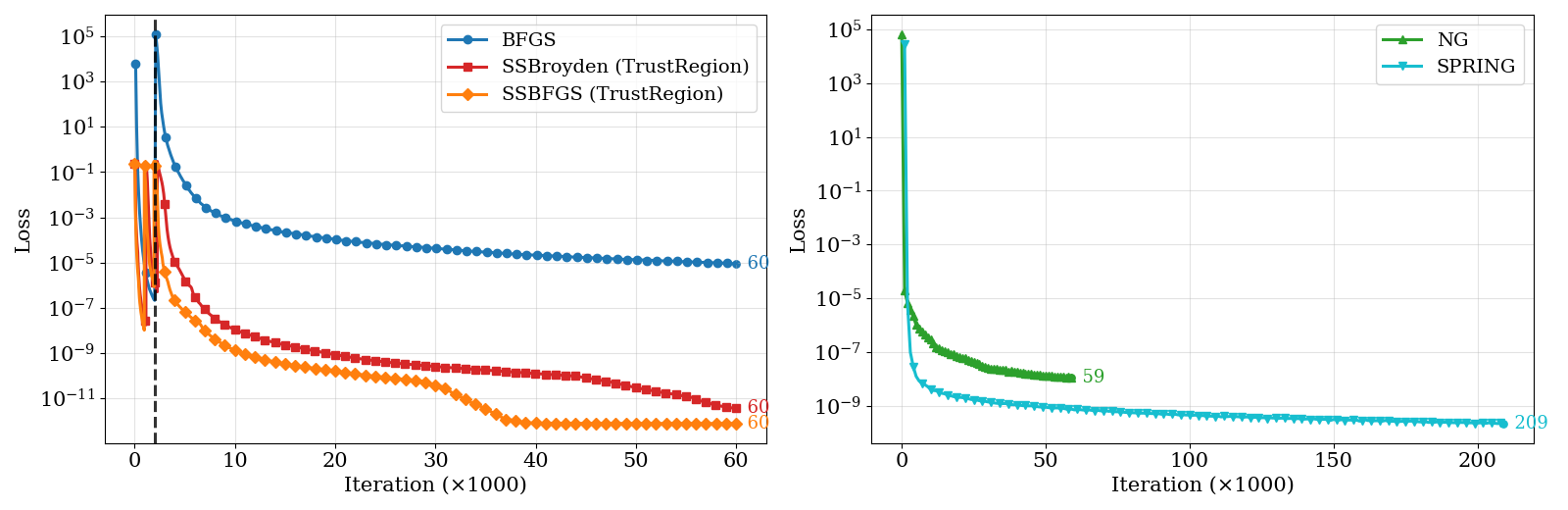}
    \caption{Training loss for $K=1$.}
    \label{fig:helmholtz_6x6_a}
  \end{subfigure}

  \vspace{0.8em}

  \begin{subfigure}[t]{0.90\linewidth}
    \centering
    \includegraphics[width=0.48\linewidth]{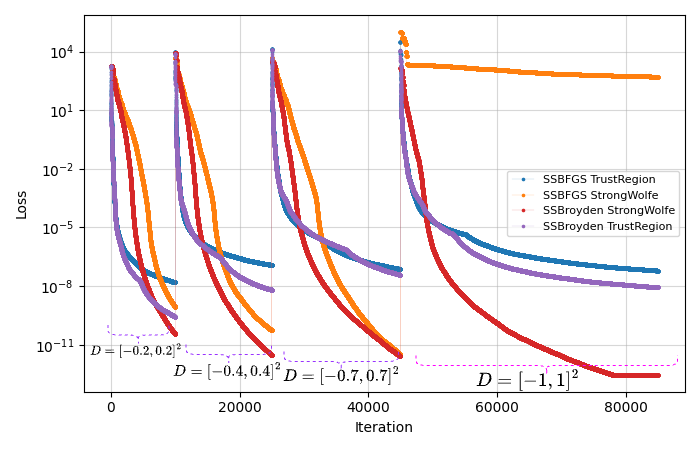}
    \hfill
    \includegraphics[width=0.48\linewidth]{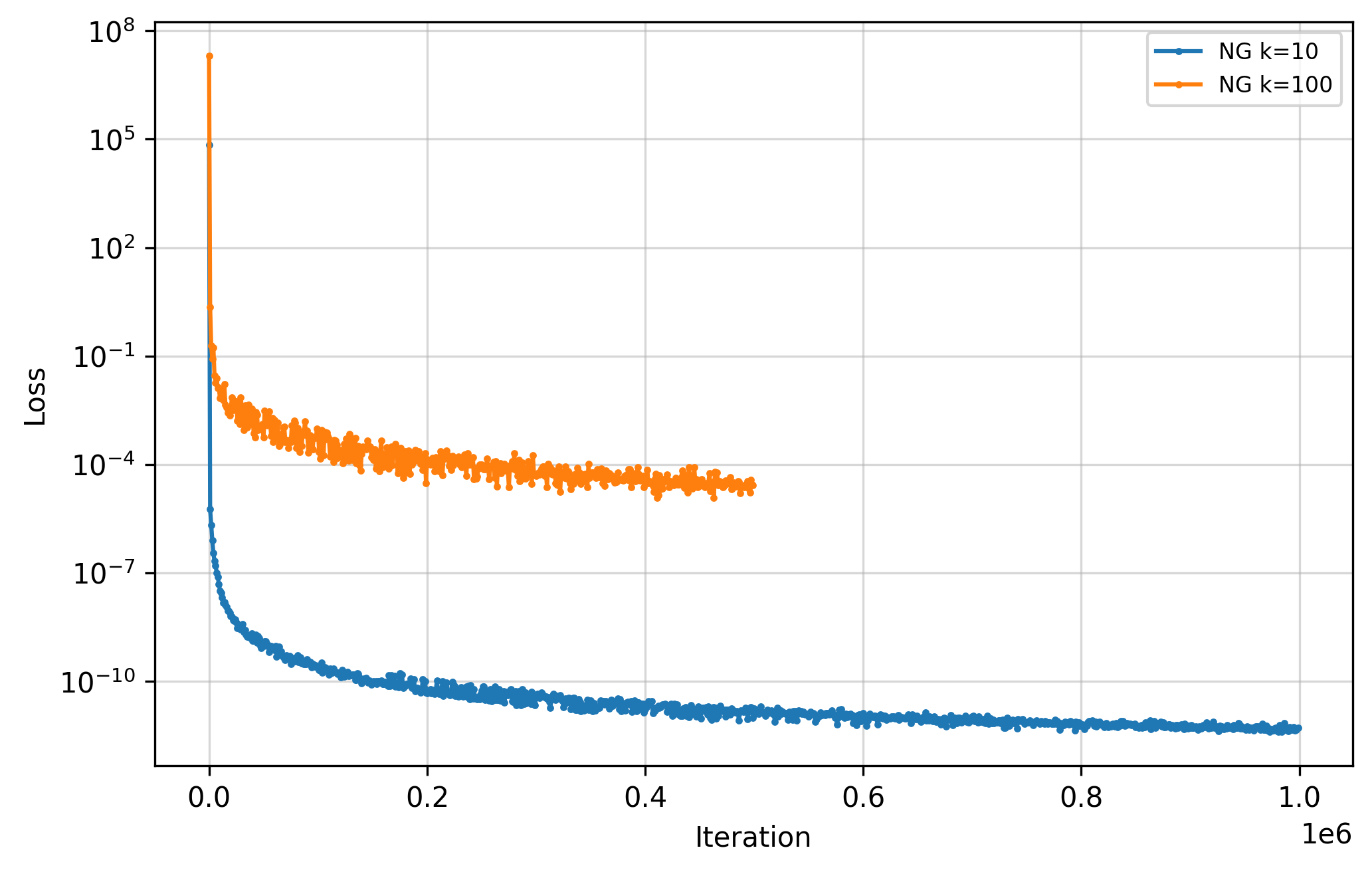}
    \caption{Loss history with curriculum domain expansion for $K=100$.}
    \label{fig:helmholtz_6x6_b}
  \end{subfigure}

  \caption{\textbf{2D Helmholtz problem with $a_1=6$ and $a_2=6$:}
  (a) $k=1$ and (b) $k=100$.
  During training for $k=100$, collocation points are sampled progressively from
  $[-0.2,0.2]^2$ for 10{,}000 iterations, $[-0.4,0.4]^2$ for 15{,}000 iterations,
  $[-0.7,0.7]^2$ for 20{,}000 iterations, and finally the full domain $[-1,1]^2$
  for 40{,}000 iterations.}
  \label{fig:helmholtz_6x6_optimizers}

  \vspace{0.7em}

  \begin{minipage}[b]{\linewidth}
    \centering
    \scalebox{0.82}{%
    \begin{tabular}{|l|c|l|c|c|c|c|}
      \hline
      \rowcolor{cyan!45}
      \textbf{Case} & \textbf{$k$} & \textbf{Optimizer} &
      \textbf{Relative $L_{\infty}$} & \textbf{Relative $L^2$} &
      \textbf{Training time (s)} & \textbf{\# parameters} \\
      \hline
      \rowcolor{kOneTint}
      1 & 1 & SSBFGS with TrustRegion        & $8.2\times 10^{-6}$  & $2.8\times 10^{-5}$ & 1127 & 4831 \\ \hline
      \rowcolor{kOneTint}
      2 & 1 & SSBroyden with TrustRegion     & $2.4\times 10^{-5}$  & $3.6\times 10^{-5}$ & 203  & 4831 \\ \hline
      \rowcolor{kOneTint}
      3 & 1 & NG                 & $6.1\times 10^{-6}$  & $2.8\times 10^{-5}$ & 944  & 4831 \\ \hline
      \rowcolor{kTenTint}
      4 & 10 & SSBFGS with TrustRegion       & $1.6\times 10^{-4}$  & $2.8\times 10^{-5}$ & 630  & 5911 \\ \hline
      \rowcolor{kTenTint}
      5 & 10 & SSBroyden with Wolfe         & $7.8\times 10^{-5}$  & $2.8\times 10^{-5}$ & 498  & 5911 \\ \hline
      \rowcolor{kTenTint}
      6 & 10 & NG               & $7.4\times 10^{-8}$  & $2.8\times 10^{-5}$ & 782  & 5911 \\ \hline
      \rowcolor{kHundredTint}
      7 & 100 & SSBFGS with TrustRegion     & $1.1\times 10^{-4}$  & $2.8\times 10^{-5}$ & 337  & 5911 \\ \hline
      \rowcolor{kHundredTint}
      8 & 100 & SSBroyden with TrustRegion  & $1.9\times 10^{-5}$  & $2.8\times 10^{-5}$ & 342  & 5911 \\ \hline
      \rowcolor{kHundredTint}
      9 & 100 & SSBroyden with Wolfe        & $3.4\times 10^{-7}$  & $2.8\times 10^{-5}$ & 393  & 5911 \\ \hline
      \rowcolor{kHundredTint}
      10 & 100 & NG              & $8.5\times 10^{-6}$  & $2.8\times 10^{-5}$ & 1561 & 5911 \\ \hline
    \end{tabular}}

    \captionof{table}{\textbf{2D Helmholtz with $a_1 = 6$ and $a_2 = 6$ ($k=1, 10$, and $100$).}
    All optimizers use the same fully connected network architecture with 6 hidden layers and 30 neurons per layer.}
    \label{tab:helmholtz_6x6_summary}
  \end{minipage}
\end{figure}

Figure~\ref{fig:helmholtz_6x6_optimizers} shows the training loss for the 2D Helmholtz problem with $a_1 = 6$ and $a_2 = 6$ for $k=1$ and $k=100$. 
For SSBFGS and SSBroyden at $k=100$, interior collocation points are introduced progressively using an expanding sampling region: 
$[-0.2,0.2]^2$ for $10{,}000$ iterations, 
$[-0.4,0.4]^2$ for $15{,}000$ iterations, 
$[-0.7,0.7]^2$ for $20{,}000$ iterations, 
and finally the full domain $[-1,1]^2$ for $40{,}000$ iterations. 
Table~\ref{tab:helmholtz_6x6_summary} summarizes the performance of SSBFGS, SSBroyden, and NG for these cases. 
For $k=1$, the Fourier feature mapping in~\ref{Fourier_modes} uses mode $m=1$, resulting in four input features, while for $k=10$ and $k=100$, mode $m=10$ is used. 
Across all cases, the same fully connected network architecture is employed, consisting of six hidden layers with 30 neurons per layer.

\begin{figure}[t]
  \centering
  \includegraphics[width=\linewidth, height=0.99\textheight, keepaspectratio]{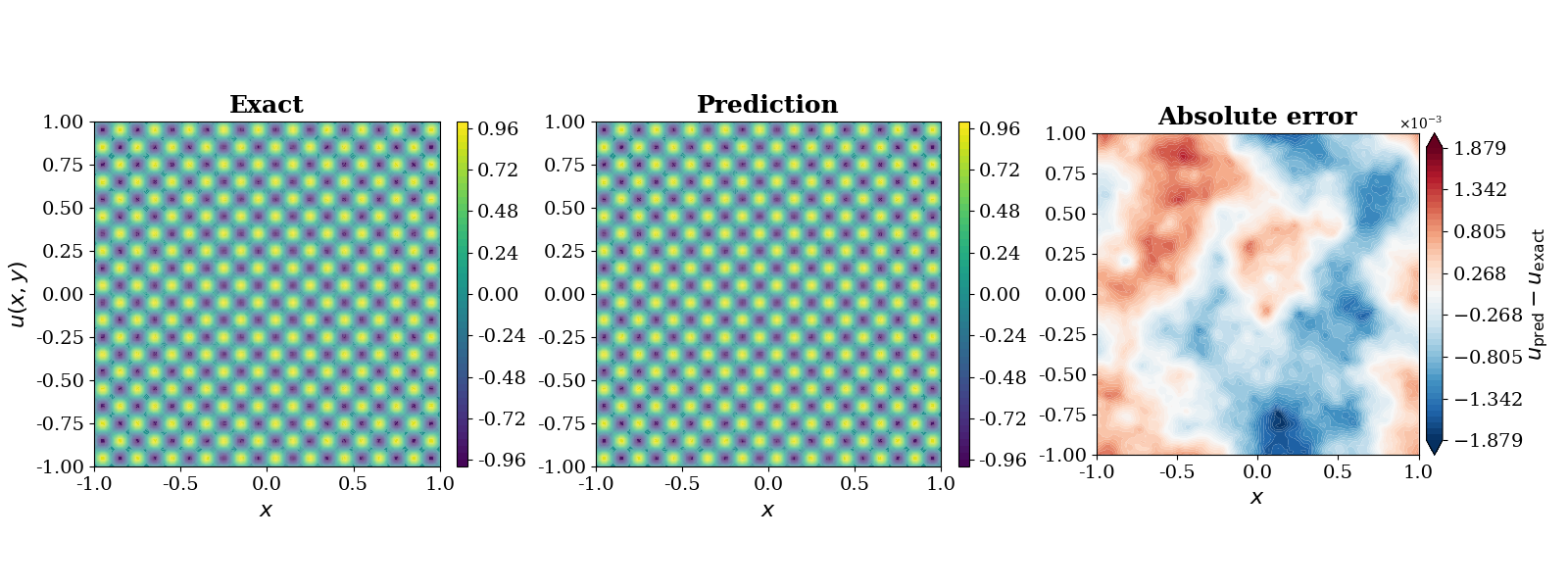}
  \caption{ \textbf{2D Helmholtz with $a_1=10$, $a_2=10$, and $k=1$:}
  SSBroyden prediction, exact solution~\eqref{eq:helmholtz_exact}, absolute error, and pointwise absolute error $|u_{\rm pred}-u_{\rm exact}|$.}
  \label{fig:Helmholtz_solution_6x6}
\end{figure}

\begin{figure}
    \centering
    \begin{minipage}[b]{\linewidth}
        \centering          
        \includegraphics[width=\linewidth]{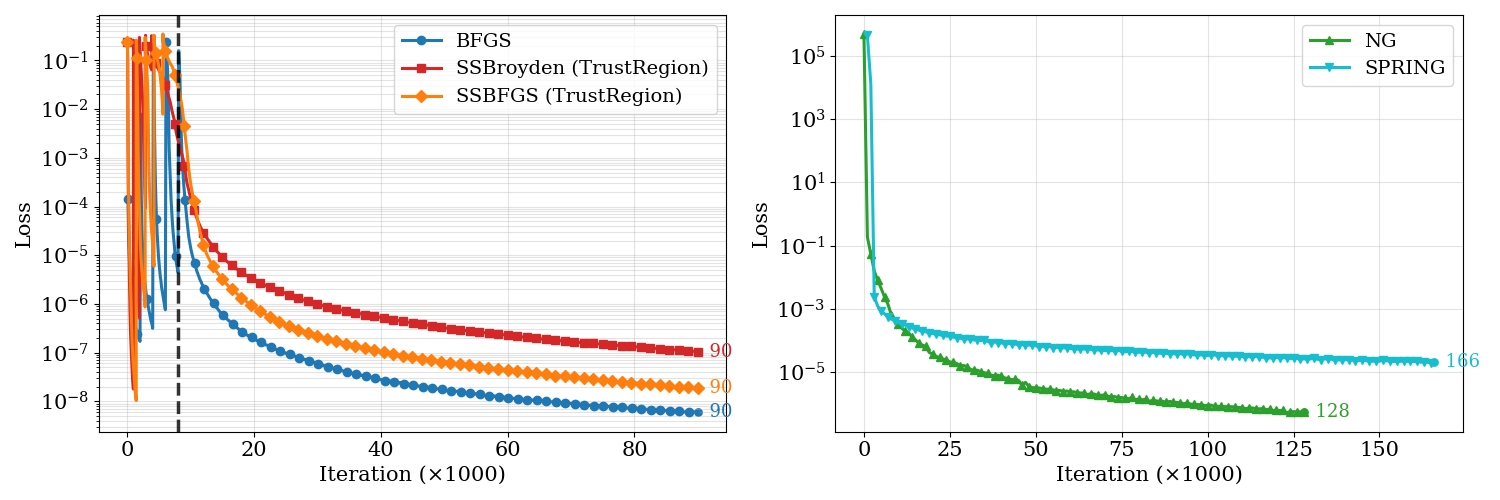}
        \caption{\textbf{2D Helmholtz with $a_1=10$, $a_2=10$, and $k=1$:}
  Loss and relative $L_2$ error for the Helmholtz problem with $a_1=10$ and $a_2=10$, comparing BFGS, SSBroyden, and NG descent.}

        \label{fig:}
    \end{minipage}
    \vspace{1em}
    \begin{minipage}[b]{\linewidth}
        \centering
        \rowcolors{2}{cyan!15}{white}
        \scalebox{0.85}{
        \begin{tabular}{|l|r|r|r|r|r|}
        \hline
        \rowcolor{cyan!40} 
        \textbf{Case} & \textbf{Optimizer} & \textbf{Relative $L_{\infty}$} & \textbf{Relative $L_{2}$} &  \textbf{Training time (s)} & \textbf{\# parameters} \\ 
        \hline
        1 & SSBFGS with TrustRegion & \(2.2 \times 10^{-3}\)  & \(3.1 \times 10^{-3}\)    &  1530 & 6,691 \\ 
        2 & SSBroyden with TrustRegion   & \(6.1 \times 10^{-3}\)  & \(6.7 \times 10^{-3}\)   &  2194 & 6,691 \\ 
        3 & NG &   \( 2.2\times 10^{-5}\)  &   \( 1.9\times 10^{-5}\)  &  3000 & 6,691 \\ 
        4 & SPRING &   \( 3.8\times 10^{-5}\) &   \( 3.8\times 10^{-5}\)   &  3000 & 6,691 \\ 
        \hline
        \end{tabular}}
\captionof{table}{\textbf{2D Helmholtz with $a_1=10$, $a_2=10$, and $k=1$:} 
Results are shown for NG, SSBroyden, and BFGS. 
For SSBroyden and BFGS, the loss is normalized by 
$\mathrm{scale}(a_1,a_2,k) = (a_1 \pi)^2 + (a_2 \pi)^2 + k^2$, 
which for $(a_1,a_2,k) = (10,10,1)$ yields 
$\mathrm{scale}  (10,10,1) = 200\pi^2 + 1 \approx 1974.92$.}
        \label{tab:}
    \end{minipage}
\end{figure}

\begin{figure}
  \centering
  \begin{minipage}[b]{\linewidth}
    \centering
    \includegraphics[width=0.78\linewidth]{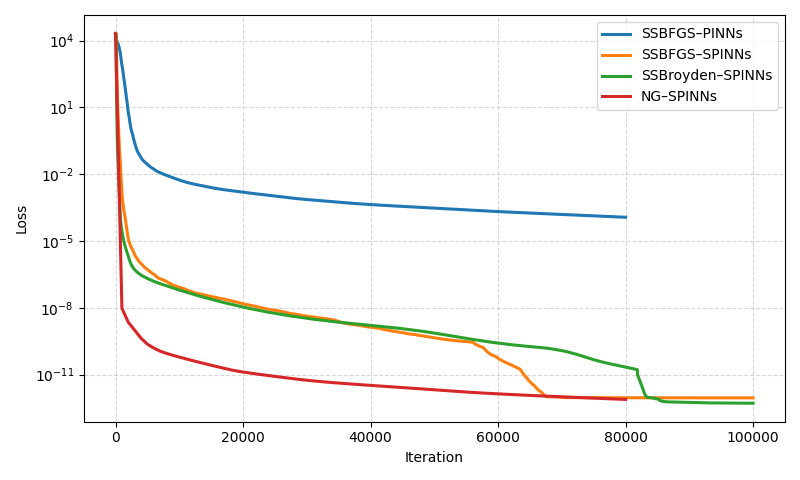}
    \caption{\textbf{3D Helmholtz with $a_1=4$, $a_2=4$, $a_3=3$ , and $k=1$.}
The network consists of a periodic Fourier feature embedding ($k_{\max}=2$),
followed by a fully connected network with four hidden layers, each containing
30 neurons with $\tanh$ activations, and a linear output layer.}
\label{fig:helmholtz3d_4_4_3_loss}
  \end{minipage}
  \vspace{1em}

  \begin{minipage}[b]{\linewidth}
    \centering
    \rowcolors{2}{cyan!15}{white}
    \scalebox{0.77}{
    \begin{tabular}{|l|l|c|c|c|c|}
      \hline
      \rowcolor{cyan!40}
      \textbf{Case} & \textbf{Optimizer} &
      \textbf{Relative $L_{\infty}$} & \textbf{Relative $L_{2}$} &
      \textbf{Training time (s)} & \textbf{\# parameters} \\
      \hline
      1 & PINN--SSBFGS with TrustRegion  & \(1.0\times 10^{-4}\) & \(3.5\times 10^{-4}\) & 551 & 3211 \\
      2 & PINN--NG & \(4.7\times 10^{-5}\) & \(4.1\times 10^{-5}\) & 571 & 3211 \\            
      3 & SPINN--SSBFGS with TrustRegion & \(1.0\times 10^{-9}\) & \(4.9\times 10^{-10}\) & 356 & 3211 \\
      4 & SPINN--SSBroyden with TrustRegion & \(7.9\times 10^{-9}\) & \(3.2\times 10^{-9}\) & 373 & 3211 \\
      5 & SPINN--NG & \(4.0\times 10^{-9}\) & \(2.1\times 10^{-9}\) & 730 & 3211 \\
      \hline
    \end{tabular}}
    \captionof{table}{\textbf{3D Helmholtz with $a_1=4$, $a_2=4$, $a_3=3$, and $k=1$:} comparison between vanilla PINNs and SPINNs~\cite{cho2023separable}.}
  \label{tab:helmholtz3d_4_4_3_summary}
  \end{minipage}
\end{figure}

Results are presented for the three-dimensional Helmholtz problem with wave numbers
$a_1=4$, $a_2=4$, and $a_3=3$~\cite{hwang2024dual}.
Figure~\ref{fig:helmholtz3d_4_4_3_loss} provides a direct comparison between vanilla PINNs and SPINNs under the same evaluation protocol. SPINNs provide improved accuracy compared to vanilla PINNs for this
higher-frequency 3D setting. In particular, SPINNs achieve lower solution error
across the reported metrics in Table~\ref{tab:helmholtz3d_4_4_3_summary}, indicating a
better ability to represent oscillatory structure induced by the Helmholtz
operator at $(a_1,a_2,a_3)=(4,4,3)$.

The neural network consists of a periodic feature layer followed by a fully connected multilayer perceptron.
The periodic layer maps the three-dimensional input $(x,y,z)$ to $6k_{\max}$ Fourier features using sine and cosine functions, with $k_{\max}=2$.
These features are passed to a five-layer fully connected network with architecture
$12 \rightarrow 30 \rightarrow 30 \rightarrow 30 \rightarrow 30 \rightarrow 1$,
using $\tanh$ activations in all hidden layers and a linear output layer.
The model is trained using $20{,}000$ randomly sampled interior collocation points in $[-1,1]^3$, without a structured training grid.

\paragraph{Key observations from the Helmholtz equation:}

For the 2D Helmholtz equation, a broad range of test cases was examined. The results indicate that, for both SSBFGS and SSBroyden, the trust-region strategy consistently outperforms the line-search variants in terms of efficiency. For the cases $(a_1,a_2)=(6,6)$ and $(a_1,a_2)=(10,10)$, NG converged successfully without any warm-up stage. Nevertheless, for SSBFGS and SSBroyden, a progressive warm-up schedule obtained by gradually increasing the frequency parameters, e.g., $(a_1,a_2)=(2,2)$, $(4,4)$, \ldots, $(6,6)$, significantly improved convergence speed and robustness.

For the high-wavenumber regimes ($k=10$ and $k=100$), NG converged in a single training run on the full domain. In contrast, for the quasi-Newton methods, convergence was improved by progressively extending the computational domain from $[-0.2,0.2]^2$ to $[-0.4,0.4]^2$, then to $[-0.7,0.7]^2$, and finally to the full domain $[-1,1]^2$.

For the 3D Helmholtz equation, SPINN consistently outperformed standard PINNs. Although the case $(a_1,a_2,a_3)=(4,4,3)$ with $k=1$ is particularly challenging, the proposed approach achieved high accuracy, reaching errors on the order of $10^{-9}$.

\FloatBarrier
\subsection{2D Stokes flow over a triangular wedge}\label{sec:Stokes_Flow}

\begin{figure}
    \centering
    \begin{minipage}[b]{\linewidth}
        \centering
        \includegraphics[width=0.49\linewidth]{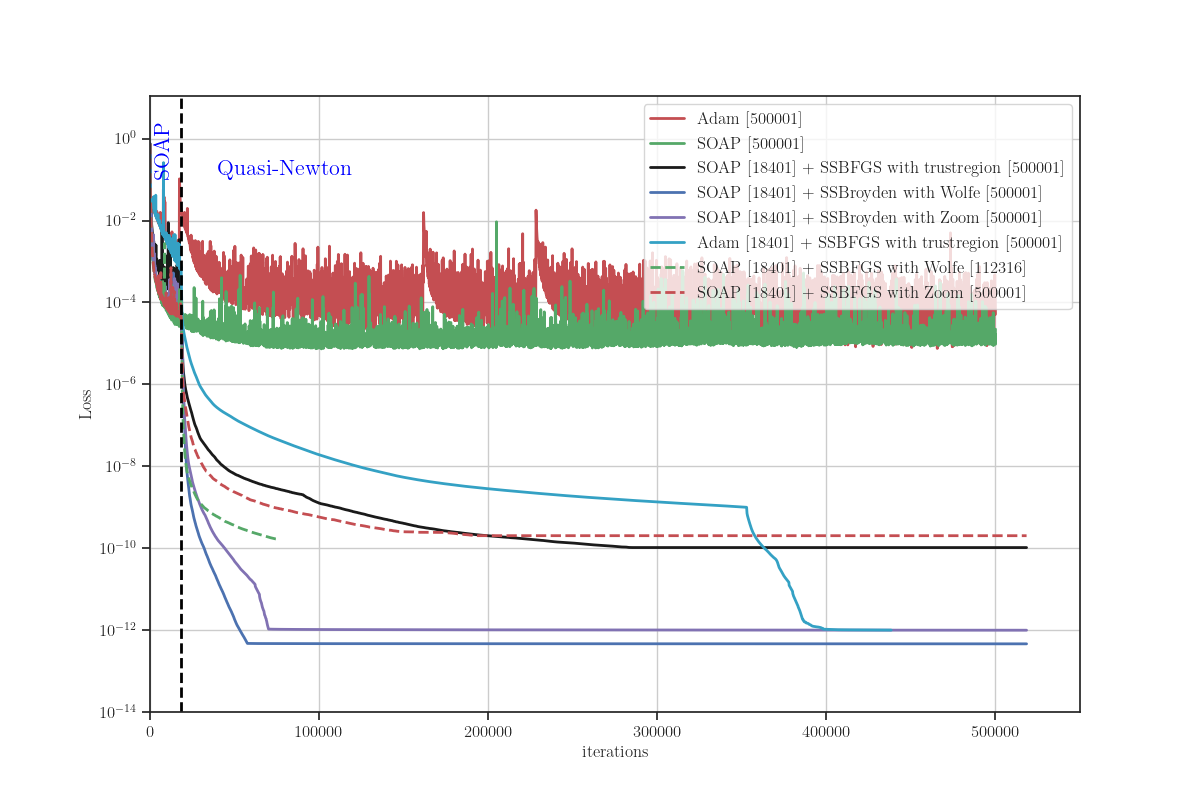}
        \includegraphics[width=0.49\linewidth]{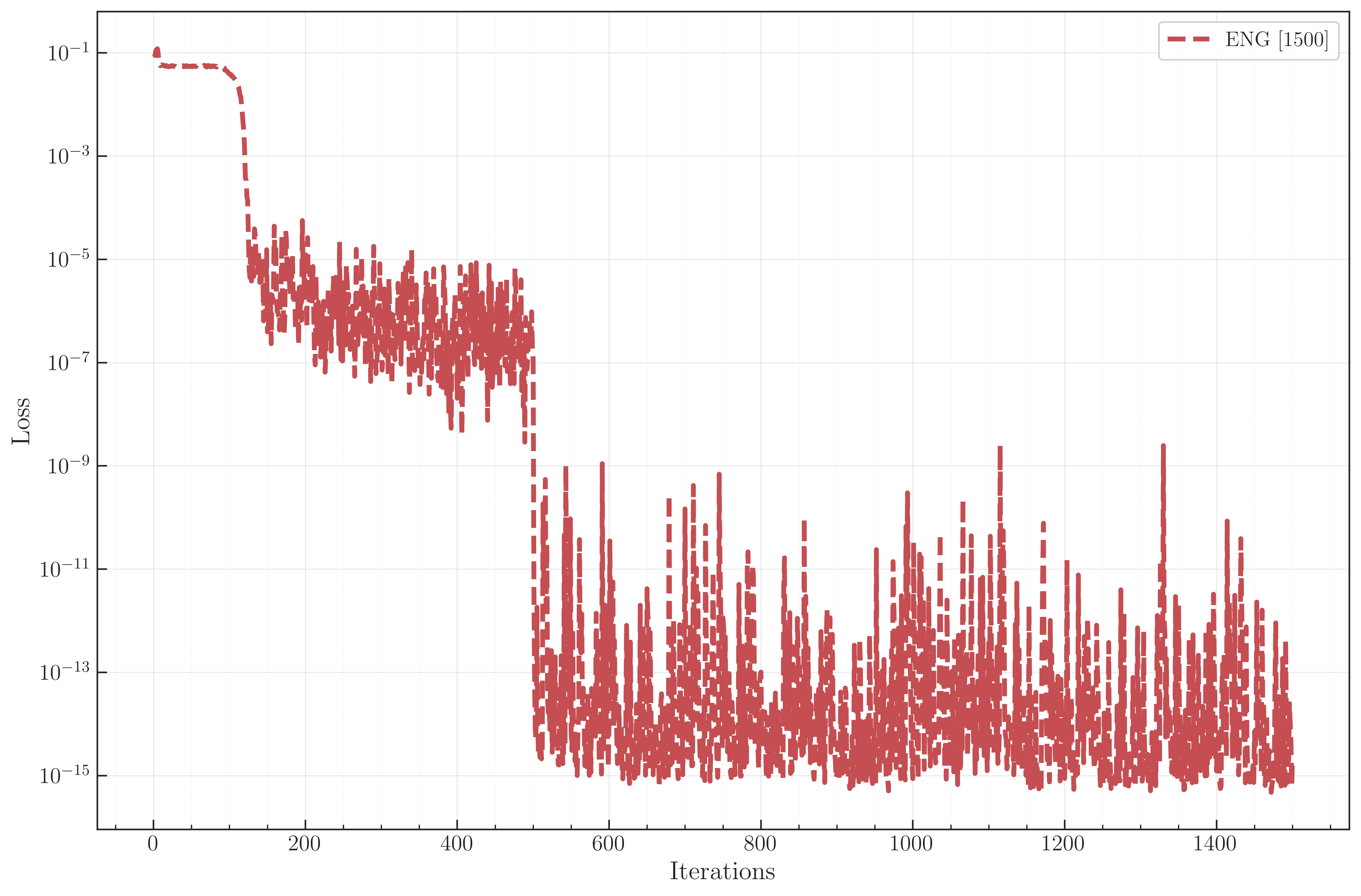}
        \caption{\textbf{Stokes equation:} 
      The left subfigure presents the loss history for Stokes flow using various combinations of optimizers and line-search routines. Each experiment starts with a warm-up phase employing first-order optimizers (Adam and SOAP), followed by a switch to quasi-Newton optimizers combined with different line-search strategies. The point of transition between optimizers is marked by a vertical dashed line. The right subfigure shows the loss history for NG optimizer, which significantly outperforms all self-scaled and first-order optimizers, achieving convergence in just 1,500 iterations.
}
        \label{fig:Stokes_loss}
    \end{minipage}

    \vspace{1em}

    \begin{minipage}[b]{\linewidth}
        \centering
        \rowcolors{2}{cyan!15}{white}
        \scalebox{0.65}{
        \begin{tabular}{|l|c|c|c|}
            \hline
            \rowcolor{cyan!40}
            \textbf{Optimizer [\# Iters.]} & \textbf{Relative \(L_2\) (u, v, p)} & \textbf{Total Parameters} & Time (s) \\
            \hline
            Adam (500,001)                                     & ($2.70 \times 10^{-2}$, $2.58 \times 10^{-2}$, $1.10 \times 10^{-1}$) & 33,667 & 3600 \\
            SOAP (500,001)                                     & ($2.02 \times 10^{-2}$, $1.87 \times 10^{-2}$, $1.90 \times 10^{-1}$) & 33,667   & 3600 \\
            SOAP (18401) + SSBFGS with trustregion  (500,001)  & ($3.25 \times 10^{-4}$, $1.59 \times 10^{-4}$, $1.07 \times 10^{-1}$) & 33,667 & 3600 \\
            SOAP (18401) + SSBroyden with Armijo and Wolfe  (500,001)  & ($3.25 \times 10^{-4}$, $1.57 \times 10^{-4}$, $1.07 \times 10^{-1}$) & 33,667  & 3600 \\
            SOAP (18401) + SSBroyden with zoom (500,001)  & ($3.25 \times 10^{-4}$, $1.57 \times 10^{-4}$, $1.07 \times 10^{-1}$) & 33,667 & 3600\\
            Adam (18401) + SSBFGS with trustregion (500,001)  & ($3.25 \times 10^{-4}$, $1.57 \times 10^{-4}$, $1.10 \times 10^{-1}$) & 33,667 & 3600 \\
            SOAP (18401) + SSBFGS with Armijo abd Wolfe (112,316)  & ($3.26 \times 10^{-4}$, $1.58 \times 10^{-4}$, $1.07 \times 10^{-1}$) & 33,667  & 3600 \\
             SOAP (18401) + SSBFGS with Zoom (500,001)  & ($3.26 \times 10^{-4}$, $1.61 \times 10^{-4}$, $1.07 \times 10^{-1}$) & 33,667  & 3600 \\
             NG with Line-search, undersampled (5'000) & ($3.25 \times 10^{-4}$, $1.57 \times 10^{-4}$, $1.07 \times 10^{-1}$) & 33,667 & 150\\
                        \hline
        \end{tabular}}
        \captionof{table}{\textbf{Stokes equation:}
      Comparison of relative $L_2$ errors and total number of model parameters for different combinations of optimizers and line-search strategies. All models use the same network architecture with 33{,}667 parameters. The absolute and relative convergence tolerances for SSBFGS and SSBroyden are set to $10^{-14}$. All computations are performed on an NVIDIA H100 GPU, with 75\% persistent memory usage (80~GB total) and 99\% compute utilization, achieving 25.35~TFLOPs out of a theoretical peak of 25.61~TFLOPs.
}\label{tab:Stokes_equation}
    \end{minipage}
\end{figure}

To show the efficacy of proposed optimizer, we finally solve the Stokes equation for a viscous flow in a lid-driven wedge. The Stokes equation is given by
\begin{align}\label{eq:Stokes}
\begin{aligned}
\frac{\partial p}{\partial x} -\left(\frac{\partial^2 u} {\partial x^2} + \frac{\partial^2 u} {\partial y^2} \right) = 0,\\
\frac{\partial p}{\partial y} -\left(\frac{\partial^2 v} {\partial x^2} + \frac{\partial^2 v} {\partial y^2} \right) = 0,\\
\frac{\partial u}{\partial x} + \frac{\partial v} {\partial y} = 0,
\end{aligned}
\end{align}

This problem was previously solved using the SSBFGS optimizer in~\cite{kiyani2025optimizing}. The problem setup and description adopted here are identical to those presented in~\cite{kiyani2025optimizing}. In this work, we examine the convergence behavior of PINNs using both the SSBFGS and SSBroyden optimizers under different line-search strategies. Their performance is compared against the NG method in terms of accuracy and computational runtime.

This benchmark is chosen because it requires solving PDEs in a regime where the solution contains extremely small-amplitude flow structures. Moffatt~\cite{moffatt1964viscous} obtained an asymptotic solution to Equation~\eqref{eq:Stokes} using a similarity approach, showing that the intensity of the eddies is governed by the wedge angle. For the case studied here, with a wedge angle of $25.53^\circ$, the strength of each successive eddy is expected to be asymptotically about $400$ times smaller than that of the previous one. Resolving such rapidly decaying eddies therefore demands an optimizer that is both highly accurate and robust.

To evaluate the accuracy of the PINN solution, we compute a reference solution of Equation~\eqref{eq:Stokes} using the spectral element method, following the formulation presented in~\cite{karniadakis2013spectral}. All reference simulations are carried out using the \texttt{Nektar++} framework~\cite{cantwell2015nektar++}.

For all experiments, the PINN architecture consists of eight fully connected hidden layers with 64 neurons per layer and $\tanh$ activation functions. A total of $120{,}000$ residual points are sampled uniformly from the computational domain using the distribution $\mathcal{U}[0,1)$ when training with the SSBFGS and SSBroyden optimizers. Training is first performed using Adam and SOAP as standalone optimizers. Subsequently, we conduct computational experiments using different line-search strategies within the self-scaled quasi-Newton family, including trust-region, Armijo, Wolfe, and zoom line-search methods. 

For the NG method, $5{,}000$ residual points are sampled at each iteration. The convergence histories for all combinations involving the self-scaled optimizers, as well as the standalone Adam and SOAP methods, are shown in the left subfigure of Figure~\ref{fig:Stokes_loss}. The convergence history of the NG method is presented in the right subfigure of Figure~\ref{fig:Stokes_loss}. The $\ell_2$ relative error, together with the optimizer type, number of parameters, and runtime, are summarized in Table~\ref{tab:Stokes_equation}.

Among the tested configurations, the most accurate results are obtained using Adam combined with SSBFGS under a trust-region strategy, and SOAP combined with SSBroyden using Armijo and Wolfe line-search strategies. The accuracy achieved by the NG method is comparable to that of the self-scaled optimizers; however, its runtime is approximately $150$ seconds, compared to roughly one hour for the self-scaled methods. Consequently, the NG approach is approximately $24$ times faster than the self-scaled optimizers.

To further assess the performance of each optimizer, we visualize the resulting streamlines in Figure~\ref{fig:Stokes_Domain} for all configurations listed in Table~\ref{tab:Stokes_equation}. The streamline patterns produced by the most accurate optimizers are shown in Figure~\ref{fig:Stokes_Domain}(d), (f), and (i). These results demonstrate successful recovery of all four eddies, including the weakest eddy with a velocity magnitude of $|u| = 10^{-8}$. Finally, Figure~\ref{fig:Stokes_Domain}(j) shows the Moffatt eddy strength~\cite{moffatt1964viscous}, comparing the transverse centerline velocity magnitude $(|u|)$. The results indicate that the NG method reproduces the eddy hierarchy with high accuracy.

\begin{figure}
\centering
\begin{subfigure}[b]{0.22\linewidth}
\includegraphics[width=\linewidth]{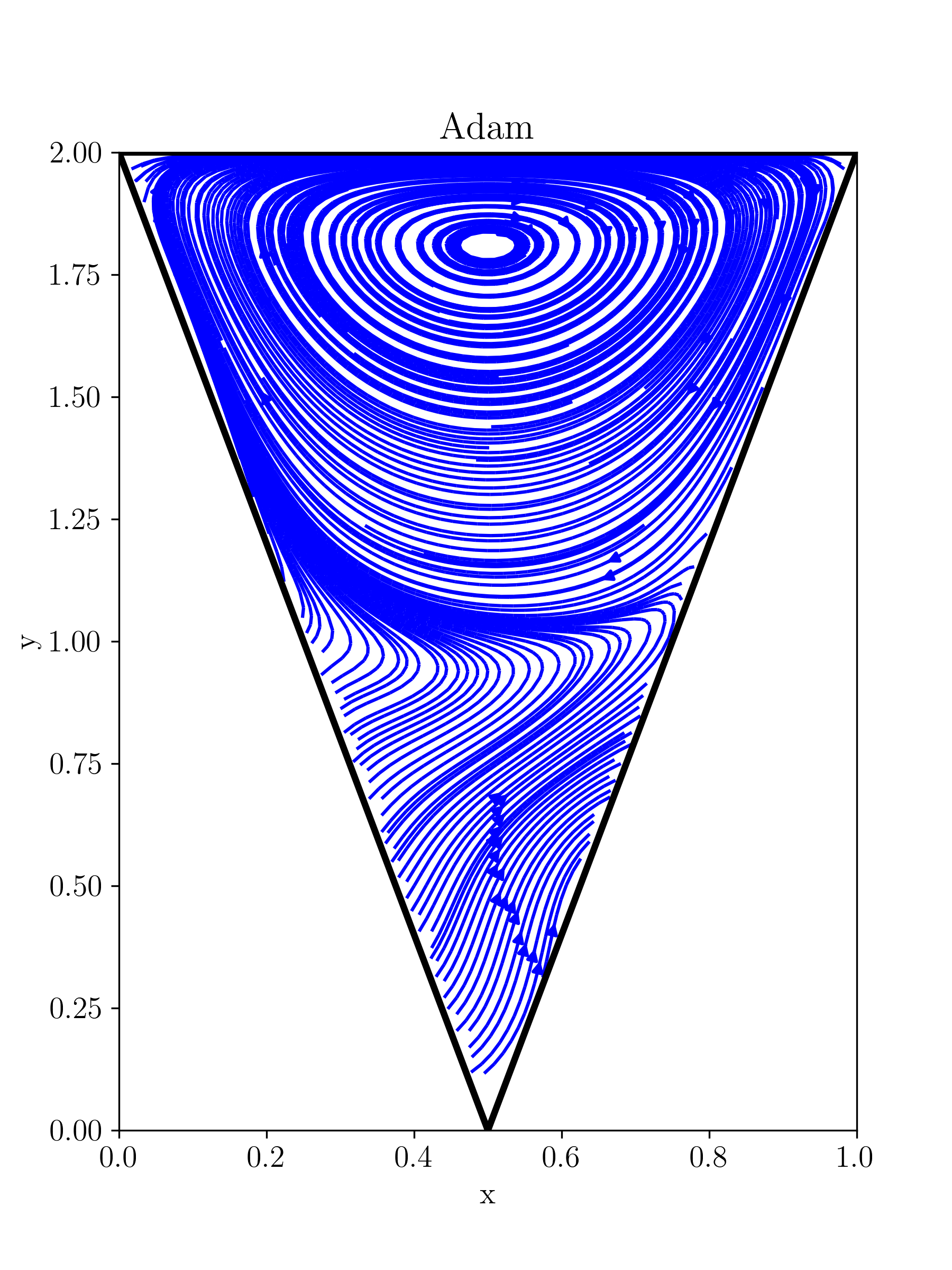}
\caption{}
\end{subfigure}\hfill
\begin{subfigure}[b]{0.22\linewidth}
\includegraphics[width=\linewidth]{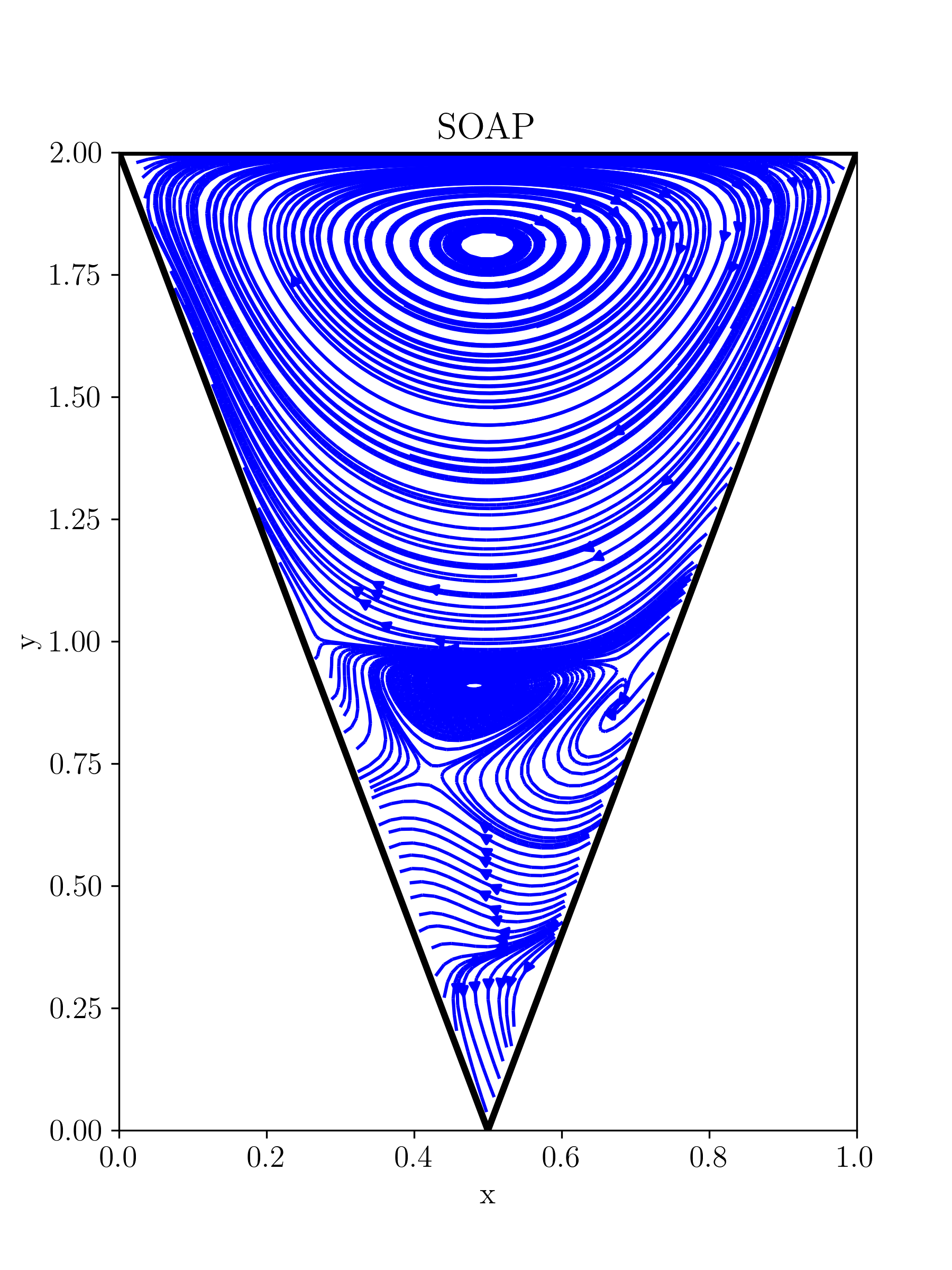}
\caption{}
\end{subfigure}\hfill
\begin{subfigure}[b]{0.22\linewidth}
\includegraphics[width=\linewidth]{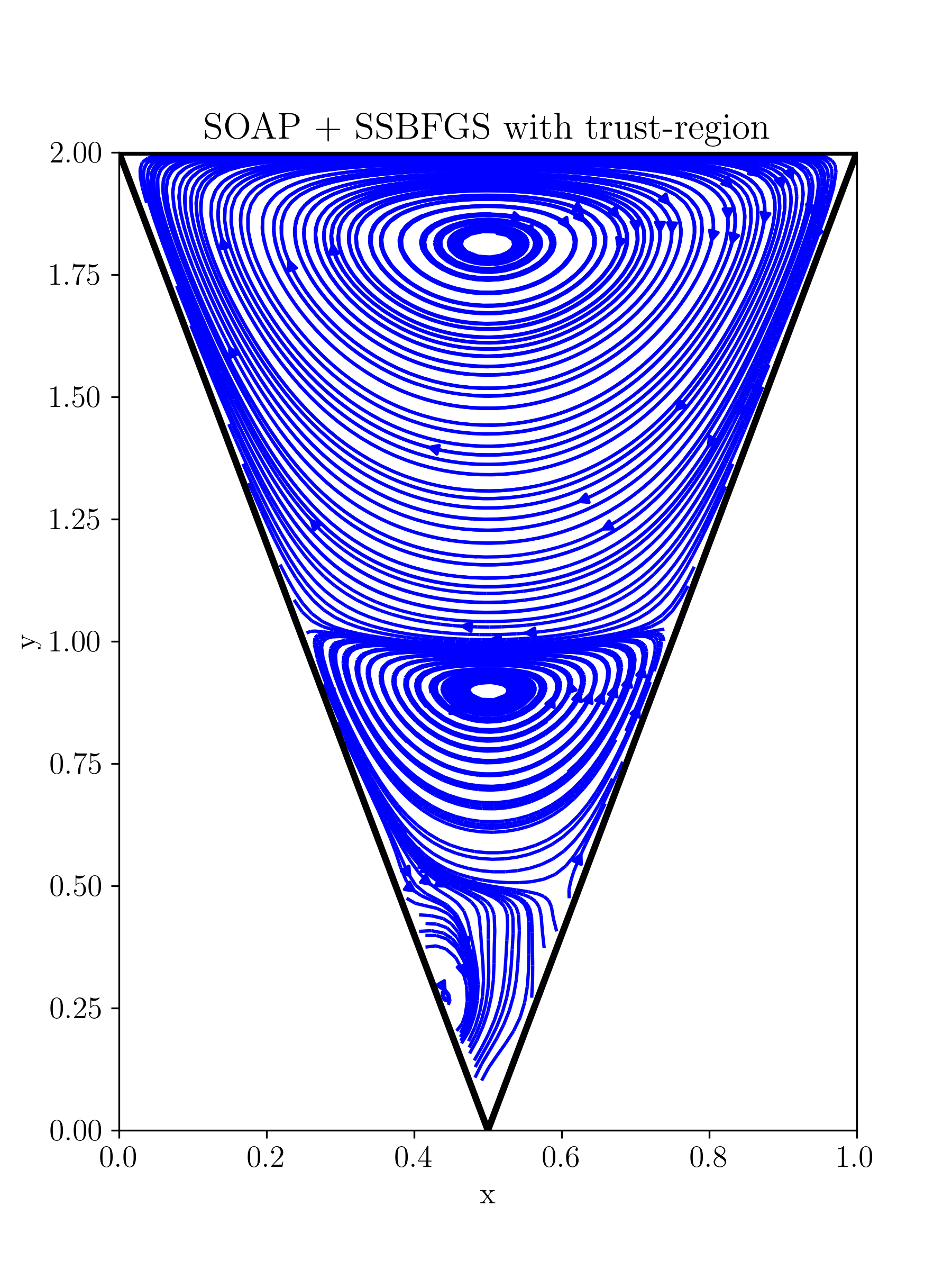}
\caption{}
\end{subfigure}\hfill
\begin{subfigure}[b]{0.22\linewidth}
\includegraphics[width=\linewidth]{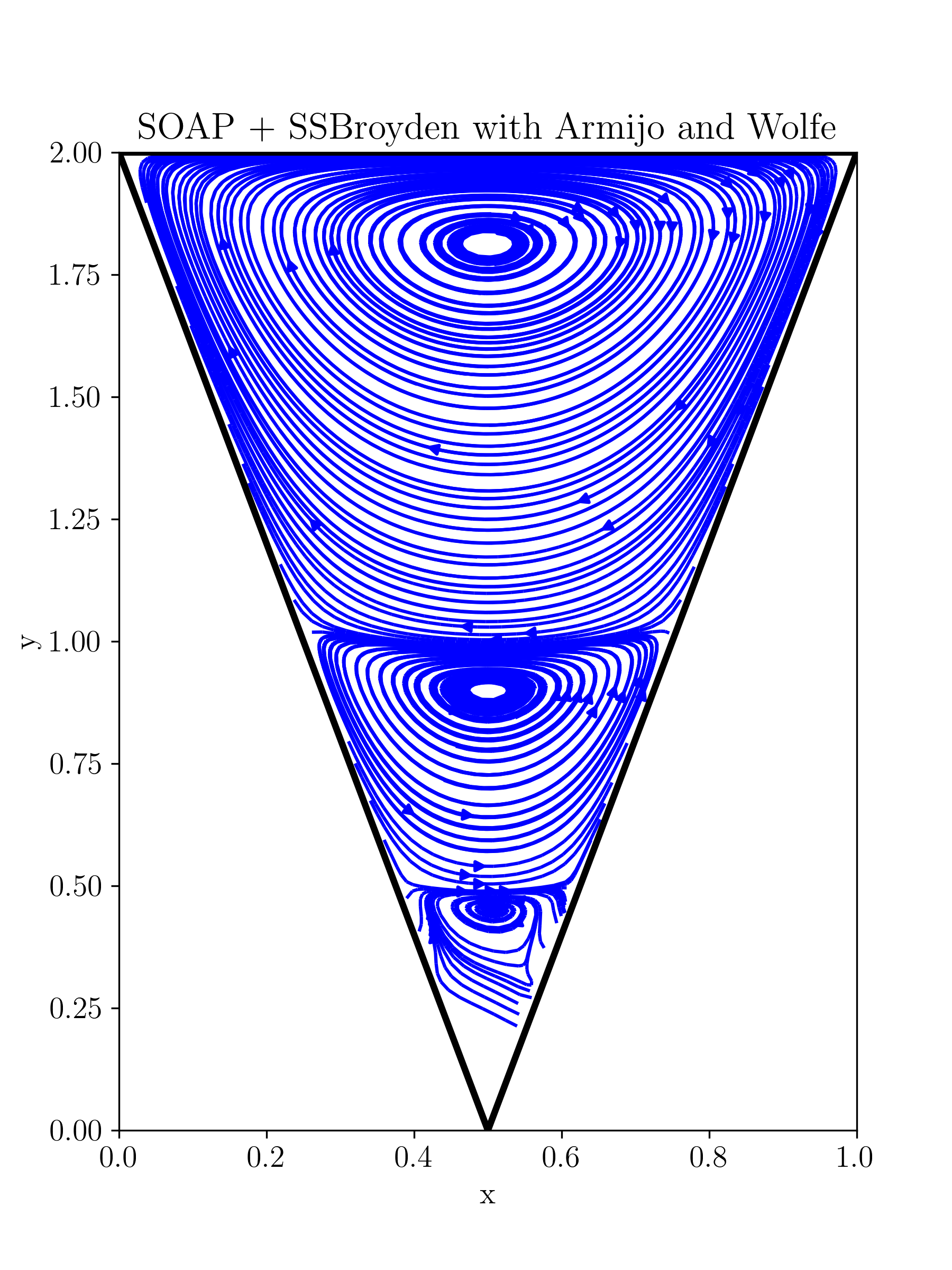}
\caption{}
\end{subfigure}

\vspace{-1mm}

\begin{subfigure}[b]{0.22\linewidth}
\includegraphics[width=\linewidth]{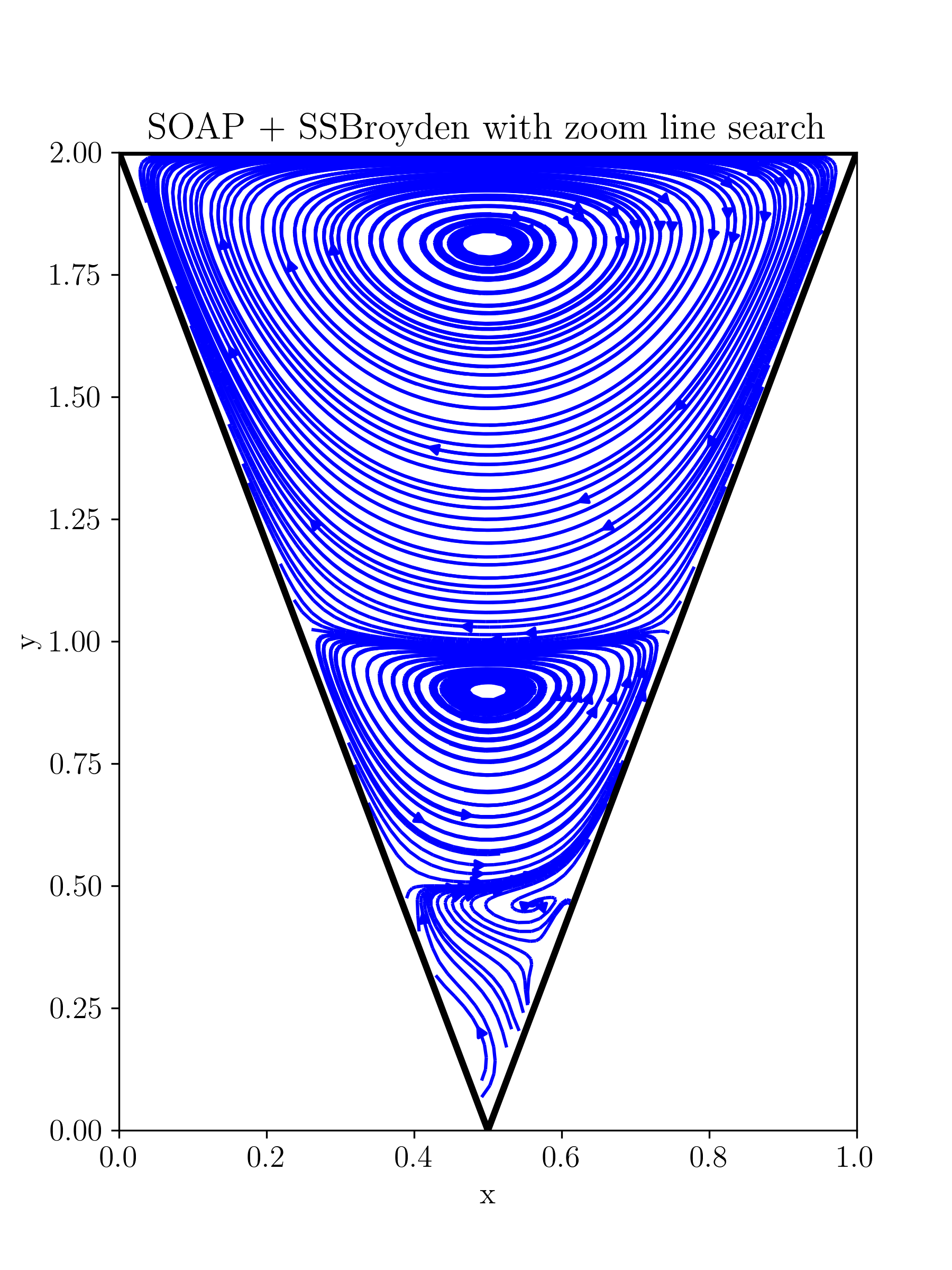}
\caption{}
\end{subfigure}\hfill
\begin{subfigure}[b]{0.22\linewidth}
\includegraphics[width=\linewidth]{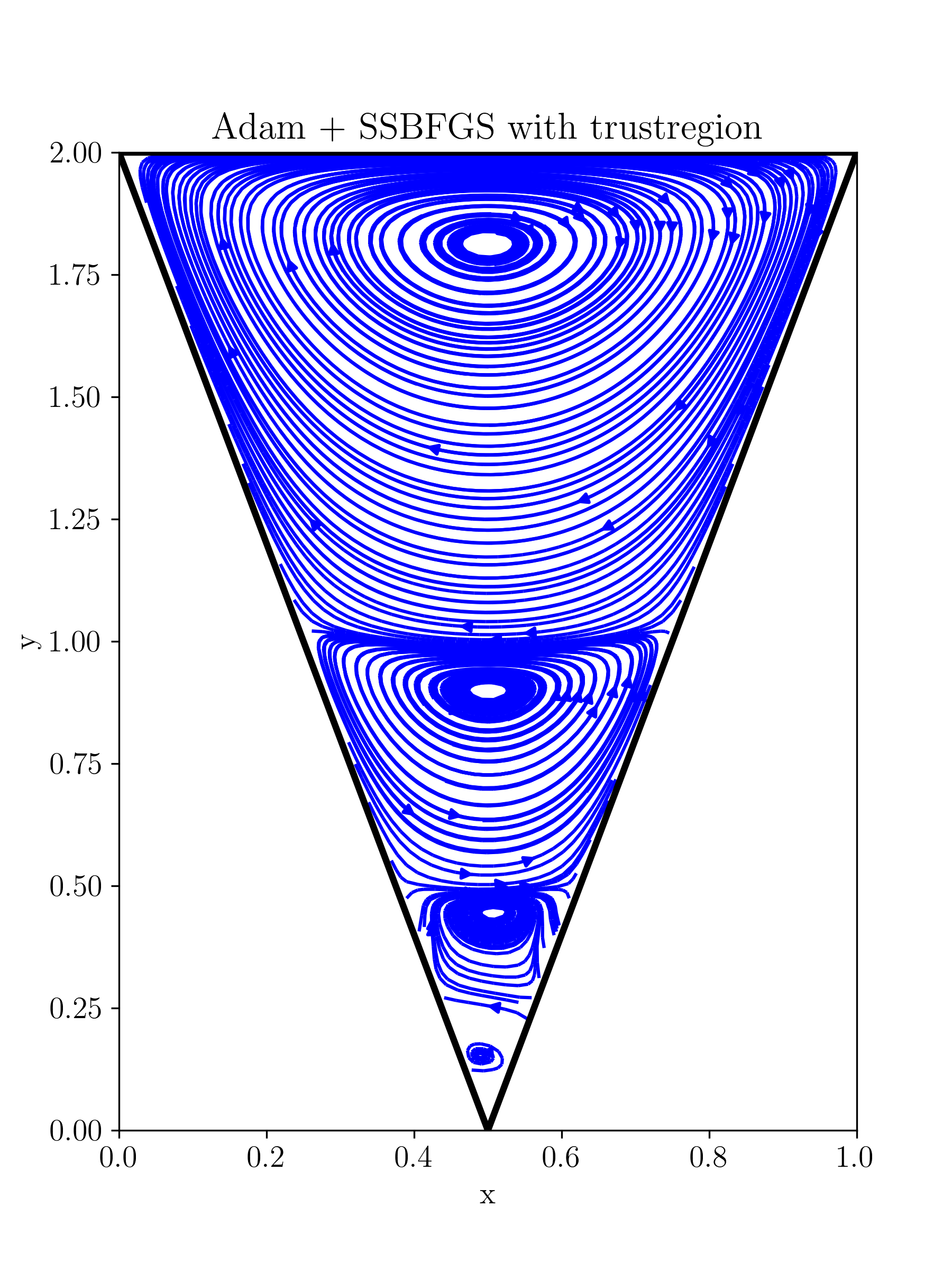}
\caption{}
\end{subfigure}\hfill
\begin{subfigure}[b]{0.22\linewidth}
\includegraphics[width=\linewidth]{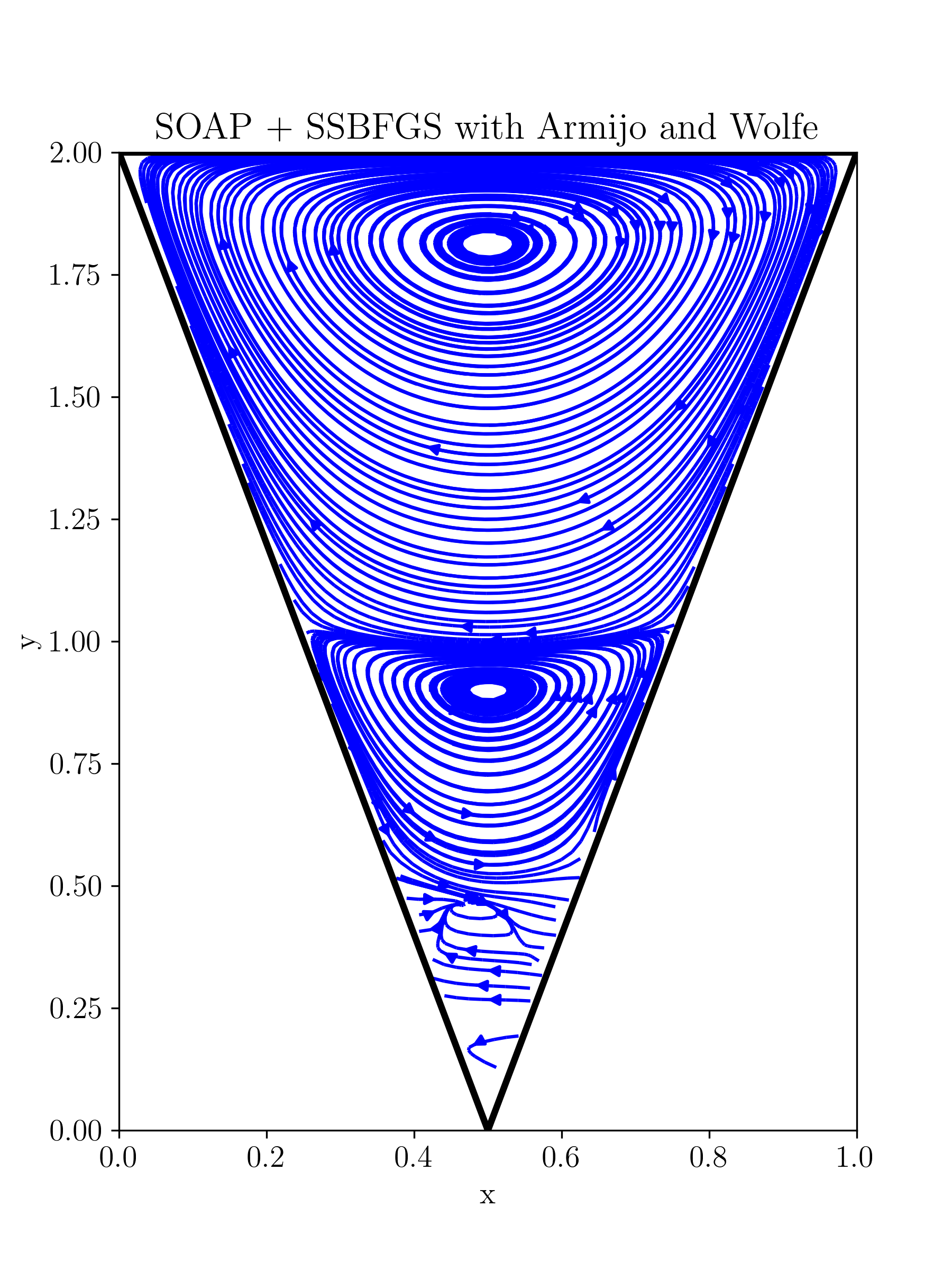}
\caption{}
\end{subfigure}\hfill
\begin{subfigure}[b]{0.22\linewidth}
\includegraphics[width=\linewidth]{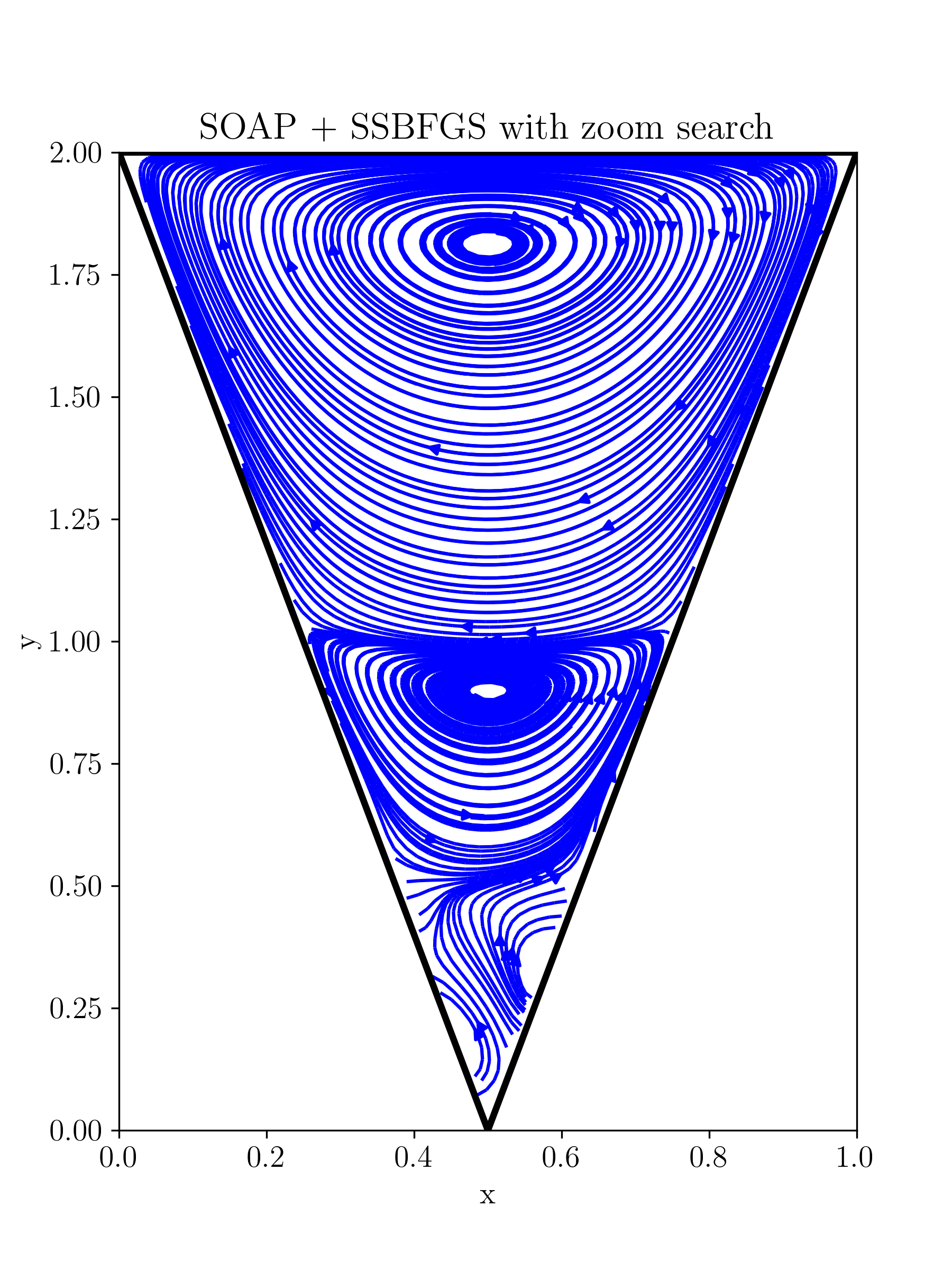}
\caption{}
\end{subfigure}

\vspace{-1mm}

\begin{subfigure}[b]{0.22\linewidth}
\includegraphics[width=\linewidth]{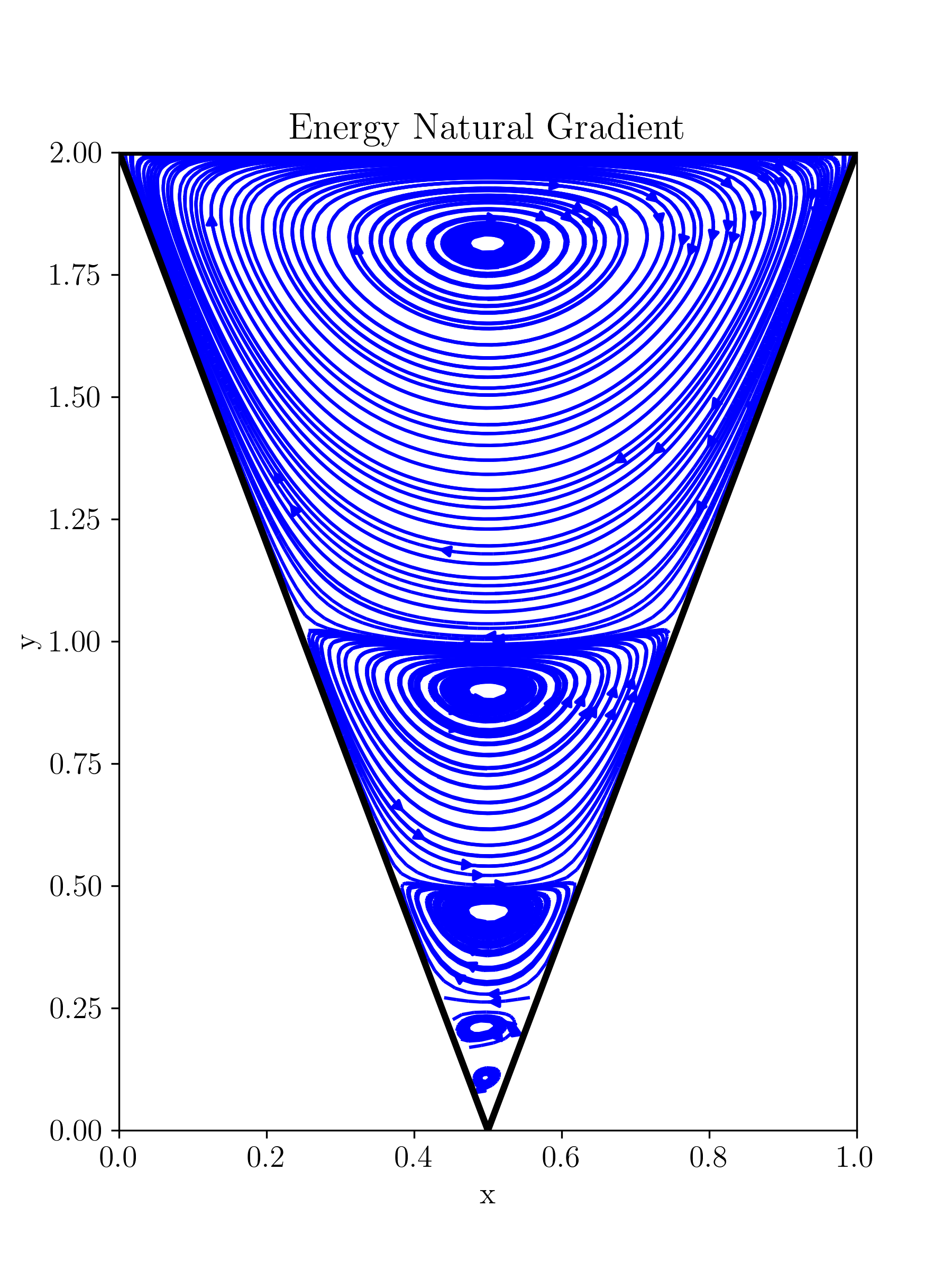}
\caption{}
\end{subfigure}\hfill
\begin{subfigure}[b]{0.35\linewidth}
\includegraphics[width=\linewidth]{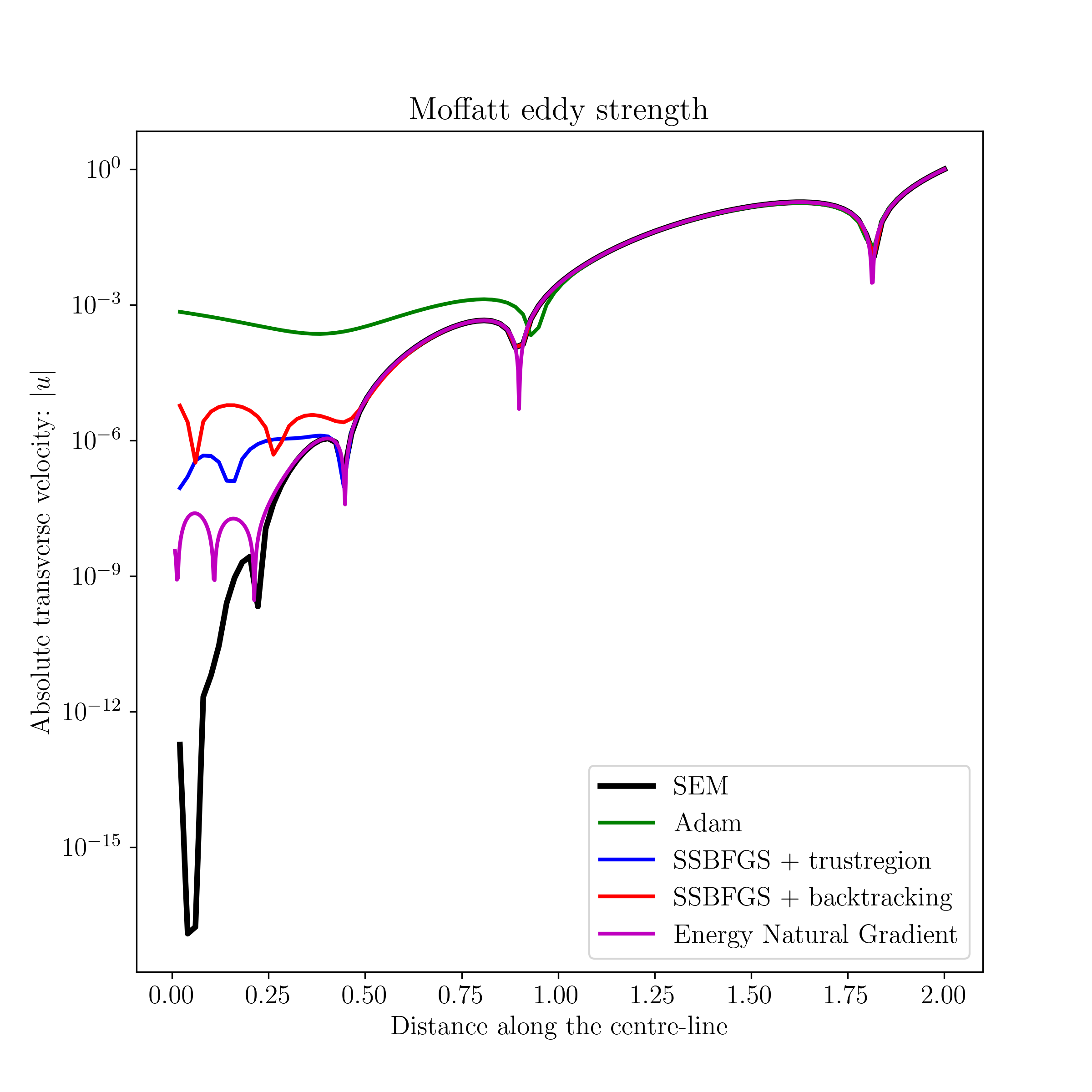}
\caption{}
\end{subfigure}

\caption{\textbf{Stokes equation:} Streamlines obtained using various optimization strategies: 
(a) Adam, (b) SOAP, (c) SOAP + SSBFGS (trust-region), 
(d) SOAP + SSBroyden (Armijo–Wolfe), 
(e) SOAP + SSBroyden (zoom), 
(f) Adam + SSBFGS (trust-region), 
(g) SOAP + SSBFGS (Wolfe), 
(h) SOAP + SSBFGS (zoom), and 
(i) NG. 
The NG achieves comparable accuracy to self-scaled optimizers while significantly reducing runtime. 
Panel (j) shows the Moffatt strength comparison. 
The reference solution is computed using a high-order spectral element method.}
\label{fig:Stokes_Domain}

\end{figure}

\paragraph{Key observations from the Stokes equation:}\
We observe that both SSBFGS and SSBroyden achieve accuracy comparable to that of the NG method across all considered line search strategies. However, the NG method exhibits significantly superior computational efficiency, with a runtime approximately 24 times faster than that of the self-scaled optimizers. 
This disparity suggests that, for elliptic problems, the NG method is better aligned with the underlying error landscape, enabling more efficient convergence. In contrast, while self-scaled optimizers attain similar levels of accuracy, they do so at a substantially higher computational cost for elliptic PDEs.

\subsection{(2+1)D viscous Burgers equation} \label{sec:2D viscous Burgers Equation}
Next, we investigate the performance of various optimizers on a parabolic 
PDE governed by the time-dependent viscous Burgers equation expressed as follows 
\begin{align}\label{eq:2D_viscous_Burger}
u_t+f(u)_x+g(u)_y=\nu\left(u_{x x}+u_{y y}\right), \quad(x, y) \in[0,1] \otimes[0,1] \text { and } t>0
\end{align}

where the fluxes in $x$ and $y$ directions are $f(u)=u^2 / 2, g(u)=u^2 / 2$, respectively and viscosity coefficient $v=0.004$. The initial and boundary conditions are obtained from the exact solution $u(x, y, t)= 1 /(1+\exp [(x+y-t) /(2 \nu)])$. In this experiment we analyze how low viscosity can be handled.

As in the 
one-dimensional inviscid case presented in Section~\ref{sec:inviscid_burgers_eq}, 
we employ flux relaxation throughout. We conduct experiments across several 
optimizer combinations --- SSBFGS, SSBroyden, and SOAP --- paired with 
various line search strategies, and additionally evaluate a NG
optimizer. All experiments use a fully connected network with 10 hidden 
layers and $\tanh$ activation functions. The training details are summarized 
in Table~\ref{tab:optimizer_comparison_2d_viscous_burger}. The results indicate 
that the Adam--SSBFGS combination with a Wolfe line search achieves a 
comparable level of accuracy to the NG optimizer; however, the 
NG optimizer requires significantly less wall time, converging 
approximately $10$--$12$ times faster than the self-scaled class of optimizers.
We present the solution of the two-dimensional viscous Burgers equation 
obtained using SOAP, SSBFGS with a Wolfe line search, and the 
NG method in Figure~\ref{fig:S2D_Viscous_Burger}. Each row 
of Figure~\ref{fig:S2D_Viscous_Burger} displays the predicted solution, the 
absolute pointwise error, and a comparison of solution profiles between 
the reference and predicted solutions along $x = 0.2$ at $t = 0.5$. The 
error is concentrated in the vicinity of the discontinuity, where SOAP 
exhibits difficulty in reducing the residual. The state-of-the-art 
accuracy for this equation has been reported at $\mathcal{O}(10^{-3})$ 
using a large number of parameters and a domain decomposition algorithm 
\cite{jagtap2020conservative}, whereas the proposed approach achieves 
a significantly lower error of $\mathcal{O}(10^{-7})$.

\begin{figure}
    \centering
    \begin{subfigure}[b]{\linewidth}
        \centering
        \includegraphics[width=\linewidth,height=0.22\textheight,keepaspectratio]{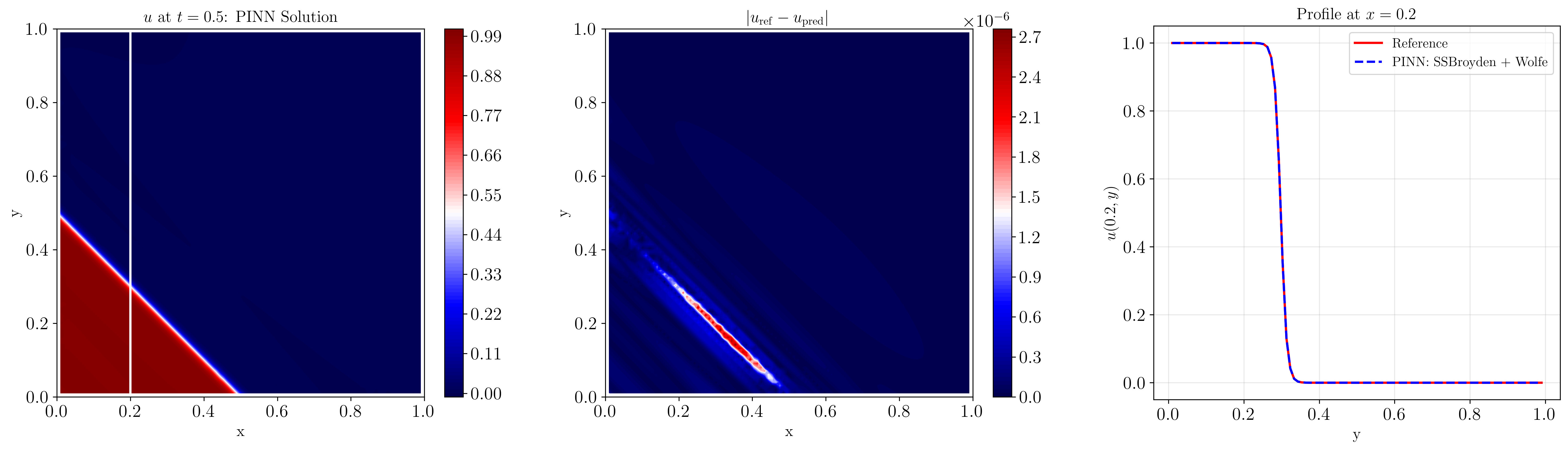}
        \caption{SSBroyden with zoom line search (best quasi-Newton + line search).}
        \label{fig:visc_burgers_a}
    \end{subfigure}
    \vspace{0.5em}
     \begin{subfigure}[b]{\linewidth}
        \centering
        \includegraphics[width=\linewidth,height=0.22\textheight,keepaspectratio]{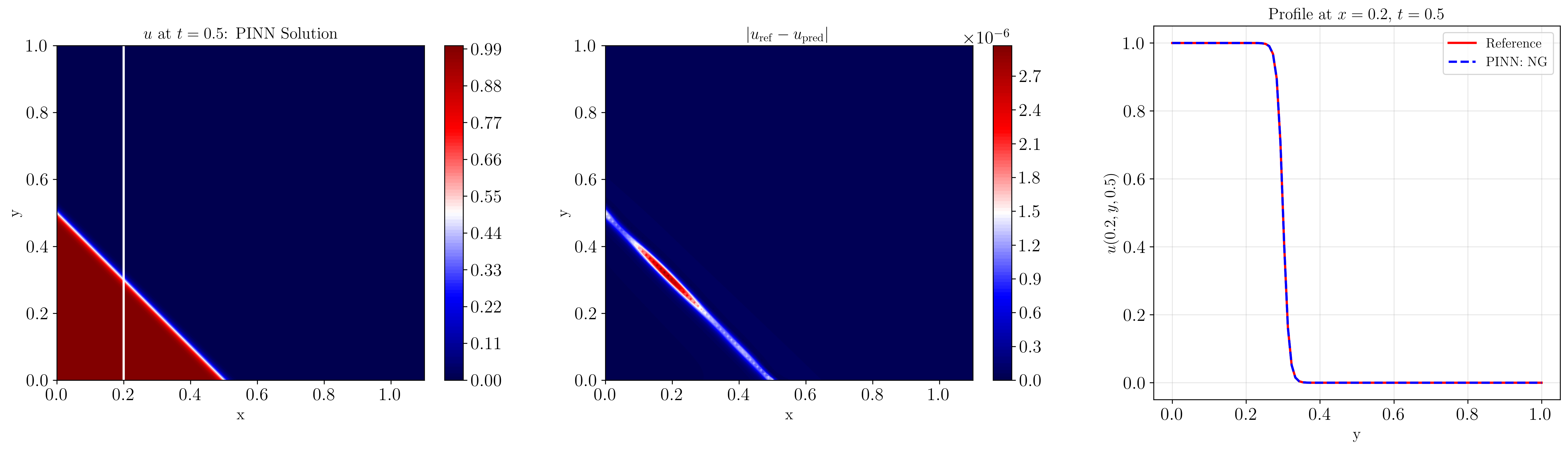}
        \caption{NG.}
        \label{fig:visc_burgers_a}
    \end{subfigure}
    \begin{subfigure}[b]{\linewidth}
        \centering
        \includegraphics[width=\linewidth,height=0.22\textheight,keepaspectratio]{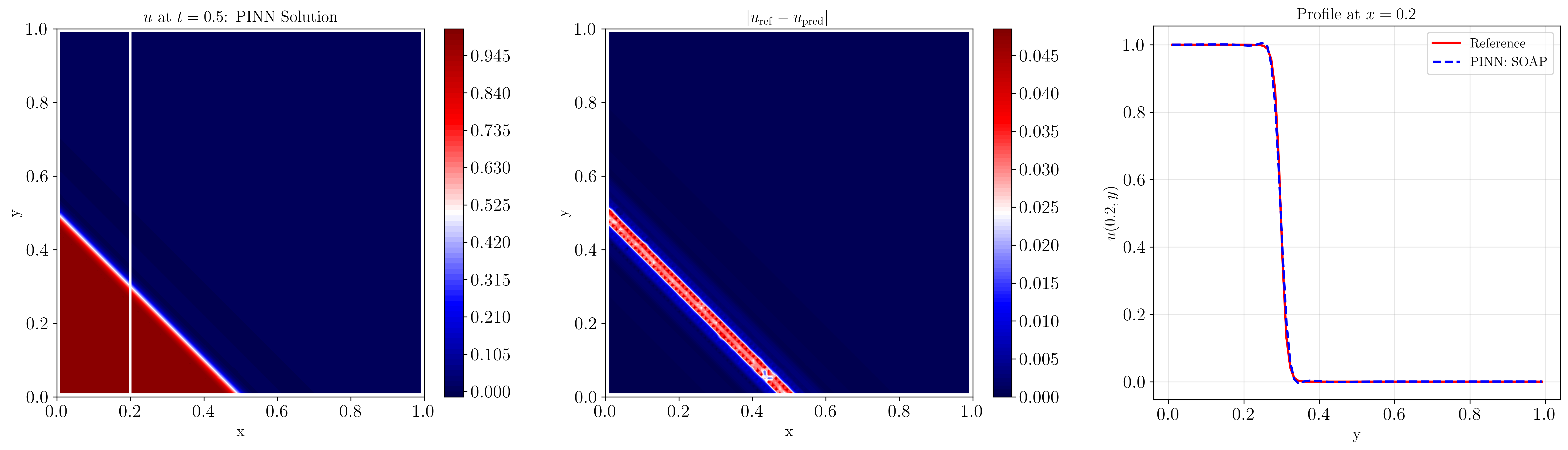}
        \caption{SOAP.}
        \label{fig:visc_burgers_b}
    \end{subfigure}
    \caption{\textbf{2D viscous Burgers equation:}
    Solution obtained using different optimization methods.
    (a) Best results from a quasi-Newton and line search method (SSBroyden with zoom line search), (b) NG and
    (c) SOAP optimization. In each panel, the left subfigure shows the PINN solution at $t=0.5$, the middle subfigure shows the spatial distribution of the absolute error (concentrated at the shock location), and the right subfigure shows the solution profile at $x=0.2$.
    The SOAP optimizer fails to converge at the shock location and exhibits slight overshoots in the shock region.
    Errors for all optimizer/line-search combinations (including SOAP) are reported in Table~\ref{tab:optimizer_comparison_2d_viscous_burger}.}
    \label{fig:S2D_Viscous_Burger}
\end{figure}
\begin{table}[H]
    \centering
    \begin{tabular}{|l|c|c|c|}
        \hline
        \rowcolor{cyan!40}
        \textbf{Optimizer [\# Iters.]} & Relative \ $L_2$  & \# parameters & Runtime (s) \\
        \hline
        SOAP (21001) & $1.7 \times 10^{-2}$ & 3881 & 414 \\
        \hline
        Adam (1001) + SSBroyden with Zoom (20000) & $1.3 \times 10^{-5}$ & 3881 & 927 \\
        \hline
        Adam (1001) + SSBroyden with Wolfe (20000) & $5.3 \times 10^{-7}$ & 3881 & 612 \\
        \hline
        Adam (1001) + SSBFGS with Zoom (20000) & $9.0 \times 10^{-7}$ & 3881 & 899 \\
        \hline
        Adam (1001) + SSBFGS with Wolfe (20000) & $5.7 \times 10^{-7}$ & 3881 & 624 \\
        \hline
         NG (20000) & $5.8 \times 10^{-7}$ & 3881 & 75 \\
        \hline
    \end{tabular}
    \caption{\textbf{Unsteady 2D viscous Burgers equation:} Performance comparison of the considered optimizers. 
Columns report the relative $L_2$ error, total number of parameters, and computational runtime in seconds.}
    \label{tab:optimizer_comparison_2d_viscous_burger}
\end{table}

\FloatBarrier
\subsection{(1+1)D inviscid Burgers equation}\label{sec:Inviscid_Burgers_Equation}
The inviscid Burgers equation is given by
\begin{equation} \label{eq:ivb}
\frac{\partial u}{\partial t} + \frac{1}{2}\frac{\partial u^2}{\partial x} = 0,
\end{equation}
where $x \in [-1,1]$ and $t \in [0,1]$, with initial condition
\[
u(x,0) = \sin(-\pi x).
\]

The inviscid Burgers equation develops discontinuities due to the intersection of characteristics, leading to the formation of shock waves. These discontinuities pose significant challenges for standard PINN training, as directly enforcing the classical PDE residual often results in instability of the optimization process. In such cases, the loss function may become degenerate, and the choice of optimizer alone is insufficient to ensure convergence to the physically correct solution. To address this issue, it is essential to incorporate conservation and entropy conditions across shocks into the PINN training framework. Enforcing these additional constraints enables the network to recover the physically admissible weak solution and improves both stability and accuracy~\cite{jagtap2022physics,LLPINNs,UrbanPons2025}. Therefore, we propose two strategies for solving the inviscid Burgers equation by introducing two new approaches aimed at improving the convergence and stability of PINN training.

\subsubsection{Inviscid Burgers equation with Roe linearization}

We follow the methodology proposed in~\cite{LLPINNs,UrbanPons2025} and apply it to the Burgers equation. First, the equation is expressed in its quasilinear form

\begin{equation}
    \frac{\partial u}{\partial t} + u \frac{\partial u}{\partial x} = 0.
\end{equation}
To impose jump conditions, we linearize the second term by changing $u$  for the shock speed in regions of potential shock formation. For Burgers equation, the velocity in terms of left and right states is very easy to obtain and is given by
\begin{equation}
    s = \frac{u_L + u_R}{2}.
\end{equation}
The regions of potential shock formation are computed in each step of the optimization process, and are defined by the following set
\begin{equation}
    \Omega_s = \left \lbrace (x,t) \in \left[-1,1\right] \times \left[0, 1 \right], \; \frac{\partial u}{\partial x} \leq -M \right \rbrace,
\end{equation}

where $M$ is a given positive constant. The values of $u_L$ and $u_R$ at any point $(x,t) \in \Omega_s$ are calculated by evaluating the network at $(x-h,t)$ and $(x+h,t)$, respectively, where $h$ is a predefined constant value. 

Hence, the PDE residual for each point $(x,t) \in \Omega$ is given by
\begin{equation}
    \mathcal{L}_\btheta(x,t) = \frac{\partial u_\btheta(x,t)}{\partial t} + \tilde{u}_\btheta(x,t)\frac{\partial u_\btheta(x,t)}{\partial x},
\end{equation}
where $u_\btheta$ represents the neural network output, and
\begin{equation}
    \tilde{u}_\btheta(x,t) =
    \begin{cases}
       \frac{u_\btheta(x+h,t) + u_\btheta(x-h,t)}{2}, \; &(x,t) \in \Omega_s, \\
       u_\btheta(x,t), \; &\text{Otherwise}
    \end{cases}
\end{equation}
We impose boundary conditions with hard-enforcement by transforming the input coordinates $(t,x)$ into $(t,\cos \pi x, \sin \pi x)$. The loss function is given by
\begin{equation}
    \mathcal{J} = \frac{1}{N_\Omega}\sum_{i=1}^{N_\Omega} \left[\mathcal{L}_\btheta(x_i,t_i) \right]^2 + \frac{1}{N_\Gamma}\sum_{i=1}^{N_\Gamma} \left[u(x_i,0)-u_\btheta(x_i,0) \right]^2,
\end{equation}
where $N_\Omega$ and $N_\Gamma$ are the number of points in the training set that belong to $\Omega$ and $\Gamma = [-1,1] \times \lbrace 0 \rbrace$.

 \begin{figure}
     \centering
     \includegraphics[width=0.90\linewidth]{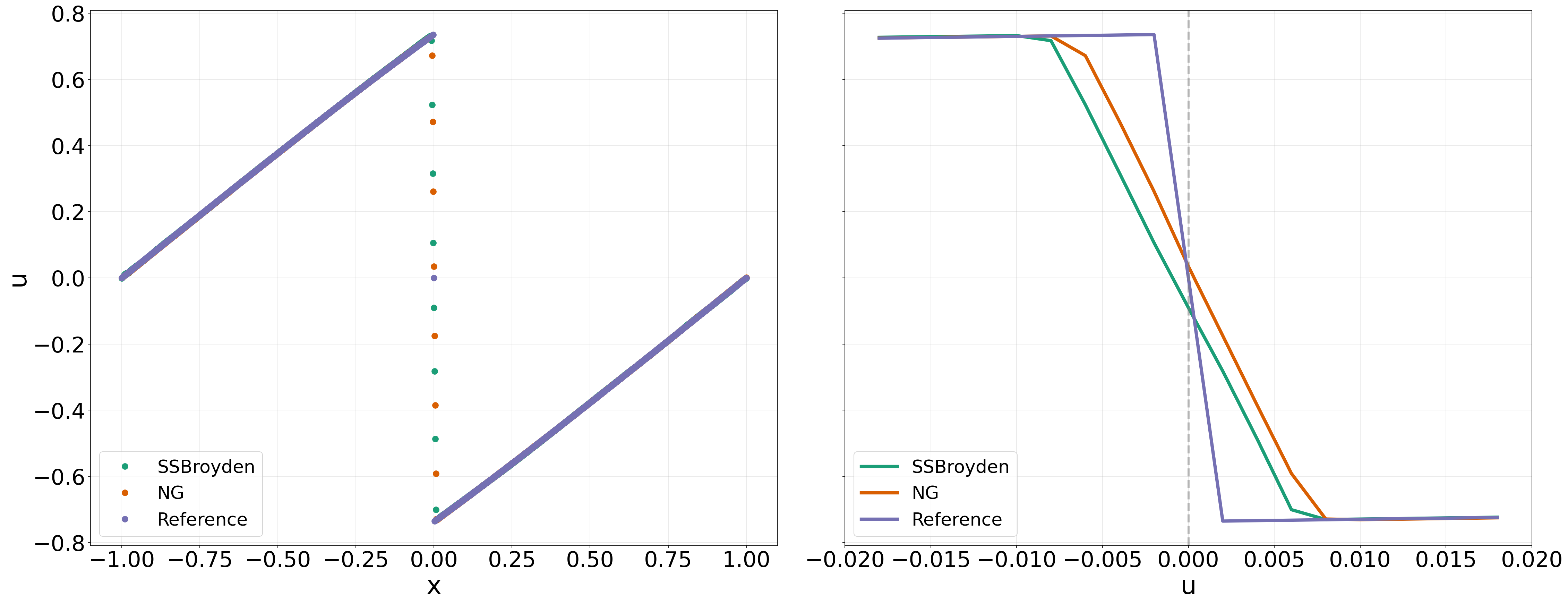}
     \caption{\textbf{Inviscid Burgers with Roe linearization:} Predicted solutions at $t=1$ with the LRPINN for both optimizers, together with a reference numerical solution obtained with a third-order WENO scheme. Left panel shows the solution for all $x \in \left[-1,1 \right]$, whereas the right panel shows a zoomed-in view of the shock region ($x=0$).}
     \label{fig:lrpinn_prediction}
 \end{figure}

We compare the performance of NG and SSBroyden optimizers. Since the inviscid Burgers equation develops discontinuities due to the formation of shock waves, one typically requires a sufficiently large number of training points to adequately sample regions near the shocks and resolve them accurately. Hence, we start by analyzing both optimizers by comparing the precision and the training time achieved in both cases as a function of the number of interior points $N_\Omega$.  Table~\ref{tab:Roe_linearization_collocation_points} shows the relative $L_1$ and $L_2$ errors of the PINN solution compared against a reference solution (obtained with a third-order WENO scheme~\cite{hesthaven2017numerical} with $1000$ spatial points) and the training time for both optimizers. 

\begin{table}[h!]
    \centering
    \scalebox{0.77}{%
    \begin{tabular}{|l|c|c|c|c|}
        \hline
        \rowcolor{cyan!40}
        \textbf{Optimizer} & Relative \ $L_1$  & Relative \ $L_2$ & \# collocation points & Training time (s) \\
        \hline
        \rowcolor{cyan!15}
        \textbf{SSBroyden} & $(1.2 \pm 0.2) \times 10^{-2}$ & $(7.8 \pm 1.9) \times 10^{-2}$ & 625   & 20 \\
        \rowcolor{white}
                           & $(1.2 \pm 0.2) \times 10^{-2}$ & $(9.2 \pm 2.5) \times 10^{-2}$ & 1250  & 21 \\
        \rowcolor{cyan!15}
                           & $(9.4 \pm 1.8) \times 10^{-3}$ & $(8.2 \pm 2.0) \times 10^{-2}$ & 2500  & 22 \\
        \rowcolor{white}
                           & $(8.9 \pm 1.9) \times 10^{-3}$ & $(7.9 \pm 1.3) \times 10^{-2}$ & 5000  & 26 \\
        \rowcolor{cyan!15}
                           & $(7.4 \pm 1.0) \times 10^{-3}$ & $(7.8 \pm 1.4) \times 10^{-2}$ & 10000 & 31 \\
        \hline
        \rowcolor{white}
        \textbf{NG}        & $(1.5 \pm 0.6) \times 10^{-2}$ & $(2.6 \pm 1.8) \times 10^{-1}$ & 625   & 54 \\
        \rowcolor{cyan!15}
                           & $(8.5 \pm 2.3) \times 10^{-3}$ & $(9.1 \pm 3.1) \times 10^{-2}$ & 1250  & 73 \\
        \rowcolor{white}
                           & $(1.7 \pm 0.9) \times 10^{-2}$ & $(1.8 \pm 0.8) \times 10^{-1}$ & 2500  & 113 \\
        \rowcolor{cyan!15}
                           & $(6.8 \pm 1.0) \times 10^{-3}$ & $(6.8 \pm 0.4) \times 10^{-2}$ & 5000  & 121 \\
        \rowcolor{white}
                           & $(6.4 \pm 0.4) \times 10^{-3}$ & $(7.1 \pm 0.5) \times 10^{-2}$ & 10000 & 193 \\
        \hline
    \end{tabular}%
    }
    \caption{\textbf{Inviscid Burgers with Roe linearization:} Performance comparison of the SSBroyden and NG optimizers as a function of the number of collocation points. The columns report the relative $L_1$ and $L_2$ errors, the number of collocation points $N_\Omega$, and the computational runtime in seconds. In all cases, the neural network has 6 hidden layers with 30 neurons each ($4801$ trainable parameters). The training process consists of $5000$ iterations for all cases. The collocation points are resampled every $10$ iterations for SSBroyden and every iteration for NG. SSBroyden is combined with backtracking line search, as implemented in~\cite{dennis1996numerical}. For each setting, the training procedure was repeated five times, and the reported errors correspond to the mean values across these runs, together with their associated standard deviations.}
    \label{tab:Roe_linearization_collocation_points}
\end{table}

To compare the LRPINN results obtained with both optimizers to the numerical reference solution, Figure \ref{fig:lrpinn_prediction} shows the LRPINN predictions at $t=1$ together with the corresponding WENO solution. In addition, we provide a zoomed-in view of the shock region ($x=0$) to better assess the amount of numerical diffusion introduced by the LRPINN. As observed, the location of the shock wave and the jump across the shock are captured correctly. However, the shock predicted by the LRPINN approach appears more smeared compared to the sharper profile produced by the WENO scheme.

\subsubsection{Inviscid Burgers equation using flux relaxation with entropy inequality} \label{sec:inviscid_burgers_eq}

We introduce a novel approach for solving the inviscid Burgers equation using PINNs. To cast the equation in conservative form, we incorporate an algebraic constraint directly into the loss function. Specifically, we employ two neural networks, one to approximate the solution \(u(x,t)\) and another to approximate the flux \(F = u^2\). However, enforcing the algebraic constraint alone does not ensure entropy-compliant solutions. To impose the entropy condition, we utilize the Tadmor entropy inequality \cite{tadmor1987numerical}, defined by the entropy pair 
\[
\eta(u) = \frac{u^2}{2}, \quad \psi(u) = \frac{u^3}{3}.
\] 
The inequality requires that
\[
\frac{\partial \eta}{\partial t} + \frac{\partial \psi}{\partial x} \leq 0,
\]
with equality in smooth regions and dissipation across shocks. A loss term is added to penalize any violations of this inequality.

We express the entropy inequality as
\begin{align}\label{eq:inequality}
\mathcal{R}_{\mathcal{I}} := \frac{\partial \eta}{\partial t} + \frac{\partial \psi}{\partial x} \leq 0.
\end{align}

The total loss function for the entropy-preserving PINN with flux-conserving constraint is formulated as
\begin{align}\label{eq:loss_entropy_burgers}
\mathcal{L}(\eta, \psi, u; \mathbf{\theta_1, \theta_2}) = \lambda_i \mathcal{L}_{ic} + \lambda_b \mathcal{L}_{bc} + \lambda_f \mathcal{L}_{f} + \lambda_{\mathcal{I}} \mathcal{L}_{\mathcal{I}} + \lambda_F \mathcal{L}_F,
\end{align}
where \(\mathcal{L}_{ic}\) and \(\mathcal{L}_{bc}\) denote the initial and boundary condition losses. The flux residual loss \(\mathcal{L}_F\) enforces
\[
\frac{\partial u(x, t ; \theta_1)}{\partial t} + \frac{\partial F(x, t ; \theta_2)}{\partial x} = 0,
\]
while the entropy loss \(\mathcal{L}_{\mathcal{I}}\) penalizes positive entropy residuals:
\[
\mathcal{L}_{\mathcal{I}} = \max(0, \mathcal{R}_{\mathcal{I}}).
\]

The algebraic flux constraint is enforced via
\[
F(x, t ; \theta_2) = u^2(x, t ; \theta_1),
\]
where \(\theta_1\) and \(\theta_2\) are the parameters of the two neural networks. The coefficients \(\lambda_i\), \(\lambda_b\), \(\lambda_f\), \(\lambda_{\mathcal{I}}\), and \(\lambda_F\) weight the contributions of the respective loss terms.

Next, we conduct computational experiments to demonstrate the convergence of the proposed method. The initial and boundary conditions are the same as those specified in Equation~\eqref{eq:ivb}. Both neural networks consist of four fully connected hidden layers, each with 50 neurons. To satisfy the residual and entropy inequality conditions, we sample 50,000 points using Latin hypercube sampling, and 300 points are used for the initial and boundary conditions. We explore different combinations of optimizers and line-search strategies for quasi-Newton methods, with the resulting convergence plots shown in Figure~\ref{fig:Losses_IVB_FR_EI}. The training procedure starts with a warm-up phase using the first-order Adam optimizer, followed by a transition to quasi-Newton optimizers with various line-search routines. This transition is indicated by the vertical dashed line in the convergence plots. For comparison, we also evaluate the SOAP optimizer, which incorporates some second-order features but converges more slowly than the quasi-Newton methods. Among the quasi-Newton approaches, the SSBroyden optimizer combined with a zoom line-search exhibits the best overall performance.

To validate the PINN solution, we compare it against numerical solutions obtained from the DGSEM scheme~\cite{chan2018discretely, ranocha2021adaptive}, the third-order WENO scheme, and the analytical solution. Since the inviscid Burgers' equation develops shocks in finite time, its physically relevant solution must be interpreted in the entropy sense. In this work, we compute the analytical entropy solution using Algorithm~\ref{alg:burgers_entropy}, which is based on the method of characteristics combined with shock detection and entropy enforcement. The characteristic map is given by
\[
x = \xi - t\sin(\pi\xi),
\]
and is derived from the initial condition
\[
u_0(x) = -\sin(\pi x).
\]
For $t \le t_{\text{shock}} = 1/\pi$, the solution remains single-valued and is obtained by inverting the characteristic map to recover the footpoint $\xi$, followed by evaluating
\[
u(x,t) = u_0(\xi).
\]
For $t > t_{\text{shock}}$, characteristics intersect and the classical solution becomes multivalued. The shock location is determined by solving for the critical characteristic point $\xi_s(t)$, after which the spatial domain is partitioned into left, right, and shock regions. The entropy solution is then constructed by restricting the characteristic inversion to admissible intervals on either side of the shock while excluding the multivalued region. Numerical stability is ensured by detecting sign changes in the characteristic mapping and applying a bisection root-finding method to compute the inverse mapping.

\begin{algorithm}[t]
\small
\DontPrintSemicolon
\caption{Analytical Entropy Solution of Burgers' Equation}
\label{alg:burgers_entropy}
\KwIn{$x\in[-1,1]$, $t>0$}
\KwOut{$u(x,t)$}

$u_0(x) \leftarrow -\sin(\pi x)$\;
$\text{CharMap}(\xi,t) = \xi - t\sin(\pi\xi)$\;
$t_{\text{shock}} = 1/\pi$\;

Compute $(x_{\min},x_{\max})$ from $\text{CharMap}$\;
\If{$x \notin [x_{\min},x_{\max}]$}{
    \Return NaN
}

\If{$t \le t_{\text{shock}}$}{
    Find $\xi$ such that $\text{CharMap}(\xi,t)=x$\;
    \Return $u_0(\xi)$
}

Find $\xi_s$ such that $\text{CharMap}(\xi_s,t)=0$\;
$x_s = \text{CharMap}(\xi_s,t)$\;

\If{$x_s < x < -x_s$}{
    \Return NaN
}
\eIf{$x \le x_s$}{
    Find $\xi \in [-1,-\xi_s]$\;
}{
    Find $\xi \in [\xi_s,1]$\;
}

\Return $u_0(\xi)$\;

\end{algorithm}

The comparisons include relative \(L_2\) errors, total model parameters, and runtime for different optimizer and line-search combinations. All models use the same network architecture with 10,401 parameters. The absolute and relative convergence tolerances for both SSBFGS and SSBroyden are set to \(10^{-14}\). Figure~\ref{fig:Plot_Sol_Burger_Entropy} shows the solution at \(t=1\) obtained using a PINN trained with the SSBroyden optimizer and zoom line-search, incorporating flux relaxation and the entropy inequality. \textbf{Panel (a)} compares the PINN solution with a DGSEM solution of order \(N=3\), where \(N\) denotes the polynomial degree within each element. To accurately resolve shocks, the DGSEM employs a shock indicator following~\cite{hennemann2021provably}. For further comparison, the inviscid Burgers equation is also solved using a third-order WENO scheme~\cite{hesthaven2017numerical}. The relative \(L_2\) error between the PINN and DGSEM solutions is \(0.70\%\). \textbf{Panel (b)} shows the PINN solution compared with the WENO solution, yielding a relative \(L_2\) error of \(0.7056\%\), while \textbf{panel (c)} compares it with the analytical solution, resulting in a relative \(L_2\) error of \(0.7172\%\). \textbf{Panel (d)} provides a zoomed-in view near the shock, indicated by the dashed box in \textbf{panel (a)}, demonstrating that DGSEM captures the shock more accurately than both PINN and WENO, closely matching the analytical solution. These discrepancies near the shock dominate the errors observed in \textbf{panels (a)--(c)}. Finally, \textbf{panels (d), (e), and (f)} present the pointwise \(L_1\) errors corresponding to the comparisons in \textbf{panels (a), (b), and (c)}, respectively, with the maximum error occurring at the shock location. Overall the error obtained by using this method is much lower than Roe Linearization approach.

\begin{figure}
    \centering
    \begin{minipage}[b]{\linewidth}
        \centering
        \includegraphics[width=\linewidth]{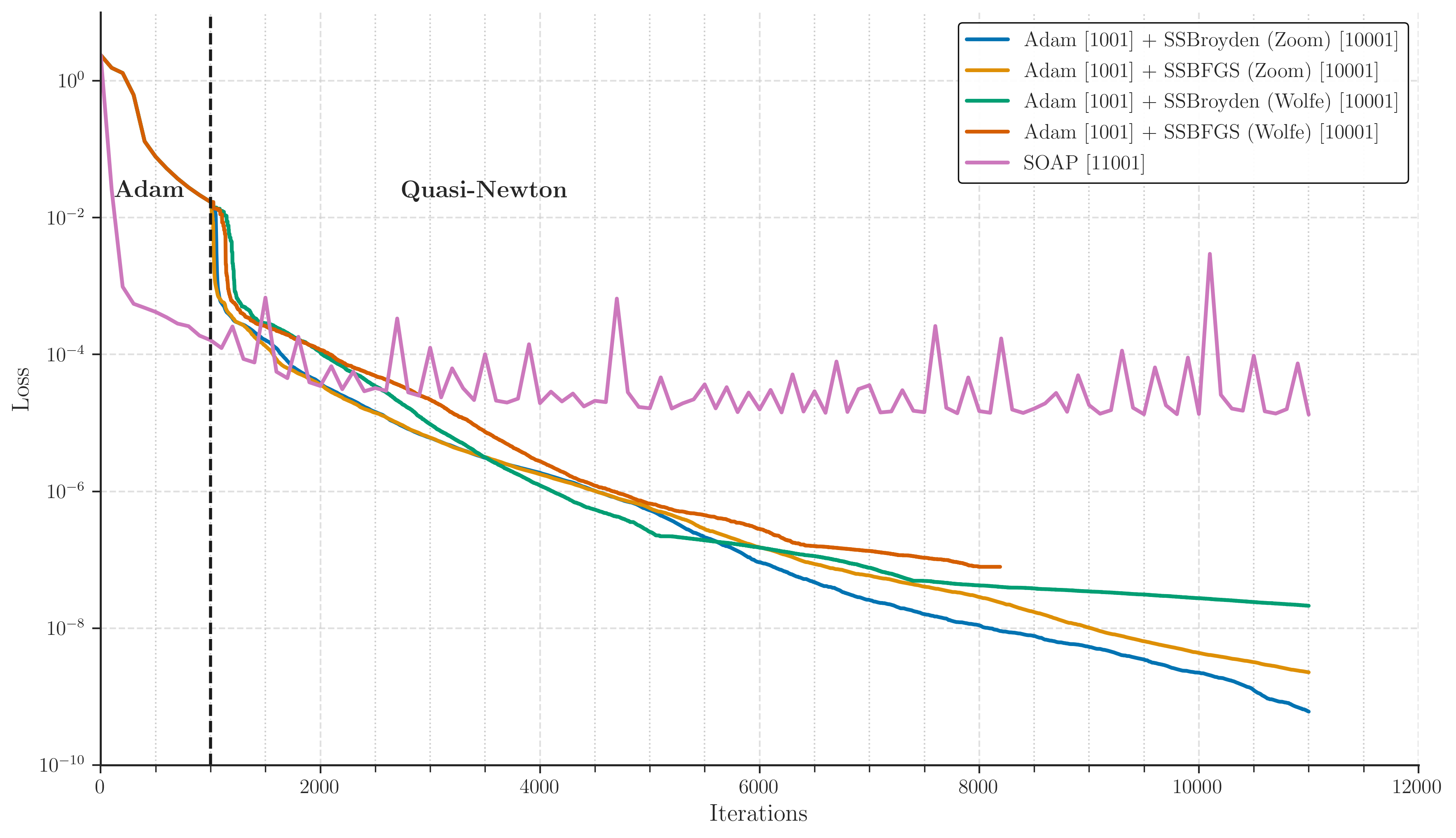}
        \caption{\textbf{Inviscid Burgers equation using relaxation and entropy inequality:} 
        Loss history for the inviscid Burgers flow obtained using different combinations of optimizers and line-search strategies, incorporating flux relaxation and an entropy inequality. The training procedure begins with a warm-up phase using the first-order Adam optimizer, followed by a transition to quasi-Newton optimizers combined with various line-search routines. The optimizer transition is indicated by the vertical dashed line. For comparison, we also report the performance of the SOAP optimizer, which exhibits slower convergence compared to the quasi-Newton methods. Among the quasi-Newton approaches, the SSBroyden optimizer coupled with a zoom line-search strategy exhibits the best overall performance.}
        \label{fig:Losses_IVB_FR_EI}
    \end{minipage}

    \vspace{1em}

    \begin{minipage}[b]{\linewidth}
        \centering
        \rowcolors{2}{cyan!15}{white}
        \scalebox{0.74}{
        \begin{tabular}{|l|c|c|c|}
            \hline
            \rowcolor{cyan!40}
            \textbf{Optimizer [\# Iters.]} & \textbf{Relative \(L_2\): $u$ [\textcolor{blue}{DGSEM}}, \textcolor{red}{WENO}, \textcolor{black}{Ana.}] & \textbf{\# parameters} & \textbf{Training time (s)} \\
            \hline
            SOAP (11001) & [\textcolor{blue}{$1.3 \times 10^{1}$}, \textcolor{red}{$1.28 \times 10^{1}$}, \textcolor{red}{$3.7 \times 10^{0}$}] & 10401 & 142 \\
            \hline
             Adam (1001) + SSBroyden with Zoom (10000) & [\textcolor{blue}{$7.1 \times 10^{-3}$}, \textcolor{red}{$7.1 \times 10^{-3}$}, \textcolor{black}{$7.3 \times 10^{-3}$}] & 10401 & 364 \\
            \hline
            Adam (1001) + SSBroyden with Wolfe (10000) & [\textcolor{blue}{$7.1 \times 10^{-3}$}, \textcolor{red}{$6.6 \times 10^{-3}$}, \textcolor{black}{$6.8 \times 10^{-3}$}]  & 10401 & 311 \\
            \hline
            Adam (1001) + SSBFGS with Zoom (10000) & [\textcolor{blue}{$6.9 \times 10^{-3}$}, \textcolor{red}{$6.8 \times 10^{-3}$}, \textcolor{black}{$7 \times 10^{-3}$}] & 10401 & 436 \\
           \hline
            Adam (1001) + SSBFGS with Wolfe (7190) & [\textcolor{blue}{$7.1 \times 10^{-3}$}, \textcolor{red}{$6.7 \times 10^{-3}$}, \textcolor{black}{$6.9 \times 10^{-3}$}] & 10401 & 259 \\
            \hline
        \end{tabular}}
        \captionof{table}{\textbf{Inviscid Burgers equation using relaxation and entropy inequality:}
      Comparison of relative $L_2$ errors, total model parameters, and runtime for different combinations of optimizers and line-search strategies. All models employ the same network architecture with $10{,}401$ parameters.
The absolute and relative convergence tolerances for both SSBFGS and SSBroyden are set to $10^{-14}$.
All runs are performed for $11{,}001$ training iterations.All computations are carried out on an NVIDIA H100 GPU with $76\%$ persistent memory usage (80~GB total) and approximately $1\%$ compute utilization, achieving 256~GFLOPs out of a theoretical peak of 25.61~TFLOPs.
The use of the H100 GPU is intended to maintain a uniform computing platform across all benchmarks.
Consequently, the H100 GPU is overprovisioned for this problem, as the overall runtime is dominated by latency rather than compute throughput.
}
\label{tab:Error_IVB_FR_EI}
    \end{minipage}
\end{figure}

\begin{figure}
    \centering  \includegraphics[width=\linewidth,height=0.60\textheight,keepaspectratio]{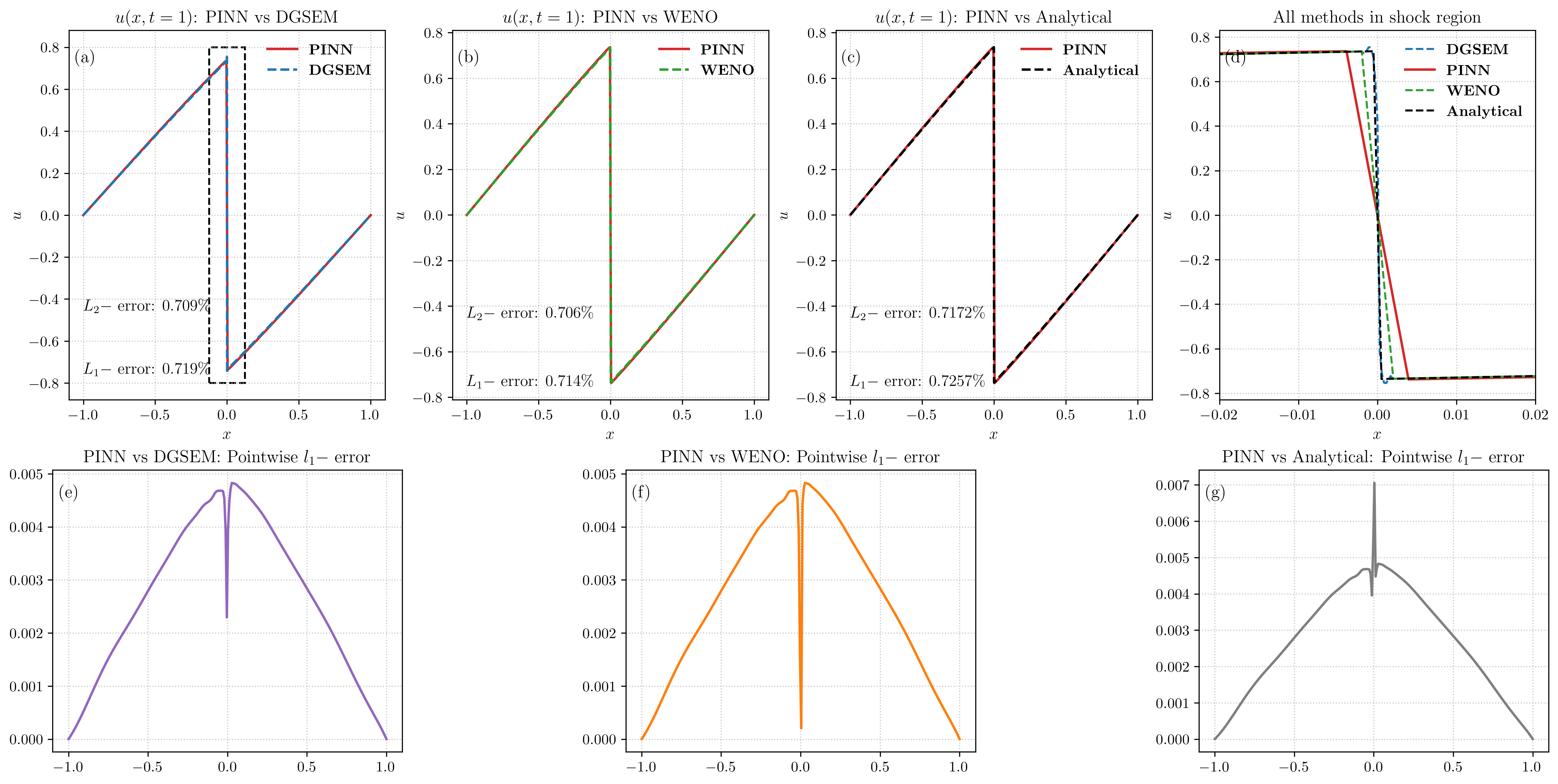}
     \captionsetup{font=small}
    \caption{\textbf{Inviscid Burgers equation using relaxation and entropy inequality:} The solution of Burgers’ equation, incorporating an entropy inequality and flux relaxation, is obtained using a physics-informed neural network (PINN) trained with the SSBroyden optimizer and a zoom line-search strategy. \textbf{Panel (a)} compares the PINN solution with a DGSEM solution of order \(N=3\), where \(N\) represents the polynomial degree used to approximate the solution within each element. To accurately resolve the shock, the DGSEM employs a shock indicator following \cite{hennemann2021provably}. For further comparison, the inviscid Burgers’ equation is also solved using a third-order WENO scheme \cite{hesthaven2017numerical}. The relative \(L_2\) error between the PINN and DGSEM solutions shown in \textbf{panel (a)} is \(0.70\%\).
\textbf{Panel (b)} shows the comparison between the PINN solution and the third-order WENO solution, yielding a relative \(L_2\) error of \(0.7056\%\), while \textbf{panel (c)} presents the comparison with the analytical solution, resulting in a relative \(L_2\) error of \(0.7172\%\). \textbf{Panel (d)} provides a zoomed-in view near the shock, indicated by the dashed box in \textbf{panel (a)}. This close-up demonstrates that the DGSEM method captures the shock more accurately than both the PINN and WENO schemes and closely matches the analytical solution. Consequently, the discrepancies near the shock dominate the errors observed in \textbf{panels (a), (b), and (c)}.
Finally, \textbf{panels (d), (e), and (f)} present the pointwise \(L_1\) errors corresponding to the comparisons shown in \textbf{panels (a), (b), and (c)}, respectively, with the maximum error occurring at the shock location.
}
\label{fig:Plot_Sol_Burger_Entropy}
\end{figure}

\paragraph{Key observations from the Inviscid Burgers equation:}
Two different approaches were investigated for solving the inviscid Burgers equation with PINNs. LRPINN strategy and the relaxation–entropy inequality formulation. The Roe-based approach improves stability by incorporating shock-speed information into the residual, allowing the network to capture discontinuities more robustly than the classical strong-form PINN. However, some numerical diffusion remains near the shock. The relaxation–entropy inequality formulation further enhances physical consistency by enforcing conservation and entropy admissibility directly in the loss function. This approach yields significantly improved accuracy, particularly in the shock region, and produces solutions that are closer to high-resolution DGSEM and WENO reference solutions. If these modifications are not incorporated into the fundamental PINN algorithm, the training process degenerates, and convergence to the true solution is not achieved regardless of the optimizer used. This behavior is consistent with classical numerical methods for hyperbolic conservation laws, where enforcing conservation properties—such as energy or entropy—is essential \cite{hesthaven2017numerical}. Without these constraints, the solution may become non-physical or exhibit instability, potentially leading to blow-up. 

For the LRPINN algorithm, SSBroyden exhibits a stable behavior across all configurations, achieving a relative $L_2$ error of approximately $0.08$ with a comparatively low computational cost. In contrast, the NG optimizer shows larger variability for smaller number of interior points, as reflected by the higher standard deviations. Although NG reaches slightly lower errors for larger numbers of interior points (e.g. $N_\Omega =5000$), this improvement comes at the expense of a significantly higher runtime. These results indicate that SSBroyden provides a more robust and computationally efficient optimization strategy, while NG achieves competitive accuracy only when a sufficiently large number of training points is used.

\paragraph{Key observations from the (2+1)D Viscous Burgers equation:}
\noindent
The two-dimensional unsteady viscous Burgers equation is parabolic in nature; as a consequence, even a small amount of viscosity introduces sufficient regularization into the system. This inherent regularization leads to a smoother and more well-conditioned loss landscape during optimization. As a result, self-scaled optimizers are able to achieve a very high level of accuracy. Furthermore, a comparable level of accuracy is also observed when using NG methods. However, an important distinction lies in computational efficiency: while NG methods recover essentially the same level of accuracy as self-scaled optimizers, they do so at nearly an order of magnitude (approximately $10\times$) lower computational cost. In contrast, the SOAP optimizer fails to attain a similar level of performance. Specifically, it does not reach accuracy comparable to either the self-scaled optimizers or the NG methods, indicating limitations in its effectiveness for this class of problems.

\FloatBarrier
\subsection{(1+1)D Euler equations } \label{sec:euler_equation}

Next, we consider a general system of one-dimensional conservation laws of the form
\begin{equation}\label{eq:conservative1D}
    \frac{\partial \mathbf{U}}{\partial t} + \frac{\partial \mathbf{F}}{\partial x} = 0,
\end{equation}
where $\mathbf{U}$ denotes the vector of conserved variables and $\mathbf{F}(\mathbf{U})$ represents the corresponding flux vector.

As a representative example, we focus on the compressible Euler equations, which govern the dynamics of an inviscid, adiabatic flow. In this setting, the conserved variables and their associated fluxes are given by
\begin{equation}
    \mathbf{U} =
    \begin{pmatrix}
        \rho \\
        \rho u \\
        E
    \end{pmatrix},
    \qquad 
    \mathbf{F}(\mathbf{U}) =
    \begin{pmatrix}
        \rho u \\
        \rho u^2 + p \\
        u(E + p)
    \end{pmatrix},
\end{equation}
where $\rho$ denotes the density, $u$ the velocity, $E$ the total energy density, and $p$ the pressure. To close the system, an equation of state relating $E$, $p$, and $\rho$ is required. For simplicity, we assume an ideal adiabatic gas, for which
\begin{equation}
    E = \frac{1}{2}\rho u^2 + \frac{p}{\gamma - 1},
\end{equation}
with $\gamma = 1.4$ representing the ratio of specific heats.

Similar to the inviscid Burgers equation, this system also generates shocks and contact waves. Simply employing advanced optimizers is not sufficient for the PINN to converge, as these discontinuities cause rapid degeneration of the loss function. To address this challenge, we propose two novel approaches specifically designed to handle such discontinuous solutions, representing a fundamental contribution to the development of PINN architectures.
Two distinct approaches were investigated to handle discontinuities within the PINN framework, The Roe linearization strategy and the HLLC flux-based formulation. Both methods aim to incorporate physically consistent shock-capturing mechanisms into the residual construction, thereby improving stability and convergence in the presence of shocks and contact discontinuities.

\subsubsection{Roe linearization approach}
In particular, if we consider that $\mathbf{F} = \mathbf{F}(\mathbf{U})$, then Equation~\eqref{eq:conservative1D} can also be written in the following quasilinear form
\begin{equation}\label{eq:quasilinear1D}
    \frac{\partial \mathbf{U}}{\partial t} + \mathbf{A}(\mathbf{U})\frac{\partial\mathbf{U}}{\partial x} = 0,
\end{equation}
where $\mathbf{A} = \mathbf{A}(\mathbf{U})$ is the Jacobian of $\mathbf{F}$. Here we describe briefly a way to generalize the methodology described before to handle systems of equations like this. We refer the reader to~\cite{UrbanPons2025} for more details, as well as to see its adaptation for problems with more than one dimension.

 The Jacobian matrix $\mathbf{A}$ for this equation of state is given by
\begin{equation}\label{eq:jacobian_euler}
\mathbf{A}(\mathbf{U}) = 
\begin{bmatrix}
0 & 1 & 0 \\
\frac{\gamma-3}{2}u^2 & -\left(\gamma-3 \right) u & \gamma -1 \\
u \left(-H + \frac{\gamma -1}{2}u^2 \right) & H - \left(\gamma -1 \right)u^2 & \gamma u
\end{bmatrix},
\end{equation}
where $H$ is the total specific enthalpy
\begin{equation}\label{eq:specific_enthalpy}
    H = \frac{E + p}{\rho} = \frac{a^2}{\gamma-1} + \frac{u^2}{2},
\end{equation}
and $a$ is the sound speed
\begin{equation}\label{eq:a_euler}
    a = \sqrt{\frac{\gamma p}{\rho}}.
\end{equation}

 Shock formation is physically characterized by a strong compression of fluid particles, so the material derivative of the density $D\rho/Dt$ should be large and positive there. Using the definition of the material derivative $D/Dt = \partial/ \partial t + u \partial /\partial x$ together with the Euler equation for the density, we get
\begin{equation*}
    \frac{D \rho}{Dt} + \rho \frac{\partial u}{\partial x} = 0.
\end{equation*}
Then, we can characterize shock regions by noting where the spatial derivative of the velocity field is large and negative. So, we define them with the following set
\begin{equation}
    \Omega_s = \left \lbrace (x,t) \in \Omega, \; \frac{\partial u}{\partial x} \leq -M \right \rbrace,
\end{equation}
where $M$ is a positive constant. 

Now, for all $(x,t) \in \Omega_s$, we modify the PDE by substituting the original Jacobian $\mathbf{A}$ in \eqref{eq:quasilinear1D} with a suitable modification of it. 
Denoting this modification by $\tilde{\mathbf{A}}$, it should follow (see \cite{ROE1981357})
\begin{itemize}
    \item $\mathbf{\tilde{A}}$ should be diagonalizable with real eigenvalues, in the same way as $\mathbf{A}$.
    \item Consistency: If the jump of the solution goes to zero, then $\mathbf{\tilde{A}} \rightarrow \mathbf{A}$.
    \item Conservation: For left and right states $\mathbf{U}_L$ and $\mathbf{U}_R$, then $\mathbf{F}(\mathbf{U}_R) - \mathbf{F}(\mathbf{U}_R) = \mathbf{\tilde{A}}(\mathbf{U}_R - \mathbf{U}_L)$. 
\end{itemize}
A well-known choice in classical numerical methods is the Roe Jacobian matrix \cite{ROE1981357}, which is defined as
\begin{equation}
    \mathbf{A}_{\text{Roe}} = \mathbf{A}(\mathbf{\tilde{U}}).
\end{equation}
That is, we evaluate the Jacobian matrix defined in \eqref{eq:jacobian_euler} at the vector $\mathbf{\tilde{U}}$, which results from substituting the velocity, the enthalpy and the sound speed by the following averages
\begin{align}
    \tilde{u} &= \frac{\sqrt{\rho_L}u_L + \sqrt{\rho_R}u_R}{\sqrt{\rho_L} + \sqrt{\rho_R}}, \\
    \tilde{H} &= \frac{\sqrt{\rho_L}H_L + \sqrt{\rho_R}H_R}{\sqrt{\rho_L} + \sqrt{\rho_R}}, \\
    \tilde{a} &= \sqrt{\left(\gamma -1 \right)\left(\tilde{H} - \frac{\tilde{u}^2}{2} \right)}
\end{align}

Now, the PDE residuals for each point $(x,t) \in \Omega$ is given by
\begin{equation}\label{eq:lrpinn_residuals}
    \mathcal{L}_\theta(x,t) = \frac{\partial \mathbf{U}_\theta(x,t)}{\partial t} + \mathbf{\tilde{A}}_\theta\frac{\partial \mathbf{U}_\theta(x,t)}{\partial x},
\end{equation}
where $\mathbf{U}_\theta$ represents the neural network parametrized with trainable variables $\Theta$, and
\begin{equation}
    \mathbf{\tilde{A}}_\theta=
    \begin{cases}
       \mathbf{ \mathbf{A}_{\text{Roe}}}, \; &(x,t) \in \Omega_s, \\
       \mathbf{A}(\mathbf{U}_\theta), \; &\text{Otherwise}.
    \end{cases}
\end{equation}
The left and right states calculated by evaluating the neural network at $x \pm h$. The resulting algorithm is summarized in Algorithm 2 of \cite{UrbanPons2025}. 
\begin{figure}[h!]
      \centering
      \includegraphics[width=0.99\linewidth]{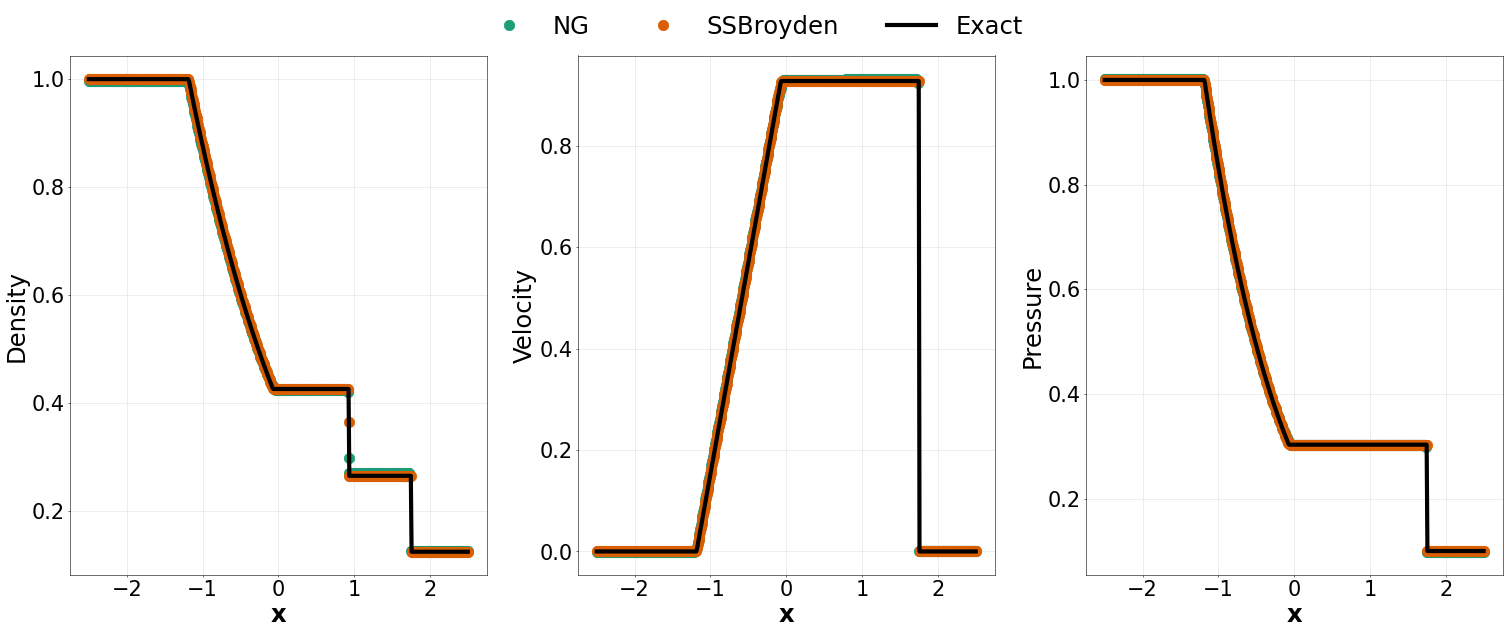}
      \caption{\textbf{1D Euler equations with Roe linearization:} Neural network predictions for the density, velocity and pressure at $t=1$. The exact solution is also plotted for reference.}
      \label{fig:lrpinn_sod_prediction}
  \end{figure}
We consider the classical Sod shock tube, with initial condition
\begin{equation}
    (\rho,u,p) =
    \begin{cases}
        (1,0,1) \: &x< 0, \\
        (0.125,0,0.1) & x>0 
    \end{cases}
\end{equation}
defined in the domain $\Omega = [-2.5,2.5] \times[0,1]$. Following the same strategy as in \cite{UrbanPons2025}, we train for $1000$ iterations using the following viscous regularization
\begin{equation}
    \frac{\partial \mathbf{U}}{\partial t} + \frac{\partial \mathbf{F}}{\partial x} = \nu \frac{\partial^2 \mathbf{U}}{\partial x^2},
\end{equation}
with $\nu = 0.003$, and then we eliminate the viscosity by using the Roe linearization method described before. To homogenize the PDE
residuals in the entire domain, we multiply point by point \eqref{eq:lrpinn_residuals} by gradient-annihilated factor introduced in \cite{GAPINN}
\begin{equation}
    \alpha(x,t) = \frac{1}{a \left| \frac{\partial u_\theta}{\partial x}\right| + 1},
\end{equation}
where $a$ is an hyperparameter. We choose $(M,h,a) = (0.001,0.0175,0.01)$ for this problem. Finally, points are resampled every 100 iterations, where every batch consists on $15000$ points in the domain.
\begin{figure}
    \centering
    \begin{minipage}[b]{\linewidth}
        \centering
        \includegraphics[width=0.99\linewidth]{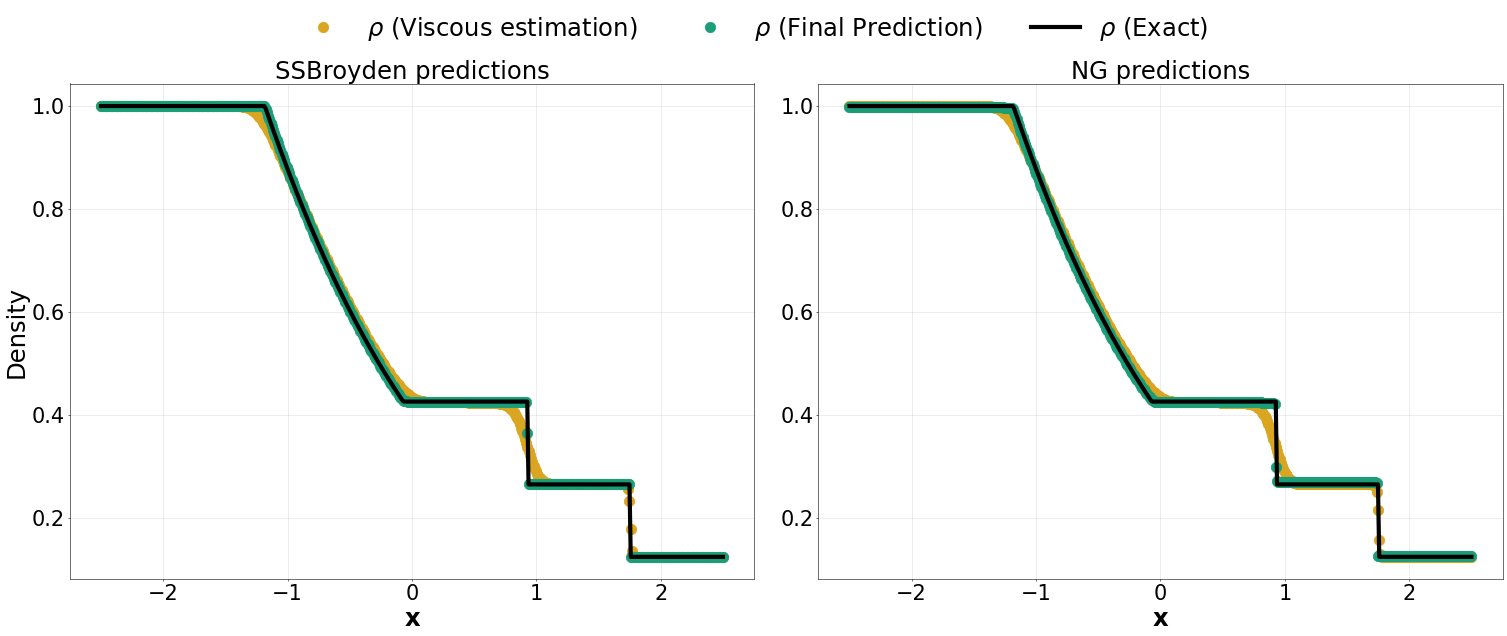}
        \caption{\textbf{1D Euler equations with Roe linearization:}
        Neural network prediction for the density, before and after the \emph{inviscid} part of the training process.}
        \label{fig:before_and_after_lrpinn}
    \end{minipage}
    \vspace{1em}
    \begin{minipage}[b]{\linewidth}
        \centering
        \rowcolors{2}{cyan!15}{white}
        \scalebox{0.95}{
        \begin{tabular}{|l|l|c|c|c|}
            \hline
            \rowcolor{cyan!40}
            \textbf{Case} & \textbf{Optimizer} & \textbf{Relative $L_1$ $(\rho, u, p)$} & \textbf{Training time (s)} & \textbf{\# parameters}  \\
            \hline
            1 & SSBroyden    & $(9.2,8.0,5.0)\times 10^{-4}$    & 160 & 4833  \\
            \hline
            2 & NG & $(6.5,9.1,3.3)\times 10^{-3}$ & 452 & 4833  \\
            \hline
        \end{tabular}}
        \captionof{table}{\textbf{1D Euler equations with Roe linearization:}
        Training times for the Sod shock tube problem.}
        \label{tab:sod_training_times}
    \end{minipage}
\end{figure}

Figure \ref{fig:lrpinn_sod_prediction} shows the predicted density, velocity, and pressure at $t = 1$, together with the corresponding exact solutions. The results are very accurate, with very few or even no transition points in the shock and contact waves, as well as very abrupt changes in the beginning and in the end of the rarefaction wave. To see the impact of the inviscid part of the training process in the neural network prediction, Figure \ref{fig:before_and_after_lrpinn} the density obtained before and after the inviscid phase of the training process. As shown, the introduction of the Roe linearization in shock regions let us eliminate the viscosity present in all the features of the solution without any spurious oscillations near the discontinuities, allowing us to correctly predict the jumps in all variables by imposing local conservation. The results, including the errors, total number of parameters, and runtime, are summarized in Table \ref{tab:sod_training_times}.
     
\begin{figure}
    \centering
    \includegraphics[trim={0cm 18cm 0cm 0cm}, clip, width=\textwidth]{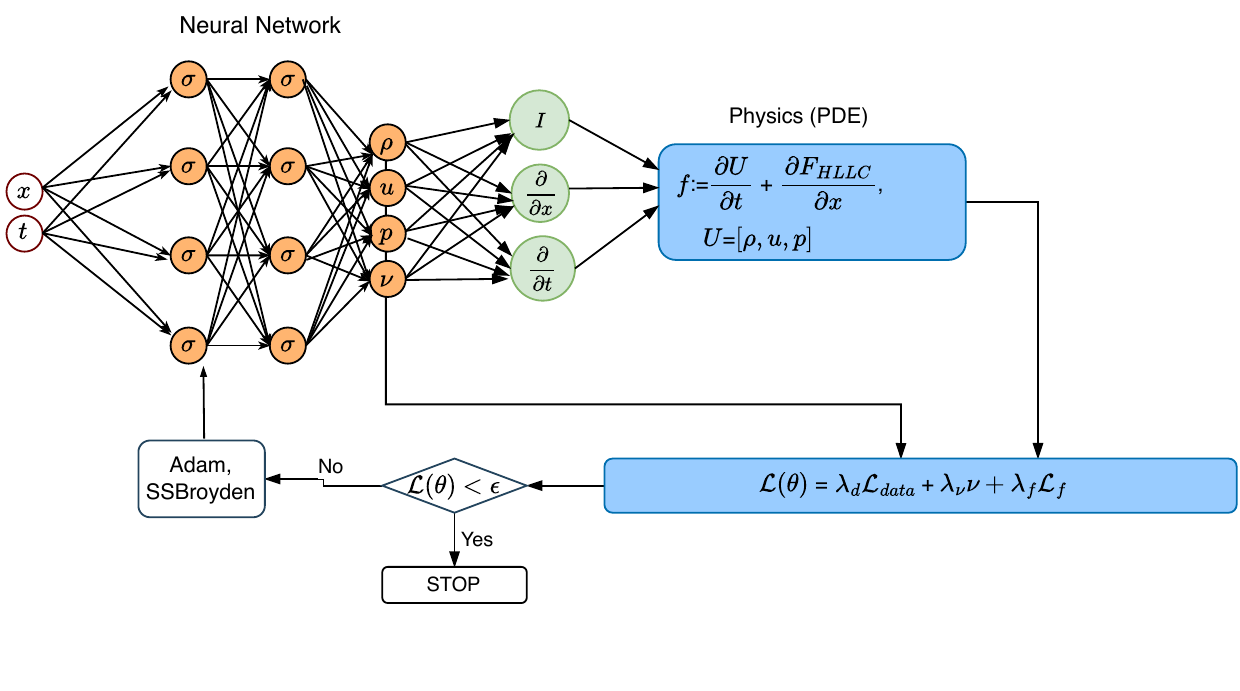}
    \caption{\textbf{1D Euler Equations with  HLLC Flux:} PINN architecture for 1D Euler equation using adaptive viscosity and  HLLC flux.}
    \label{fig:Euler_1D_PINN_architecture}
\end{figure}

\FloatBarrier

\subsubsection{Harten--Lax--van Leer--Contact (HLLC) Flux approach }
\label{sec:Euler_equations_HLLC}
Here, we propose a new method based on the Harten--Lax--van Leer--Contact (HLLC) approximate Riemann solver \cite{toro1994restoration}. In computational fluid dynamics, HLLC is often preferred over Roe's linearization for the Euler equations due to its improved robustness and stronger physical consistency in the presence of strong discontinuities. Although Roe's method can provide sharp wave resolution through a local Jacobian linearization, it may violate the entropy condition in transonic rarefaction regions and can yield nonphysical states, such as negative density or pressure, in strong rarefactions or near-vacuum regimes unless additional entropy fixes and positivity-preserving corrections are employed. In contrast, the HLLC flux is derived from a three-wave approximate Riemann structure that explicitly restores the contact discontinuity and, with appropriate wave-speed estimates, naturally maintains physically admissible states. Consequently, HLLC provides a stable and reliable shock-capturing framework for high-Mach-number flows and problems involving strong shocks, while retaining accurate resolution of contact and shear waves.

To incorporate the HLLC flux within the PINN framework, we first solve a viscous approximation of the Euler equations. The corresponding residuals are defined as
\begin{align}\label{eq:viscous_euler}
R_0 &= \rho_t + u \rho_x + \rho u_x - \nu \rho_{xx}, \\[6pt]
R_1 &= \rho u_t + \rho u u_x + p_x 
      - \nu \left( \rho u_{xx} + 2 \rho_x u_x \right), \\[6pt]
R_2 &= p_t + u p_x + \gamma p u_x 
      - \nu \left( p_{xx} + \rho (\gamma - 1) u_x^2 \right).
\end{align}
In Equation~\eqref{eq:viscous_euler}, the viscosity $\nu$ is not prescribed or tuned; instead, within the PINN architecture it is treated as an additional network output together with $\rho$, $u$, and $p$, as illustrated in Figure~\ref{fig:Euler_1D_PINN_architecture}. Since the target problem is inviscid, we enforce $\nu \to 0$ by introducing it as a soft constraint in the loss function. The loss used in this warm-up stage is therefore given by
\begin{align}\label{eq:loss_function_euler}
\mathcal{L} = \lambda_{ic}\mathcal{L}_{ic} + \lambda_{bc}\mathcal{L}_{bc} + \lambda_f\mathcal{L}_{f} + \lambda_{\nu}\nu.
\end{align}
After this initial training stage, the residual loss is computed using the inviscid conservative form with the HLLC numerical flux,

\begin{align}
\frac{\partial \mathbf{U}}{\partial t} + \frac{\partial \mathbf{F}_{\mathrm{HLLC}}}{\partial x} = 0.
\end{align}

\noindent
For the one-dimensional Euler equations, the HLLC numerical flux is defined by \cite{toro1994restoration}
\begin{equation}
\mathbf{F}_{\mathrm{HLLC}}(\mathbf{U}_L,\mathbf{U}_R)=
\begin{cases}
\mathbf{F}(\mathbf{U}_L), & S_L \ge 0, \\[6pt]
\mathbf{F}(\mathbf{U}_L)+S_L\left(\mathbf{U}_L^{*}-\mathbf{U}_L\right), 
& S_L \le 0 \le S_{*}, \\[6pt]
\mathbf{F}(\mathbf{U}_R)+S_R\left(\mathbf{U}_R^{*}-\mathbf{U}_R\right), 
& S_{*} \le 0 \le S_R, \\[6pt]
\mathbf{F}(\mathbf{U}_R), & S_R \le 0.
\end{cases}
\end{equation}

The intermediate (star) states are given by
\begin{equation}
\mathbf{U}_K^{*}
=
\rho_K \frac{S_K-u_K}{S_K-S_*}
\begin{pmatrix}
1\\
S_*\\[4pt]
\displaystyle
\frac{E_K}{\rho_K} + (S_* - u_K)
\left(
S_* + \frac{p_K}{\rho_K (S_K-u_K)}
\right)
\end{pmatrix},
\qquad K\in\{L,R\},
\end{equation}
and the contact wave speed is
\begin{equation}
S_*=
\frac{
p_R-p_L+\rho_L u_L(S_L-u_L)-\rho_R u_R(S_R-u_R)
}{
\rho_L(S_L-u_L)-\rho_R(S_R-u_R)
}.
\end{equation}

\begin{figure}
    \centering
\includegraphics[width=\linewidth,height=0.72\textheight,keepaspectratio]{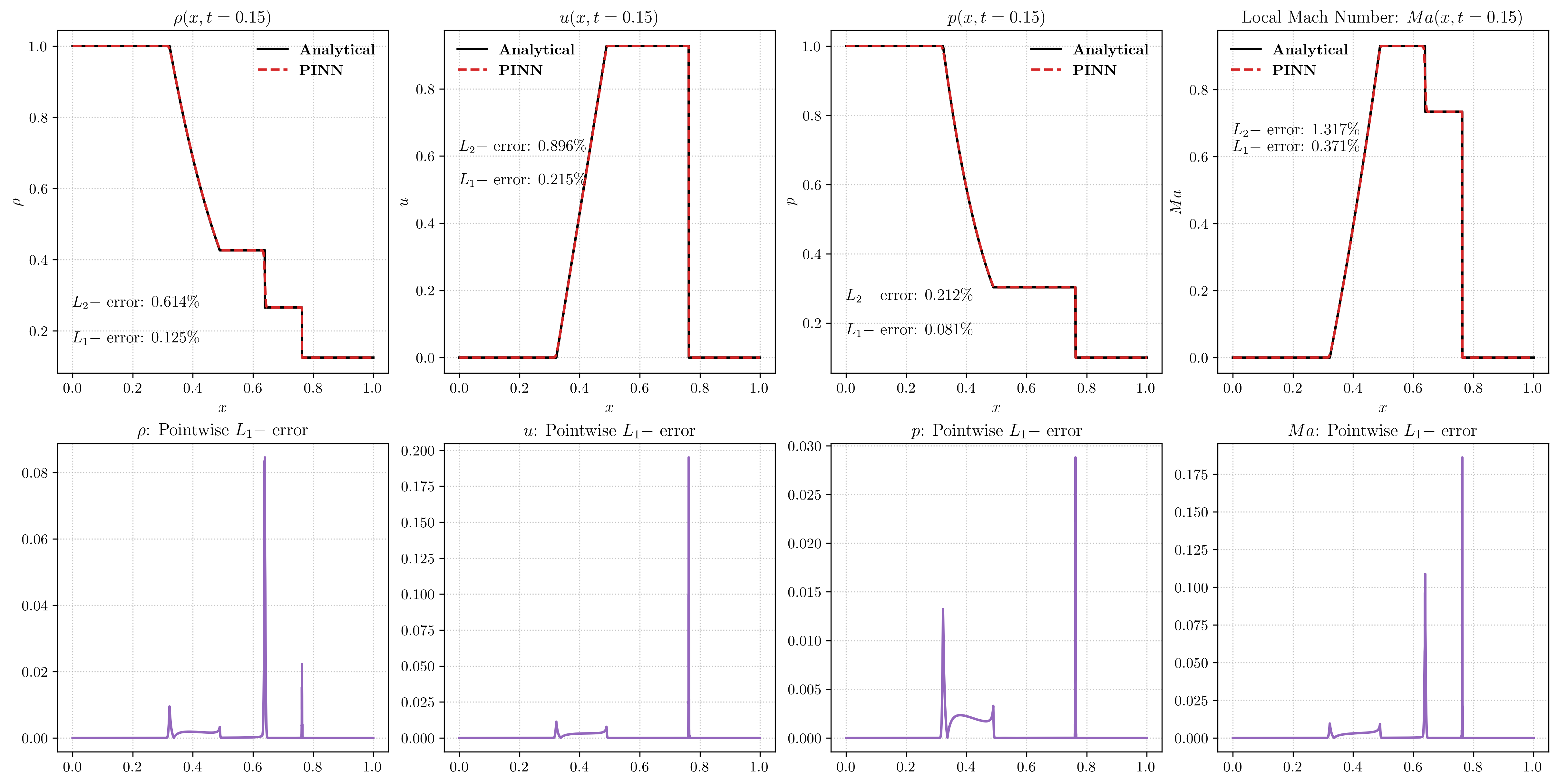}
    \captionsetup{font=small}
    \caption{\textbf{Euler equations with  HLLC  flux vs analytical solution:} Comparison of the 1D Euler solution for the Sod shock tube problem obtained using the PINN framework augmented with the HLLC flux against the analytical solution. The top row displays the density $\rho$, velocity $u$, pressure $p$, and local Mach number $Ma$ (from left to right) at $t = 0.15$. The bottom row presents the corresponding pointwise absolute error distributions in space. The results demonstrate excellent agreement between the analytical and PINN-predicted solutions, with accurate resolution of both the shock and the contact discontinuity.}
    \label{fig:Euler_1D_Analytical}
\end{figure}

Next, we conduct a computational experiment to assess the effectiveness of the proposed method. The PINN architecture consists of six hidden layers, each containing 20 neurons, forming a fully connected network capable of representing the nonlinear dynamics of the flow. A hyperbolic tangent (Tanh) activation function is employed in all hidden layers of the network. To enforce the governing equations, 30,000 collocation points are sampled in the computational domain to impose the residual constraints.
The training is performed in two stages. In the first (viscous warm-up) stage, the network is trained using the Adam optimizer for 5000 iterations in order to obtain a stable approximation of the viscous regularized system. In the second (inviscid) stage with HLLC flux, the optimization is switched to the SSBroyden quasi-Newton method with a zoom line search algorithm, which improves convergence and allows the network to satisfy the inviscid Euler residual formulated using the HLLC flux more accurately. 
\begin{figure}
    \centering  \includegraphics[width=\linewidth,height=0.72\textheight,keepaspectratio]{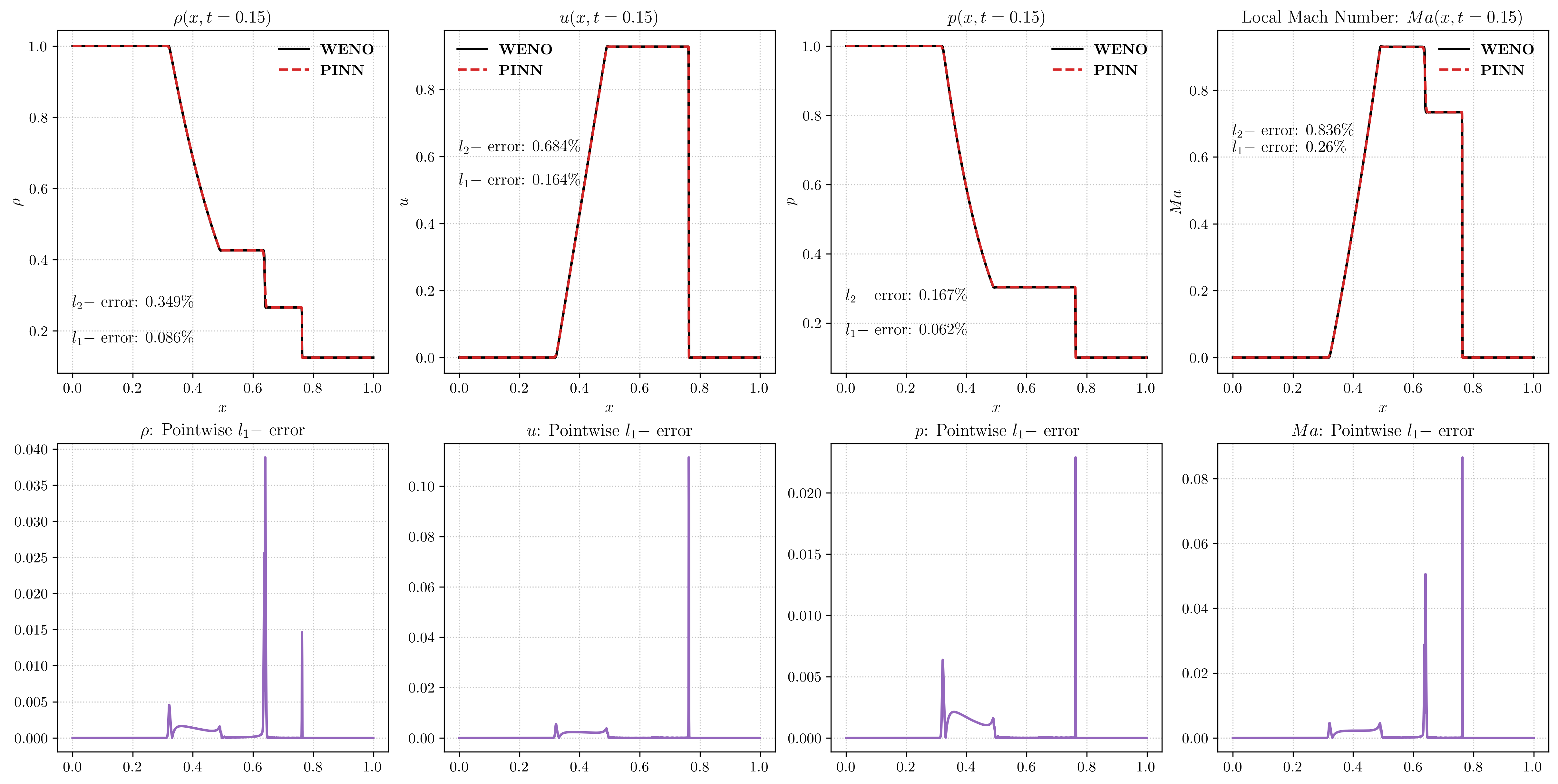}
    \captionsetup{font=small}
    \caption{\textbf{Euler Equations with  HLLC flux vs WENO solution:} Comparison of the 1D Euler solution for the Sod shock tube problem obtained using the PINN framework augmented with the HLLC flux against a numerical reference solution computed using a WENO scheme \cite{hesthaven2017numerical}. The top row shows the density $\rho$, velocity $u$, pressure $p$, and local Mach number $Ma$ (from left to right) at $t = 0.15$. The bottom row presents the corresponding pointwise absolute error in space. The results indicate very good agreement between the WENO reference solution and the PINN predictions, with the shock and contact discontinuity captured accurately. Overall, the discrepancy between the PINN solution and the WENO reference solution is smaller than that observed with respect to the analytical solution, primarily due to the inherent dissipative nature of the WENO scheme.}
    \label{fig:Euler_1D_WENO}
\end{figure}
\begin{table}[t]
\centering
\resizebox{\textwidth}{!}{
\begin{tabular}{|p{5cm}|c|c|c|c|}
\hline
\rowcolor{cyan!40}
\textbf{Optimizer [\# Iters.]} 
& \textbf{Relative $L_1$  ($u,~v,~p$)}
& \textbf{Relative $L_2$  ($u,~v,~p$)} 
& \textbf{\# parameters} 
& \textbf{Training time (s)} \\
\hline
Adam (5001) + SSBroyden with Zoom (10000) + Analytical 
& $[1.3~ 2.2, ~ 0.8] \times 10^{-3}$
& $[6.1,~8.9,~2.1] \times 10^{-3}$
& 2601 
& 460 \\
\hline
Adam (5001) + SSBroyden with Zoom (10000) + WENO 
& $[0.8~ 1.6, ~ 0.6] \times 10^{-3}$
& $[3.4,~6.8,~1.7] \times 10^{-3}$ 
& 2601 
& 460 \\
\hline
\end{tabular}}
\caption{\textbf{1D Euler Equations with  HLLC Flux:} Relative $L_1$ and $L_2$ error comparison with analytical and WENO reference solutions for the 1D Euler equations using PINN with HLLC flux.}
\label{tab:optimizer_results_hllc}
\end{table}

Figure~\ref{fig:Euler_1D_Analytical} illustrates the performance of the proposed PINN--HLLC framework for the Sod shock tube problem at $t=0.15$. The predicted profiles of density, velocity, pressure, and local Mach number are overlaid with the exact solution to evaluate the approximation. In addition, the spatial distribution of the absolute error is presented to quantify the local discrepancies. The comparison confirms that the proposed approach successfully reproduces the main flow features, including the shock front and contact discontinuity, while maintaining a low error level throughout the computational domain. In Figure~\ref{fig:Euler_1D_WENO}, we repeat the comparison presented in Figure~\ref{fig:Euler_1D_Analytical}, but replace the exact solution with a numerical reference computed using a third-order WENO scheme. The PINN predictions remain in close agreement with the WENO solution across all variables. Moreover, the observed errors are slightly reduced compared to the analytical case, which can be attributed to the additional numerical diffusion introduced by the WENO discretization. The results, including the relative $L_2$ and $L_1$ errors, total number of parameters, and runtime, are summarized in Tablel~\ref{tab:optimizer_results_hllc}. All runtimes were measured on an NVIDIA H100 80\,GB GPU. Since the compute utilization remained below 60\% during training, the GPU was not fully saturated by floating-point operations, which led to increased runtime.

\paragraph{Key observations from the 1D Euler equations:}
Two different methodologies were explored for solving the 1D Euler equations within the PINN framework: a Roe linearization-based approach and an HLLC flux-based formulation. 
The Roe-based approach improves local conservation through Jacobian linearization, while the HLLC formulation provides enhanced robustness and physical admissibility, particularly in the presence of strong shocks. We observe that both SSBroyden and NG converges if the Roe-linearization approach is considered, being again SSBroyden the winner in terms of computational cost and precission. 



\FloatBarrier
\subsection{Stiff PK--PD System with Discontinuous Drug Administration}
\label{sec:stiff_odes}

We consider a coupled PK–PD model describing the effect of a cytotoxic chemotherapy agent (paclitaxel) on tumor growth, studied in detail to gain new insights into resistance mechanisms and to predict tumor growth dynamics across different patients \cite{cminns, mamba2}. The system is solved over a treatment window of $t \in [0,17]$ days and is intentionally constructed to exhibit strong stiffness arising from multi-timescale dynamics, nonlinear growth saturation, and injection-driven forcing.

\paragraph{Pharmacokinetics (PK):}
The pharmacokinetic component is described by a two-compartment model with central and peripheral compartments. Let $A_1(t)$ and $A_2(t)$ denote the drug amounts in the central and peripheral compartments, respectively. Their dynamics are governed by
\begin{align}
\frac{dA_1}{dt} &= - (k_{10} + k_{12}) A_1(t) + k_{21} A_2(t) + I(t), \\
\frac{dA_2}{dt} &= k_{12} A_1(t) - k_{21} A_2(t),
\end{align}
where the plasma drug concentration driving the pharmacodynamic response is defined as
\begin{equation}
c(t) = \frac{A_1(t)}{V_1}.
\end{equation}

The input term $I(t)$ represents chemotherapy drug injection according to the treatment schedule. In this benchmark, an injection is administered at day $t=2$. In addition, a prior injection is assumed to have occurred before $t=0$, resulting in residual drug present at the start of the simulation. Rather than modeling the injection as an idealized impulse, the PK system is solved analytically, yielding a continuous concentration profile composed of fast and slow exponential decay modes. The analytical solution takes the form
\begin{equation}
c(t) = \sum_{t_d \in \mathcal{T}_{\text{dose}}}
\left(
A e^{-\alpha (t - t_d)} + B e^{-\beta (t - t_d)}
\right)\mathbb{I}_{t \ge t_d},
\end{equation}
where $\alpha \gg \beta$ correspond to the eigenvalues of the compartmental system. This formulation captures the rapid rise in drug concentration following injection, followed by slower elimination dynamics.

\paragraph{Pharmacodynamics (PD):}
Tumor dynamics are modeled using a four-state transit compartment system, where $x_1(t)$ represents proliferating tumor cells and $x_2(t)$--$x_4(t)$ denote successive damaged or dying cell compartments. The governing equations are
\begin{align}
\frac{dx_1}{dt} &=
\frac{\lambda_1 x_1(t)}
{\left[1 + \left(\frac{\lambda_1}{\lambda_2}\,\omega(t)\right)^{\Psi}\right]^{1/\Psi}}
- k_2\, c(t)\, x_1(t), \\
\frac{dx_2}{dt} &= k_2\, c(t)\, x_1(t) - k_1 x_2(t), \\
\frac{dx_3}{dt} &= k_1 \big(x_2(t) - x_3(t)\big), \\
\frac{dx_4}{dt} &= k_1 \big(x_3(t) - x_4(t)\big),
\end{align}
with the total tumor burden defined as
\begin{equation}
\omega(t) = x_1(t) + x_2(t) + x_3(t) + x_4(t).
\end{equation}
Tumor weight observations are available at $t=0$, corresponding to $\omega(0)$.

The tumor growth term follows a generalized logistic formulation with shape parameter $\Psi$, while the drug-induced cell kill enters multiplicatively through the PK concentration $c(t)$. The cytotoxic effect is first applied to the proliferating compartment and subsequently propagates through the downstream transit compartments with a delay governed by $k_1$.
Following intravenous administration of paclitaxel, the pharmacokinetic parameters are $V_1 = 0.81~\mathrm{L/kg}$, $k_{10} = 0.868~\mathrm{h^{-1}}$, $k_{12} = 0.0060~\mathrm{h^{-1}}$, and $k_{21} = 0.0838~\mathrm{h^{-1}}$, while the pharmacodynamic parameters are $\Psi = 20$, $\lambda_1 = 0.273~\mathrm{day^{-1}}$, and $\lambda_2 = 0.814~\mathrm{g\cdot day^{-1}}$. The time evolution of $x_1$, $x_2$, $x_3$, $x_4$, together with $W$, is illustrated in Figure~\ref{fig:pkpd}.

The coupled PK--PD system exhibits \textbf{stiffness} arising from intrinsic multi-timescale dynamics in the governing equations. First, the two-compartment PK model has eigenvalues $\alpha \gg \beta$ that differ by several orders of magnitude, creating coexisting fast distribution and slow elimination phases in the drug concentration profile $c(t)$. This disparity in decay rates is a primary source of mathematical stiffness, requiring small integration timesteps for numerical stability even during smooth evolution between injections. Second, the sequential transit compartments ($x_2 \to x_3 \to x_4$) introduce delayed pharmacodynamic responses with characteristic timescales governed by $k_1$, which interact with the much faster cytotoxic effects (controlled by $k_2 c(t)$) acting on the proliferating compartment. Third, the Hill-type growth inhibition term with $\Psi = 20$ introduces extreme nonlinearity near the carrying capacity, creating steep gradients and strong curvature in the solution manifold that further exacerbate stiffness.

Beyond the inherent stiffness, the system also features \textbf{discontinuous forcing} due to drug administration. Chemotherapy injections at discrete timepoints introduce sharp, localized increases in $c(t)$ that act as non-smooth forcing terms, generating rapid transients in the PD equations immediately following administration. While these injection-driven transients create numerical challenges distinct from stiffness---primarily requiring adaptive timestep refinement and careful collocation placement---they interact with the underlying stiff dynamics to produce a particularly challenging coupled system. The combination of multi-timescale eigenvalue separation, nonlinear saturation effects, and discontinuous forcing makes this benchmark representative of real-world pharmacological modeling scenarios. The solution of the full system is shown in Figure~\ref{fig:pkpd}, alongside the corresponding PINN predictions.

\begin{figure}
    \centering
    \includegraphics[width=0.9\linewidth]{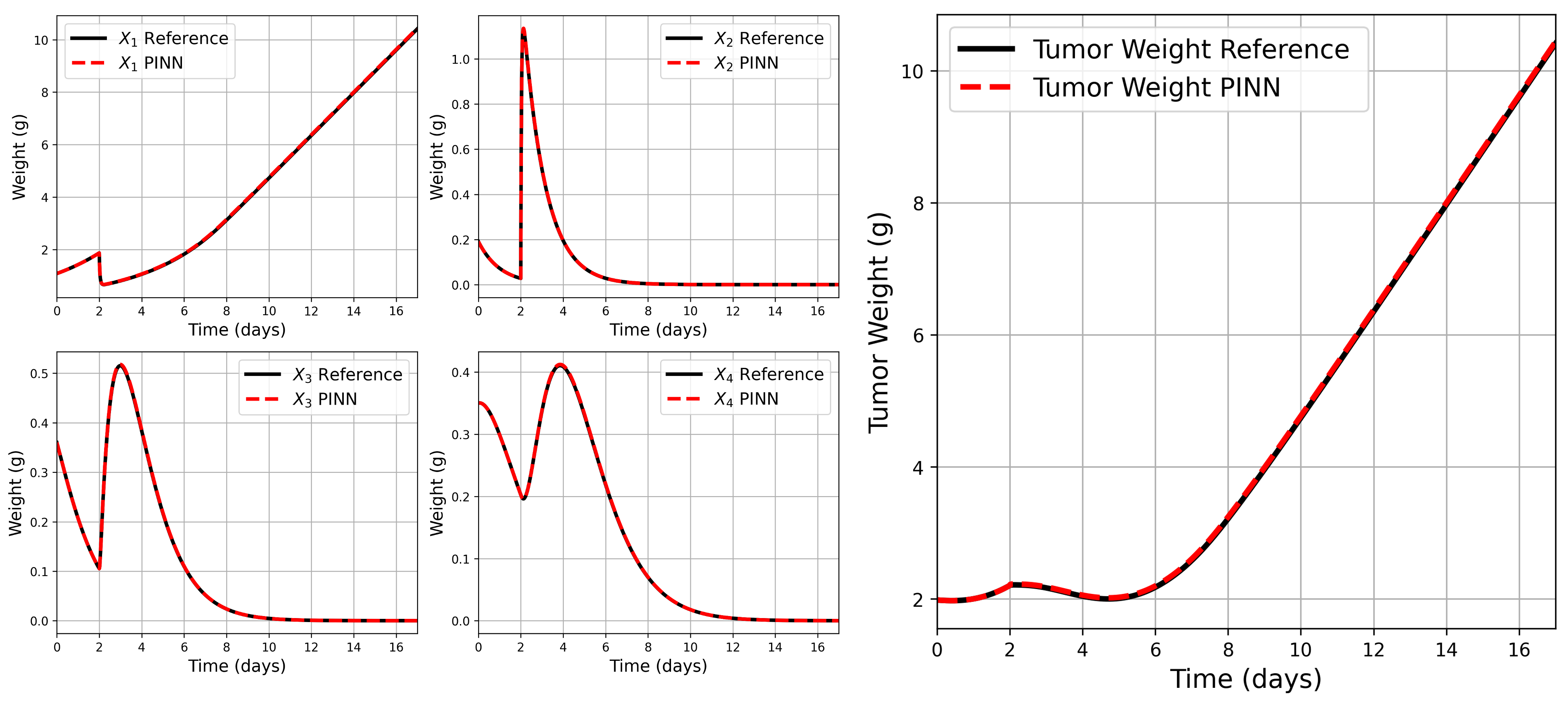}
    \caption{\textbf{Stiff PK--PD ODE System:} Reference vs. PINN (NG).Comparison of the exact solution and the PINN prediction obtained with the NG optimizer, showing accurate reconstruction of the multiscale stiff dynamics.}

    \label{fig:pkpd}
\end{figure}


To accurately resolve the stiff dynamics induced by drug injection, we employed a non-uniform temporal collocation strategy. The final distribution was selected after an ablation study comparing different point allocations. A total of 1000 collocation points were deterministically distributed using piecewise uniform grids: 300 points over $[0,1.9]$ days, 300 points over $[1.9,4.0]$ days, and 400 points over $[4.0,17.0]$ days. This allocation increases sampling density around the injection time ($t \approx 2$ hr), while maintaining sufficient coverage of the slower long-time pharmacodynamic relaxation regime. The resulting structured grid balances resolution of fast–slow multiscale behavior without excessively increasing computational cost.

Table~\ref{tab:stiff_pkpd_optimizers} summarizes the quantitative performance of the different optimizers on the stiff PK--PD system. Adam exhibits the largest errors across all state variables, particularly in the intermediate transit compartments, indicating difficulty in resolving the coupled fast--slow dynamics. SOAP substantially improves accuracy, reducing errors by more than one order of magnitude compared to Adam, though some degradation remains in deeper compartments. Classical BFGS shows uneven performance: while achieving low error for $X_1$, it struggles to consistently control errors in $X_2$--$X_4$, which reflects sensitivity to stiffness and curvature anisotropy. The structured second-order variants (SSBFGS and SSBroyden) provide more balanced accuracy across all compartments, with SSBroyden achieving uniformly low errors while maintaining competitive computational cost. The NG method achieves the lowest errors in the deeper transit states and tumor burden $W$, which demonstrates strong robustness to multiscale stiffness. 

Figure~\ref{fig:pkpd_abs} shows the time-resolved absolute error (log scale) for each state variable $X_1$--$X_4$ under different optimization strategies. The stiff transient immediately following drug injection (around $t \approx 2$ hr) induces a sharp rise in error for all methods. Adam exhibits the largest and most persistent errors across all compartments as expected, particularly in $X_1$ and $X_2$, indicating difficulty in resolving the fast PK-driven transient and its propagation through the PD chain. SOAP significantly reduces the overall magnitude of the error compared to Adam but still shows localized spikes near the injection time and mild long-time drift. Classical BFGS improves late-time accuracy but demonstrates instability across intermediate compartments, with slower decay of error in $X_2$ and $X_3$. In contrast, the structured quasi-Newton methods (SSBFGS and SSBroyden) achieve more uniform error reduction across all compartments, with smoother temporal decay and improved stability after the transient phase. The NG method exhibits the most consistent long-time error decay, reaching the lowest error levels in $X_2$--$X_4$ and maintaining strong robustness to stiffness. 
Results show that curvature-aware and geometry-informed optimizers significantly outperform first-order methods in both accuracy and stability for this stiff PK--PD problem.

\begin{figure}
    \centering
    \begin{minipage}{0.95\linewidth}
        \centering
        \includegraphics[width=0.9\linewidth]{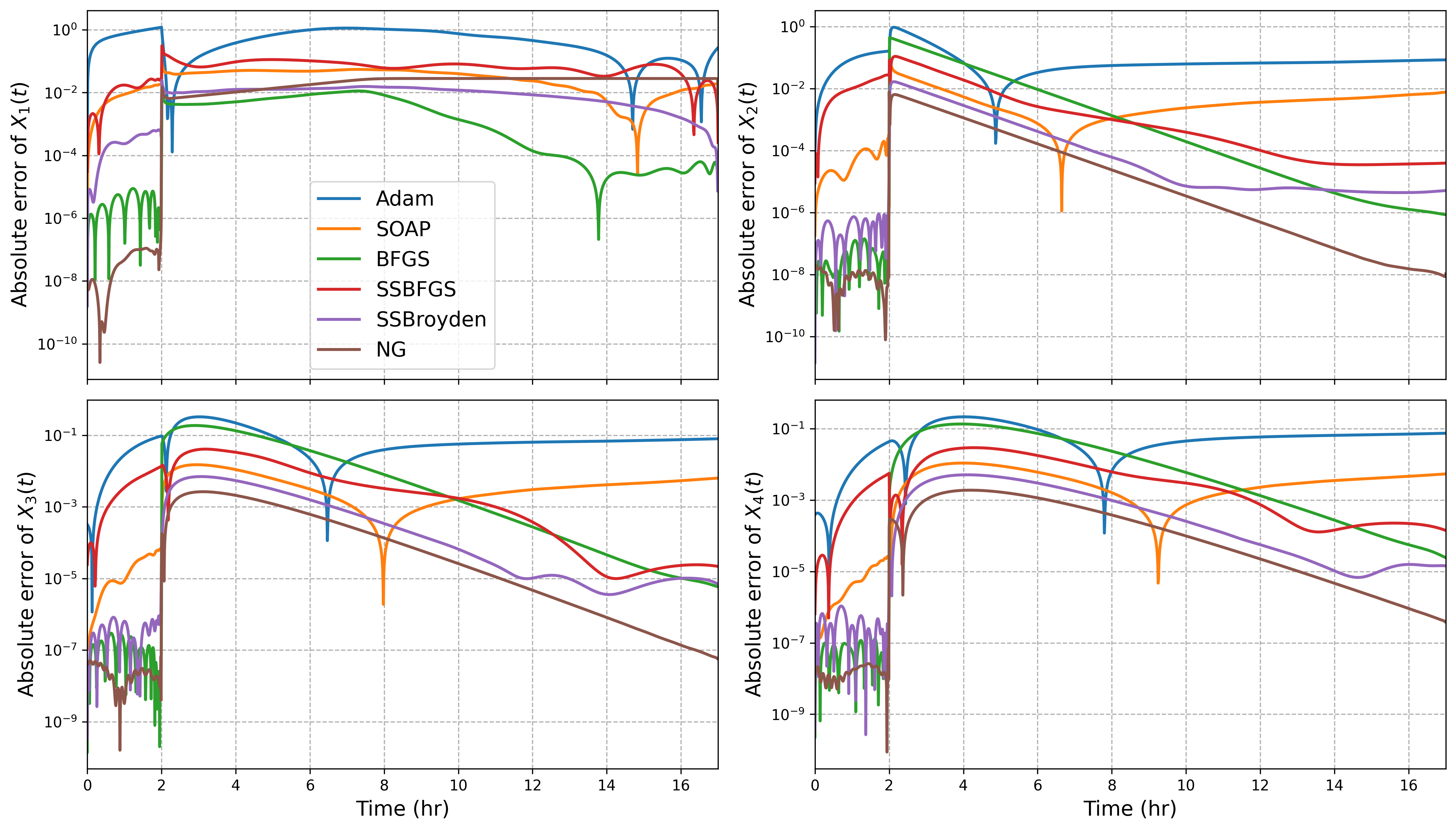}
        \caption{\textbf{Stiff PK--PD ODE System:}
        Time-resolved absolute error for each state variable, shown on a logarithmic scale, comparing PINN solutions trained with different optimization methods. The results illustrate optimizer-dependent performance in resolving stiff multiscale dynamics.}
        \label{fig:pkpd_abs}
    \end{minipage}
    \vspace{0.4cm}
    \begin{minipage}{0.95\linewidth}
        \centering
        \rowcolors{2}{cyan!15}{white}
        \scalebox{0.8}{
        \begin{tabular}{|l|r|r|r|r|}
        \hline
        \rowcolor{cyan!40}
        \textbf{Optimizer} 
        & \textbf{Relative \ $L_2 (X_1, X_2, X_3, X_4)$} 
        & \textbf{Relative \ $L_2$ (W)} 
        & \textbf{Relative \ $L_\infty$ (W)} 
        & \textbf{Training time (s)} \\
        \hline
        Adam 
        & $1.25 \times 10^{-1} , 7.93 \times 10^{-1} , 5.79\times 10^{-1}, 4.50\times 10^{-1}$ 
        & $1.31 \times 10^{-1}$ 
        & $1.07 \times 10^{-1}$ 
        & 323 \\
        SOAP 
        & $6.80 \times 10^{-3} , 3.67 \times 10^{-2} , 2.97\times 10^{-2}, 2.46\times 10^{-2}$ 
        & $5.79 \times 10^{-3}$ 
        & $5.10 \times 10^{-3}$ 
        & 461 \\
        BFGS 
        & $8.95 \times 10^{-4} , 3.44 \times 10^{-1} , 3.28\times 10^{-1}, 2.74\times 10^{-1}$ 
        & $3.26 \times 10^{-2}$ 
        & $4.87 \times 10^{-2}$ 
        & 118 \\
        SSBFGS 
        & $1.38 \times 10^{-2} , 9.72 \times 10^{-2} , 7.23\times 10^{-2}, 5.91\times 10^{-2}$ 
        & $1.92 \times 10^{-2}$ 
        & $2.46 \times 10^{-2}$ 
        & 83 \\
        SSBroyden 
        & $1.89 \times 10^{-3} , 1.44 \times 10^{-2} , 1.23\times 10^{-2}, 1.03\times 10^{-2}$ 
        & $2.79 \times 10^{-3}$ 
        & $2.76 \times 10^{-3}$ 
        & 98 \\
        NG
        & $1.29 \times 10^{-2} , 1.08 \times 10^{-3} , 7.61\times 10^{-4}, 6.45\times 10^{-4}$ 
        & $2.55 \times 10^{-3}$ 
        & $1.57 \times 10^{-3}$ 
        & 302 \\
        \hline
        \end{tabular}}
        \captionof{table}{\textbf{Stiff PK--PD ODE System:}
        Relative $L_2-$ errors for the transit compartments $(X_1,X_2,X_3,X_4)$ and tumor burden $W$, together with relative $L_\infty$ error for $W$ and total training time (seconds), obtained using different optimization strategies for PINN training. The results highlight differences in accuracy, robustness across state variables, and computational efficiency in the stiff multiscale regime.}
        \label{tab:stiff_pkpd_optimizers}
    \end{minipage}

\end{figure}

\paragraph{Key observations from the Stiff PK--PD ODE System:} We evaluated Wolfe, trust-region, and zoom line-search strategies, and report results for quasi-Newton methods using trust-region line search due to its superior robustness across optimizers. The NG method was particularly sensitive to increased collocation density near drug administration, which led to ill-conditioning of the Hessian/Fisher matrix and numerical instability. For other second-order methods, redistributing collocation points significantly affected convergence speed and stability, occasionally causing stagnation in sharp local minima due to stiffness-induced curvature in the loss landscape.

\section{Performance Analysis of 
Optimizers using Roofline Models}

The roofline model~\cite{williams2009roofline, fischer2022nekrs} is a performance bound
that characterizes the achievable floating-point throughput of a kernel
as a function of its arithmetic intensity. It provides an intuitive
visual framework for identifying whether a kernel is limited by compute
throughput or memory bandwidth, and quantifies the gap between measured
and peak hardware performance.

\paragraph{Arithmetic Intensity}
The arithmetic intensity $I$ of a kernel is defined as the ratio of
the total number of floating-point operations performed to the total
number of bytes transferred between main memory (DRAM) and the processor:

\begin{equation}
    I = \frac{W}{Q} \quad \left[\frac{\text{FLOP}}{\text{byte}}\right],
    \label{eq:arithmetic_intensity}
\end{equation}

\noindent where $W$ denotes the total floating-point operation count
(FLOP) and $Q$ denotes the total memory traffic (bytes). A kernel with
high arithmetic intensity reuses data heavily from cache and is
\emph{compute-bound}, while a kernel with low arithmetic intensity is
\emph{memory-bound}.

\paragraph{ Roofline Analysis}

The theoretical roofline bound $P_{\text{roof}}(I)$ is constructed
from two hardware-specific constants taken directly from the device
specification sheet:

\begin{itemize}
    \item $\pi$ : peak floating-point throughput (FLOP/s), and
    \item $\beta$ : peak memory bandwidth (bytes/s).
\end{itemize}

\noindent The roofline performance bound at a given arithmetic intensity
$I$ is then defined as:

\begin{equation}
    P_{\text{roof}}(I) = \min\!\left(\pi,\ \beta \cdot I\right)
    \quad \left[\frac{\text{FLOP}}{\text{s}}\right],
    \label{eq:roofline}
\end{equation}

\noindent which enforces that no kernel can exceed either the peak
compute ceiling $\pi$ or the memory-bandwidth ceiling $\beta \cdot I$.
The intersection of these two bounds defines the \emph{ridge point}:

\begin{equation}
    I^{*} = \frac{\pi}{\beta} \quad \left[\frac{\text{FLOP}}{\text{byte}}\right],
    \label{eq:ridge_point}
\end{equation}

\noindent which separates the memory-bound regime ($I < I^{*}$) from
the compute-bound regime ($I > I^{*}$). For the NVIDIA H100 SXM GPU
used in this work, the hardware parameters and the resulting ridge
point are summarized in Table~\ref{tab:h100_specs}.

\begin{table}[ht]
    \centering
    \caption{NVIDIA H100 SXM hardware specifications used to construct
             the theoretical roofline model.}
    \label{tab:h100_specs}
    \begin{tabular}{lcc}
        \hline
        \textbf{Parameter} & \textbf{Symbol} & \textbf{Value} \\
        \hline
        Peak FP64 throughput & $\pi$    & $67$~TFLOP/s   \\
        Peak memory bandwidth & $\beta$ & $3.35$~TB/s    \\
        Ridge point           & $I^{*}$ & $10.15$~FLOP/byte \\
        \hline
    \end{tabular}
\end{table}

The resulting roofline plot is shown in Figure~\ref{fig:rl_SSBFGS}. The horizontal ceiling corresponds to the peak FP64 compute performance of $67$~TFLOP/s, while the diagonal line represents the peak memory bandwidth of $3.35$~TB/s. Their intersection, at approximately $10.15$~FLOP/byte, defines the ridge point, which indicates the arithmetic intensity threshold beyond which a kernel transitions from being memory-bound to compute-bound.
\begin{figure}
    \centering
    \includegraphics[width=\linewidth]{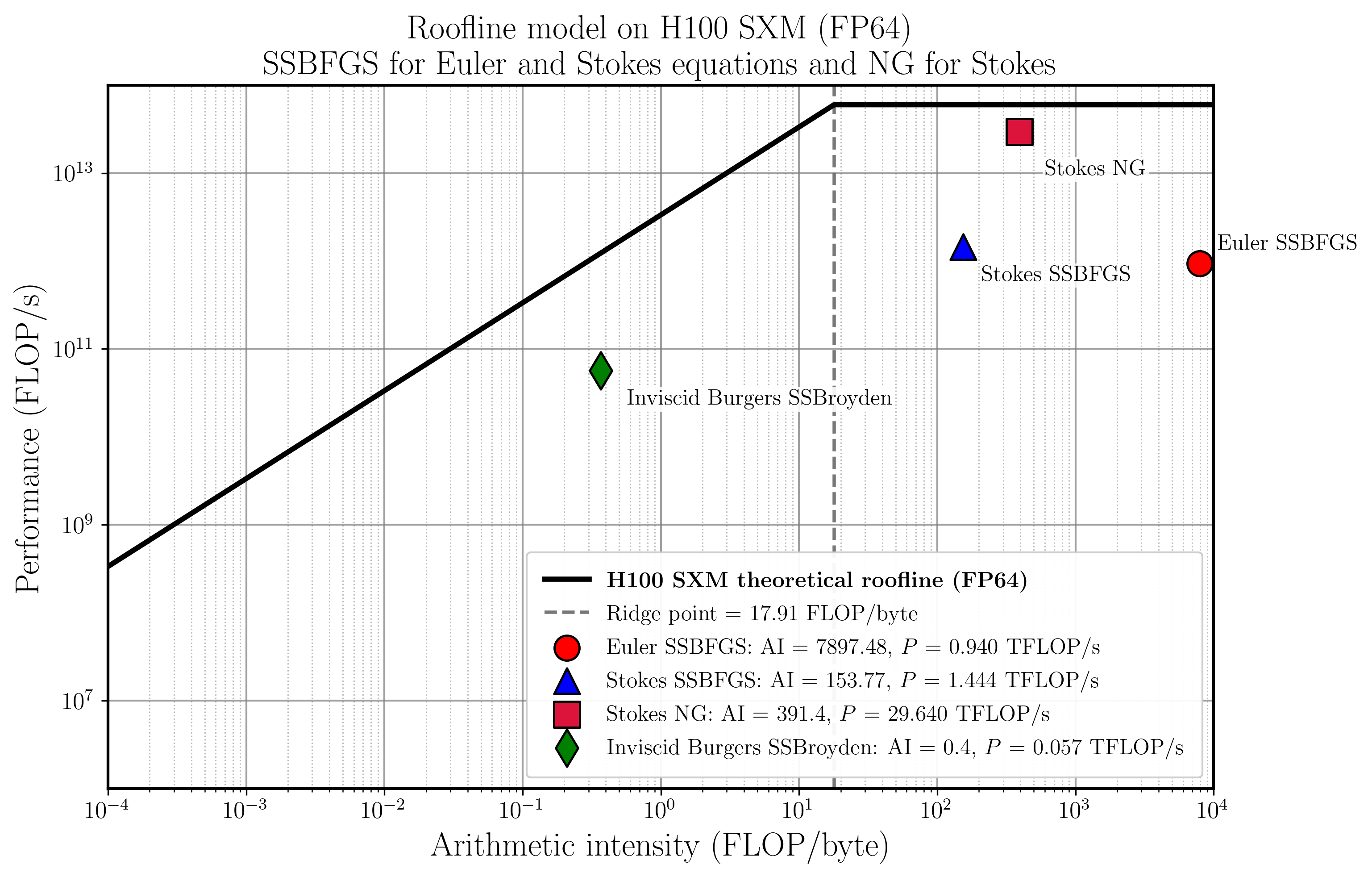}
    \caption{Roofline analysis for the SSBFGS and NG optimizers 
applied to the inviscid Burgers, Euler, and Stokes equations. The 
solid black line denotes the theoretical roofline, and the vertical 
dashed line marks the ridge point, which delineates the transition 
from the memory-bound to the compute-bound regime. }
    \label{fig:rl_SSBFGS}
\end{figure}

Three kernels---Euler SSBFGS, Stokes SSBFGS, and Stokes NG---appear to the right of the roofline ridge point, indicating that they operate in the \emph{compute-bound} regime. In this regime, performance is governed primarily by the GPU's peak floating-point throughput rather than by memory bandwidth. This behavior is desirable, as it suggests that memory traffic is not the principal bottleneck and that the available bandwidth is being utilized efficiently relative to the computational workload.

By contrast, the 1D inviscid Burgers kernel lies to the left of the ridge point, placing it in the \emph{memory-bound} regime. In this case, performance is limited primary by data movement rather than arithmetic throughput. A plausible explanation is the relatively small neural network required to represent the 1D inviscid Burgers solution: because the model is compact, the amount of computation performed per byte of data moved is reduced, leading to a lower arithmetic intensity. As a result, execution time becomes dominated by memory access and data transfer rather than by floating-point operations.

Despite sharing the same regime, the solvers exhibit distinct arithmetic intensities. The Euler SSBFGS solver achieves a higher arithmetic intensity, which can be attributed to its relatively small network architecture (approximately $3{,}000$ parameters) and a correspondingly modest dense SSBFGS Hessian. The small parameter size allows both the Hessian and network weights to reside largely within the GPU L2 cache, leading to about $75\%$ memory bandwidth utilization. This indicates substantial data reuse at the cache level rather than frequent access to DRAM.

In contrast, the Stokes SSBFGS solver employs a significantly larger network (approximately $33{,}667$ parameters) and a much larger dense Hessian matrix (around $9$~GB). As a result, it saturates both compute and memory bandwidth at nearly $100\%$, achieving slightly higher performance than the Euler case. However, the large dense Hessian induces substantial DRAM traffic at each iteration, which lowers its arithmetic intensity relative to Euler despite the increased problem size. For the Gauss-Newton solver, the arithmetic intensity is computed as $\text{AI} = 391.4~\text{FLOP/byte}$. The achieved performance measured as reliable FLOPs divided by wall-clock step time is $29.640$~TFLOP/s, further confirming operation in the compute-bound regime.

Overall, these results highlight an important scalability challenge for second-order optimization of PINNs. As the network size increases, the $\mathcal{O}(n^{2})$ memory footprint associated with dense SSBFGS becomes a dominant factor. This growth drives the solver closer to the memory bandwidth ceiling, ultimately constraining achievable arithmetic intensity and limiting compute efficiency on modern GPU architectures.

\begin{lstlisting}[caption={Instantiation of SSBroyden and SSBFGS optimizers with zoom and Wolfe line search with backtracking line search, respectively}, label={code:Class_variants}]
from typing import Callable
import optimistix as optx
class SSBroydenZoom(optx.AbstractSSBroyden):
    """SSBroyden + zoom line search search."""
    rtol: float
    atol: float
    norm: Callable = optx.max_norm
    use_inverse: bool = True
    search: optx.AbstractSearch = optx.Zoom()
    descent: optx.AbstractDescent = optx.NewtonDescent()
    verbose: frozenset[str] = frozenset()


class SSBFGSBacktrackingWolfe(optx.AbstractSSBFGS):
    """ SSBFGS + backtracking line search."""
    rtol: float
    atol: float
    norm: Callable = optx.max_norm
    use_inverse: bool = True
    search: optx.AbstractSearch = optx.BacktrackingStrongWolfe()
    descent: optx.AbstractDescent = optx.NewtonDescent()
    verbose: frozenset[str] = frozenset()
\end{lstlisting}
\section{Reproducibility} 
To promote reproducibility and facilitate wider adoption, we outline the design principles underlying our implementation and provide a publicly available codebase. All implementations are developed using the JAX framework, ensuring both efficiency and flexibility in constructing optimizer variants.

Each optimizer instance is defined by several key components:
\begin{enumerate}
\item Relative tolerance (\texttt{rtol})
\item Absolute tolerance (\texttt{atol})
\item Choice of norm for convergence criteria (e.g., $L_1$, $L_2$, or $L_{\infty}$ norm of the gradient)
\item Line search strategy (e.g., backtracking or trust-region)
\item Newton-type descent update (cf. Equation~\eqref{eq:ssbroyden_Hk})
\end{enumerate}

The selection of these components allows the optimizer—whether SSBFGS or SSBroyden—to be tailored to the specific problem under consideration. To support this flexibility, we utilize the state-based \textit{Optimistix} library~\cite{rader2024optimistix}. Different optimizer variants are implemented by extending the \texttt{AbstractQuasiNewton} class, enabling modular customization of line search strategies, tolerance settings, convergence criteria, and descent formulations. As an illustrative example, we present an implementation of the SSBFGS optimizer using a trust-region line search combined with an $L_{\infty}$-norm convergence criterion in the code listing shown in Codelisting ~\ref{code:Class_variants}.We plan to release the complete library after the acceptance of the paper.

\paragraph{Code Availability:} 
The GitHub repository associated with this work is publicly accessible at \url{https://github.com/CrunchOptimizer} and will be updated upon acceptance of the manuscript for publication.

\section{Summary}\label{summary}

This work studied curvature-aware optimization for PINNs across a broad set of challenging scientific machine learning problems. The computational experiments were organized around three main classes of PDEs—elliptic, parabolic, and hyperbolic—as well as a stiff PK--PD ODE system. Across all cases, we compared self-scaled quasi-Newton methods, NG / Gauss--Newton methods, and structured adaptive optimizers, with emphasis on convergence speed, robustness, accuracy, and scalability. The results show that curvature-aware optimizers significantly improve PINN training over standard first-order approaches. SSBFGS, SSBroyden, and NG methods performed strongly across several experiments and were particularly effective in ill-conditioned settings and when very high solution accuracy was required. At the same time, the experiments highlight that optimization alone is not sufficient for discontinuous problems: for inviscid Burgers and Euler equations, accurate solutions required shock-aware and conservation-consistent PINN formulations.

For elliptic problems, including the Helmholtz and Stokes equations, the results showed that all curvature-aware optimizers were capable of reaching high-accuracy solutions when properly configured. Among them, the NG method was often the fastest. The SSBFGS and SSBroyden methods also performed well, achieving comparable final accuracy, though sometimes at a higher computational cost. 
For the 2D Helmholtz problem with \(a_1 = 1\), \(a_2 = 4\), and \(k = 1\), SSBFGS, SSBroyden, and NG all achieved errors of order \(10^{-9}\). Among these methods, NG was the most efficient, requiring only 251\,s of training time, compared with 319\,s for SSBroyden and 564\,s for SSBFGS. In contrast, SOAP stagnated at an error of order \(10^{-4}\) for the same case while also requiring substantially longer training time.
In the Stokes problem, both methods matched the accuracy of the NG approach across all line search strategies. However, the NG method was approximately $24\times$ more computationally efficient than these self-scaled optimizers. In contrast, the SOAP optimizer failed to outperform any of the curvature-aware methods, both in terms of runtime and accuracy, and remained significantly inferior.

For parabolic problems, represented here by the 2D viscous Burgers equation, a similar trend was observed. Curvature-aware optimization again significantly outperformed standard first-order training, and all of the main optimizers considered were able to produce high-accuracy solutions. Within this class of problems, the natural-gradient method was particularly attractive because it combined strong accuracy with lower computational cost, while the quasi-Newton methods also delivered competitive accuracy and robust convergence.
For hyperbolic problems, including the inviscid Burgers equation and the 1D Euler equations, the picture was more subtle. In these cases, optimization alone was not sufficient to recover accurate physical solutions, because shocks and discontinuities make PINN training fundamentally more difficult. We found that the natural-gradient and Broyden-family methods behaved similarly in overall performance once the formulation was made shock-aware and conservation-consistent. This motivated the introduction of new PINN methodologies tailored to hyperbolic PDEs, including Roe linearization, entropy-constrained relaxation formulations, and the Harten--Lax--van Leer--Contact (HLLC) flux approach. These additions were essential for obtaining stable and accurate solutions in the presence of shocks and contact discontinuities.
%
Beyond the PDE benchmarks, the stiff PK--PD ODE system---characterized by multi-timescale dynamics (eigenvalues $\alpha \gg \beta$), extreme nonlinear saturation ($\Psi = 20$), and discontinuous forcing---confirmed that curvature-aware optimizers significantly outperform first-order methods. NG achieved the lowest errors in deeper transit states (relative $L_2 < 10^{-3}$), while structured quasi-Newton methods (SSBroyden, SSBFGS) provided balanced accuracy across all compartments. In contrast, classical BFGS showed uneven performance and Adam struggled throughout, particularly in intermediate transit compartments.

Based on these results, we recommend natural-gradient methods as a first choice for elliptic and parabolic PDEs, especially when fast convergence is the primary objective. When robustness and very high final accuracy are desired, SSBFGS and SSBroyden methods provide strong alternatives. For hyperbolic PDEs with discontinuities, we recommend combining curvature-aware optimizers with shock-aware and conservation-consistent PINN formulations, such as Roe-based linearization or HLLC-flux constructions, since optimization alone is generally not sufficient in this regime.

We further analyzed optimizer performance using roofline models, focusing on arithmetic intensity, memory bandwidth, and compute throughput. The results show that the main second-order kernels operate near the roofline ceiling and are mostly compute-bound, while smaller problems such as 1D inviscid Burgers remain memory-bound. This analysis highlights both the strong hardware efficiency of the proposed methods and the memory-scaling limitations of dense second-order optimization. 
Last but not least, we developed a novel batch training algorithm for the SSBFGS and SSBroyden methods. We validated its correctness through testing on the Helmholtz and Poisson equations, demonstrating improved stability and convergence rates. These properties make the approach readily applicable to large-scale data-driven problems.

\section*{Acknowledgments}
This research was primarily supported as part of the AIM for Composites, an Energy Frontier Research Center funded by the U.S. Department of Energy (DOE), Office of Science, Basic Energy Sciences (BES), under Award DE-SC0023389 (computational studies, data analysis). The work of KS and GEK  was partially supported by the U.S. Department of Energy (DOE), Office of Science, Advanced Scientific Computing Research (ASCR) program under the Scientific Discovery through Advanced Computing (SciDAC) Institute “LEADS: LEarning-Accelerated Domain Science,” Subcontract 831126 under DE-AC05-76RL01830. We would like to acknowledge \textbf{Prabhjyot Singh Saluja} (CCV, Brown) for providing the hardware expertise and support for the AI testbed used to benchmark the applications presented in the paper. NA is supported by the National Institutes of Health (NIH) grant R01HL154150. JFU is supported by the predoctoral fellowship
ACIF 2023, cofunded by Generalitat Valenciana and the European Union through the European Social Fund. JFU also acknowledges the support through the grant PID2021-127495NB-I00 funded by MCIN/AEI/10.13039/501100011033 and by the European Union, the Astrophysics and High Energy Physics programme of the Generalitat Valenciana ASFAE/2022/026 funded by MCIN and the European Union NextGenerationEU (PRTR-C17.I1), and the Prometeo excellence programme
grant CIPROM/2022/13. Finally, JFU gratefully acknowledges the support provided by grant CIBEFP/2024/108, cofunded by Generalitat Valenciana and the European Union through the European Social Fund, which enabled a research stay at Brown University, where part of this work was carried out, as well as the kind hospitality of the Division of Applied Mathematics.
JM acknowledges support by the European Union (ERC, FluCo, grant agreement No. 101088488). Views and opinions expressed are however those of the author(s) only and do not necessarily reflect those of the European Union or of the European Research Council. Neither the European Union nor the granting authority can be held responsible for them. A.J. acknowledges financial support from the European Union via the European Defence Fund project Archytas under grant agreement nr.~101167870. Views and opinions expressed are however those of the author(s) only and do not necessarily reflect those of the European Union or the European Commission. Neither the European Union nor the granting authority can be held responsible for them. Furthermore, A.J. acknowledges Flavio Vella for his diligent support and supervision.

\FloatBarrier
\bibliographystyle{cas-model2-names}
\bibliography{cas-refs}

\cleardoublepage

\appendix
\section{Details on Self-Scaled Quasi-Newton Methods}\label{sec:appendix_Self_Scaled}

The parameters $(\tau_k,\phi_k)$ cannot be chosen arbitrarily; they must satisfy certain conditions to guarantee convergence. One of the most fundamental requirements is that the update formulas for the matrices \(H_k\) (or \(B_k\)) preserve positive definiteness, thereby ensuring that the resulting search directions remain descent directions. This requirement imposes the following restriction on \(\eta_k\):
\begin{equation}\label{eq:thetak_star}
\eta_k > \eta_k^{*} \equiv -\frac{1}{h_k b_k - 1} = -\frac{1}{a_k},
\end{equation}
where
\begin{align}
b_k &= \frac{\mathbf{s}_k^\top B_k \mathbf{s}_k}{\mathbf{y}_k^\top \mathbf{s}_k}, \\
h_k &= \frac{\mathbf{y}_k^\top H_k \mathbf{y}_k}{\mathbf{y}_k^\top \mathbf{s}_k}.
\end{align}

Note that if \(H_k\) is positive definite, then \(\eta_k^{*}\) is negative. Indeed,
\begin{equation}
a_k \equiv h_k b_k - 1
= \frac{\left(\mathbf{y}_k^\top H_k \mathbf{y}_k\right)\left(\mathbf{s}_k^\top B_k \mathbf{s}_k\right)}{\left(\mathbf{y}_k^\top \mathbf{s}_k\right)^2} - 1
= \frac{\left\|H_k^{1/2}\mathbf{y}_k\right\|^2 \left\|B_k^{1/2}\mathbf{s}_k\right\|^2}
{\left[\left(H_k^{1/2}\mathbf{y}_k\right)^\top \left(B_k^{1/2}\mathbf{s}_k\right)\right]^2} - 1
\ge 0,
\end{equation}
where the last inequality follows from the Cauchy-Schwarz inequality.

Another restriction on these parameters can be obtained from the convergence theory of quasi-Newton methods. As shown by Al-Baali~\cite{albaali1998}, the sequence \(\{x_k\}_{k\in\mathbb{N}}\) generated by \eqref{eq:qn_1}, with \(\{H_k\}_{k\in\mathbb{N}}\) updated by the self-scaled Broyden formula and \(\{\alpha_k\}_{k\in\mathbb{N}}\) chosen to satisfy the Wolfe conditions, converges superlinearly for convex objective functions if \(\eta_k\tau_k \in [0,1)\), provided that \(0 < \tau_k \le 1\). In particular, the scaled version of BFGS, corresponding to \(\phi_k = 1\), with
\begin{equation}\label{eq:tauk_ssbfgs}
    \tau_k^{(1)} \coloneqq \min\{1,\tau_k^{\mathrm{OL}}\}.
\end{equation}
where
\begin{equation}
    \tau_k^{\mathrm{OL}}
    =
    \frac{1}{b_k}
    \equiv
    \frac{\mathbf{y}_k^\top \mathbf{s}_k}{\mathbf{s}_k^\top B_k \mathbf{s}_k}
    =
    -\frac{\mathbf{y}_k^\top \mathbf{s}_k}{\alpha_k\, \mathbf{s}_k^\top \mathbf{g}_k},
\end{equation}
where \(\tau_k^{\mathrm{OL}}\) is the Oren--Luenberger scaling factor~\cite{oren1974selfscaling}. This choice has been shown to be competitive with the unscaled BFGS algorithm, corresponding to \(\tau_k = 1\), including in PINN applications; see, for example,~\cite{urban2025unveiling,kiyani2025optimizing}. Moreover, if \(\phi_k\) is also allowed to vary with the iterations, then the following choice of \(\tau_k\) and \(\phi_k\), originally introduced by Al-Baali and Khalfan~\cite{AlBaaliKhalfan}, can be used:
\begin{align}
    \tau_k^{(2)} &= \begin{cases}
                        \tau_k^{(1)} \min \left(\sigma_k^{-1/(n-1)}, \frac{1}{\eta_k} \right) 
                        & \quad \mathrm{ if } ~ \eta_k > 0 \\
                        \min \left( \tau_k^{(1)} \sigma_k^{-1/(n-1)}, \sigma_k \right) 
                        & \quad \mathrm{ if } ~ \eta_k \leq 0, 
                    \end{cases} \label{eq:tau2} \\ 
    \eta_k^{(2)} &= \max \left(\eta_k^{-}, \min 
    \left(\eta_k^{+}, {\frac{1 - b_k}{b_k}}\right) \right), \label{eq:phi1}
\end{align}
where \(n = \mathrm{size}(\btheta_k)\) denotes the total number of trainable parameters, and
\begin{align}
    \sigma_k &= 1 + a_k \eta_k^{(2)}, \\
    \eta_k^{-} &= \frac{\rho_k^{-}-1}{a_k}, \\
    \eta_k^{+} &= \frac{1}{\rho_k^{-}}, \\
    \rho_k^{-} &= \min\!\left(1,\, h_k(1-c_k)\right), \\
    c_k &= \sqrt{\frac{a_k}{a_k + 1}}.
\end{align}
This choice has also produced very accurate results for a range of physics problems in PINNs; see, for example,~\cite{urban2025unveiling,kiyani2025optimizing}.

To preserve the positive definiteness of the quasi-Newton update, the curvature condition \( \mathbf{y}_k^\top \mathbf{s}_k > 0 \) must hold at every iteration. To enforce this requirement, suitable line-search strategies are employed. In this work, the focus is on inexact line-search methods used in conjunction with quasi-Newton optimizers. In particular, backtracking line-search procedures based on the Armijo--Wolfe conditions~\cite{wolfe1969convergence} are considered, together with the zoom line-search algorithm described in~\cite{nocedal2006numerical}.
Given a descent direction \(d_k\), the step length \(\alpha_k\) is chosen to satisfy the following conditions:

\begin{itemize}
    \item \textbf{Armijo (sufficient decrease) condition:}
    \begin{equation}
    f(x_k + \alpha_k d_k)
    \le
    f(x_k)
    +
    c_1 \alpha_k \nabla f(x_k)^{\top} d_k,
    \qquad 0 < c_1 < 1.
    \end{equation}

    \item \textbf{Wolfe curvature condition:}
    \begin{equation}
    \nabla f(x_k + \alpha_k d_k)^{\top} d_k
    \ge
    c_2 \nabla f(x_k)^{\top} d_k,
    \qquad c_1 < c_2 < 1.
    \end{equation}

    \item \textbf{Strong Wolfe condition (enforced via the zoom procedure):}
    \begin{equation}
    \left|
    \nabla f(x_k + \alpha_k d_k)^{\top} d_k
    \right|
    \le
    c_2
    \left|
    \nabla f(x_k)^{\top} d_k
    \right|.
    \end{equation}
\end{itemize}
 
In addition to line-search strategies, a linear trust-region approach is also considered as an alternative globalization mechanism~\cite{conn2000trust}. In this setting, the step is computed by minimizing a first-order model of the objective subject to a trust-region constraint,
\begin{equation}
\min_{p_k}\; m_k(p_k)
=
f(x_k) + \nabla f(x_k)^{\top} p_k,
\qquad
\text{subject to } \|p_k\| \le \Delta_k,
\end{equation}
where \(\Delta_k>0\) denotes the trust-region radius. The resulting step is given by the Cauchy, or steepest-descent, direction truncated to the trust-region boundary,
\begin{equation}
p_k = -\Delta_k \frac{\nabla f(x_k)}{\|\nabla f(x_k)\|}.
\end{equation}
The step is then accepted, and the trust-region radius is updated using the ratio of actual to predicted reduction,
\begin{equation}
\rho_k =
\frac{
f(x_k) - f(x_k + p_k)
}{
m_k(0) - m_k(p_k)
}.
\end{equation}
with $\Delta_k$ adjusted adaptively based on the value of $\rho_k$. For a more detailed and rigorous treatment of trust-region methods, we refer the reader to Chapter 4 of \cite{nocedal2006numerical}.
 

{

\section{Details on 2D Helmholtz problem }\label{sec:appendix_2D_Helmholtz_problem}

In this section, we discuss the 2D Helmholtz problem in greater detail for the cases $a_1 = 1$ and $a_2 = 4$, considering both $K = 10$ and $K = 100$.

\begin{minipage}[b]{\linewidth}
  \centering
  \scalebox{0.85}{%
  \begin{tabular}{|c|c|l|c|c|c|c|}
    \hline
    \rowcolor{cyan!45}
    \textbf{Case} & \textbf{$K$} & \textbf{Optimizer} &
    \textbf{Relative $L_{\infty}$} & \textbf{Relative $L^2$} &
    \textbf{Training time (s)} & \textbf{\# parameters} \\
    \hline
    \rowcolor{kTen} 11 & 10 & SSBFGS (\texttt{optx.TrustRegion})  & $7.6 \times 10^{-6}$ & $6.3 \times 10^{-6}$ & 564 & 4051 \\ \hline
    \rowcolor{kTen} 12 & 10 & SSBFGS (\texttt{optx.Wolfe})        & $2.0 \times 10^{+00}$ & $2.4 \times 10^{+00}$ & 449 & 4051 \\ \hline
    \rowcolor{kTen} 13 & 10 & SSBroyden (\texttt{optx.TrustRegion}) & $8.5 \times 10^{-7}$ & $6.9 \times 10^{-7}$ & 617 & 4051 \\ \hline
    \rowcolor{kTen} 14 & 10 & SSBroyden (\texttt{optx.Wolfe})      & $4.7 \times 10^{-7}$ & $4.2 \times 10^{-7}$ & 475 & 4051 \\ \hline
    \rowcolor{kTen} 15 & 10 & NG                   & $5.5 \times 10^{-8}$ & $8.0 \times 10^{-8}$ & 587 & 4051 \\ \hline

    \rowcolor{kHundred} 16 & 100 & SSBFGS (\texttt{optx.TrustRegion}) & $7.0 \times 10^{-6}$ & $3.6 \times 10^{-6}$ & 336 & 4051 \\ \hline
    \rowcolor{kHundred} 17 & 100 & SSBFGS (\texttt{optx.Wolfe})       & $5.8 \times 10^{-7}$ & $2.8 \times 10^{-7}$ & 377 & 4051 \\ \hline
    \rowcolor{kHundred} 18 & 100 & SSBroyden (\texttt{optx.TrustRegion}) & $1.0 \times 10^{-6}$ & $4.0 \times 10^{-7}$ & 341 & 4051 \\ \hline
    \rowcolor{kHundred} 19 & 100 & SSBroyden (\texttt{optx.Wolfe})    & $3.9 \times 10^{-7}$ & $2.0 \times 10^{-7}$ & 384 & 4051 \\ \hline
    \rowcolor{kHundred} 20 & 100 & NG                   & $1.0 \times 10^{-7}$ & $2.9 \times 10^{-8}$ & 296 & 4051 \\ \hline

  \end{tabular}%
  }

  \captionof{table}{\textbf{2D Helmholtz problem with $a_1=1$, $a_2=4$, and K = 1.}
  All optimizers use the same fully connected network architecture with four hidden layers and 30 neurons per layer.
  The Fourier feature mapping in~\ref{Fourier_modes} uses mode $m=1$, resulting in four input features.}
  \label{tab:Helmholtz_witha_1=1a_2=4k100}
\end{minipage}


\section{Stochastic SSBFGS and SSBroyden Methods}\label{app:algo_sSSBFGS_sSSBroyden}
This appendix presents detailed pseudocode for the stochastic SSBFGS and SSBroyden optimizers in Algorithm \ref{alg:sQuasiNewton}.
\begin{algorithm}
\DontPrintSemicolon
\SetAlgoLined
\LinesNumbered
\KwIn{Neural network model $f_\theta$ with parameters $\theta \in \mathbb{R}^n$}
\KwIn{Learning rate $\eta > 0$, damping threshold $\tau > 0$, mini-batch size $B$, number of iterations $K$}
\KwOut{Trained model parameters $\theta$}

Initialize Hessian approximation $H_0 \gets I_n$\;
Initialize previous gradient $g_0 \gets 0$\;
Initialize model parameters $\theta_0$\;

\For{$k \gets 0$ \KwTo $K-1$}{
    Sample mini-batch $\mathcal{B}_k$ of size $B$\;
    Compute mini-batch loss: $L_k \gets \text{Loss}(\theta_k; \mathcal{B}_k)$\;
    Compute gradient: $g_k \gets \nabla_\theta L_k$\;
    
    \If{$k > 0$}{
        Compute step and gradient difference:\;
        $s_{k-1} \gets -\eta H_{k-1} g_{k-1}$\;
        $y_{k-1} \gets g_k - g_{k-1}$\;
        
        \If{$y_{k-1}^\top s_{k-1} \ge \tau \|s_{k-1}\|^2$}{
            Update Hessian:\;
    
             $H_{k} \gets \frac{1}{\tau_{k-1}} \left[H_{k-1} - \frac{H_{k-1}y_{k-1}y_{k-1}^TH_{k-1}}{y_{k-1}^T H_{k-1} y_{k-1}} + \phi_{k-1} v_{k-1} v_{k-1}^T\right] + \frac{s_{k-1} s_{k-1}^T}{y_{k-1}^T s_{k-1}}$, \\
    
        }
        \Else{
            $H_k \gets H_{k-1}$\;
        }
    }
    \Else{
        $H_k \gets H_0$\;
    }
    
    Compute update direction: $\Delta \theta_k \gets - \eta H_k g_k$\;
    Update parameters: $\theta_{k+1} \gets \theta_k + \Delta \theta_k$\;
    Store previous gradient: $g_k \to g_{k-1}$\;
}

\caption{Lazy Stochastic sSSBFGS and sSSBroyden for Mini-Batch Training}
\label{alg:sQuasiNewton}
\end{algorithm}

\section{Proof of Theorem \ref{them:batch_th1eorem}} \label{app:ssbfgs_thm}
\begin{proof}
Define the function $g : [0,1] \to \mathbb{R}^n$ by
\[
g(t) := \nabla f(x_{k-1} + t s_{k-1}).
\]
Since $f \in C^2(\mathbb{R}^n)$, the function $g$ is continuously differentiable and
\[
g'(t) = \nabla^2 f(x_{k-1} + t s_{k-1})\, s_{k-1}.
\]

By the Fundamental Theorem of Calculus,
\begin{align*}
y_{k-1}
&= g(1) - g(0) \\
&= \int_0^1 g'(t)\, dt \\
&= \int_0^1 \nabla^2 f(x_{k-1} + t s_{k-1})\, s_{k-1}\, dt.
\end{align*}

Since the Hessian is continuous, the Mean Value Theorem for integrals implies that there exists
$\theta \in (0,1)$ such that
\[
y_{k-1} = \nabla^2 f(x_{k-1} + \btheta s_{k-1})\, s_{k-1},
\]
which proves~\eqref{eq:hessian_rep}.

Taking the inner product with $s_{k-1}$ yields
\[
y_{k-1}^{\top} s_{k-1}
= s_{k-1}^{\top} \nabla^2 f(\xi)\, s_{k-1}.
\]
Using the uniform positive definiteness assumption~\eqref{eq:uniform_pd}, we obtain
\[
s_{k-1}^{\top} \nabla^2 f(\xi)\, s_{k-1}
\ge \tau \, \|s_{k-1}\|^2,
\]
which establishes \eqref{eq:curvature_condition}.
\end{proof}

\section{Correctness of Algorithm \ref{alg:sQuasiNewton}}\label{app:correctness_algorithm}
To demonstrate the correctness of Algorithm~\ref{alg:sQuasiNewton}, we compare its convergence behavior with that of Adam using batch training, alongside the sSSBFGS update, for a one-dimensional Poisson equation given by
\begin{align}
\begin{aligned}\label{eq:poisson1d} 
\frac{d^2 u}{dx^2} &= f(x), 
&& x \in [0,1], \\
f(x) &= -\sin(\pi x), 
\end{aligned}
\end{align}
subject to the Dirichlet boundary conditions
\[
u(0) = 0, \quad u(1) = 0. 
\]

\begin{figure}
\includegraphics[width=\textwidth]{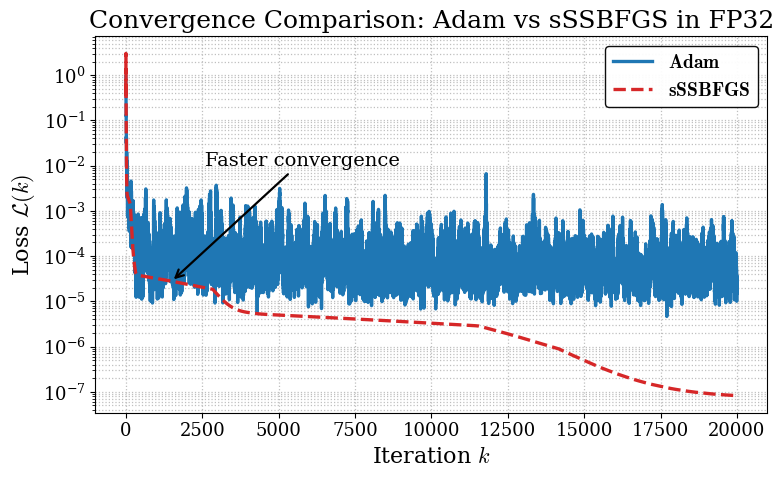}
\caption{Comparison of batch training using Adam and the sSSBFGS approach (Algorithm~\ref{alg:sQuasiNewton}) for the 1D Poisson equation. Notably, in FP32 precision, the sSSBFGS algorithm reaches loss values on the order of $10^{-7}$, significantly outperforming Adam.}
\label{fig:ssBFGS_convergence}
\end{figure}

The convergence histories of Adam and sSSBFGS for batch training are shown in Figure~\ref{fig:ssBFGS_convergence}, with results reported in FP32 precision. Notably, sSSBFGS attains FP32-level accuracy, whereas Adam does not converge to the same level, indicating that sSSBFGS achieves machine precision in the batch training setting.

\end{document}